%% file: elsarticle-template-num.tex
\documentclass[preprint,12pt]{elsarticle}

\usepackage{amsthm}
\usepackage{amsmath,amssymb,amsfonts}
\usepackage{hyperref}

\usepackage{algorithm}
\usepackage{algpseudocode}

\input{math_commands.tex} 
\usepackage{xcolor}
\usepackage{subfigure}
\usepackage{caption}
\usepackage{booktabs}
\usepackage{multirow}
\newcommand\Tstrut{\rule{0pt}{2.2ex}}

\usepackage{wrapfig} 
\usepackage{float}

\journal{Artificial Intelligence}

\makeatletter
\def\ps@pprintTitle{%
 \let\@oddhead\@empty
 \let\@evenhead\@empty
 \def\@oddfoot{\hbox{\textit{\small To appear in Artificial Intelligence Journal}}}%
 \let\@evenfoot\@oddfoot}
\makeatother

\begin{document}

\theoremstyle{plain}
\newtheorem{theorem}{Theorem}
\newtheorem{proposition}{Proposition}
\newtheorem{lemma}{Lemma}
\newtheorem{corollary}{Corollary}
\theoremstyle{definition}
\newtheorem{definition}{Definition}
\newtheorem{assumption}{Assumption}
\theoremstyle{remark}
\newtheorem{remark}{Remark}

\begin{frontmatter}



\title{Open-World Continual Learning: Unifying Novelty Detection and Continual Learning}

\author[label1]{Gyuhak Kim\fnref{labela}}
\fntext[labela]{Equal contribution}
\ead{gkim87@uic.edu}
\affiliation[label1]{organization={University of Illinois Chicago},
            addressline={851 S Morgan St},
            city={Chicago},
            state={Illinois},
            postcode={60607},
            country={United States}}

\author[label2]{Changnan Xiao\fnref{labela}}
\ead{changnanxiao@gmail.com}
\affiliation[label2]{organization={ByteDance},
            addressline={Building 24, Zone B, 1999 Yishan Road},
            city={Shanghai},
            postcode={201100},
            country={China}}

\author[label1]{Tatsuya Konishi\fnref{labelb}}
\fntext[labelb]{The work was done when this author was visiting Bing Liu's group at the University of Illinois Chicago. The original affiliation and address of this author are \textit{KDDI Research, Inc., 2-1-15 Ohara, Fujimino-shi, Saitama, 356-8502, Japan}}
\ead{tt-konishi@kddi.com}

\author[label1]{Zixuan Ke}
\ead{zke4@uic.edu}

\author[label1]{Bing Liu\corref{cor1}}
\ead{liub@uic.edu}
\cortext[cor1]{Corresponding Author}

\begin{abstract}
As AI agents are increasingly used in the real open world with unknowns or novelties, they need the ability to (1) recognize objects that (a) they have learned before and (b) detect items that they have never seen or learned, and (2) learn the new items incrementally to become more and more knowledgeable and powerful. (1) is called \textit{novelty detection} or \textit{out-of-distribution} (OOD) \textit{detection} and (2) is called \textit{class incremental learning} (CIL), which is a setting of \textit{continual learning} (CL). 
In existing research, OOD detection and CIL are regarded as two completely different problems. This paper first provides a theoretical proof that good OOD detection for each task within the set of learned tasks (called \textit{closed-world OOD detection}) is \textit{necessary} for successful CIL. We show this by decomposing CIL into two sub-problems: \textit{within-task prediction} (WP) and \textit{task-id prediction} (TP), and proving that TP is correlated with closed-world OOD detection. The \textit{key theoretical result} is that regardless of whether WP and OOD detection (or TP) are defined explicitly or implicitly by a CIL algorithm, good WP and good closed-world OOD detection are \textit{necessary} and \textit{sufficient} conditions for good CIL, which unifies novelty or OOD detection and continual learning (CIL, in particular). We call this traditional CIL the \textit{closed-world CIL} as it does not detect future OOD data in the open world. The paper then proves that the theory can be generalized or extended to \textit{open-world CIL}, which is the proposed \textit{open-world continual learning}, that can perform CIL in the open world and detect future or open-world OOD data. 
Based on the theoretical results, new CIL methods are also designed, which outperform strong baselines in CIL accuracy and in continual OOD detection by a large margin.
\end{abstract}






\end{frontmatter}


\section{Introduction}
\label{sec.intro}
The current dominant machine learning paradigm (ML) makes the~\textit{closed-world assumption}, which means that the classes of objects seen 
by the system in testing or deployment must have been seen during training~\citep{scholkopf1999support,scheirer2014probability,bendale2015towards,fei2016learning,liu2020learning}, i.e., there is nothing novel occurring during testing or deployment. This assumption is invalid in practice as the real environment is an \textbf{open world} that is full of unknowns or novel objects.
To make an AI agent thrive in the open world, it has to detect novelties and learn them incrementally to make the system more knowledgeable and adaptable over time. This process involves multiple activities, such as \textit{novelty/OOD detection}, \textit{novelty characterization}, \textit{adaption}, \textit{risk assessment}, and \textit{continual learning} of the detected novel items or objects~\citep{liu2023ai,liu2022self}.
Novelty detection, also called \textit{out-of-distribution} (OOD) \textit{detection}, 
aims to detect unseen objects that the agent has not learned. On detecting novel objects or situations, the agent has to respond or adapt its actions. But in order to adapt, it must first \textit{characterize} the novel object as without it, the agent would not know how to respond or adapt. For example, it may characterize a detected novel object as looking like a dog. Then, the agent may react like it would react to a dog. In the process, the agent also constantly assesses the risk of its actions. Finally, it also learns to recognize the new object incrementally so that it will not be surprised when it sees the same kind of object in the future. This incremental learning is called \textit{continual learning} (CL) or \textit{lifelong learning}~\citep{Chen2018lifelong,ke2022continual}. Note that before learning, the agent must obtain labeled training data, which can be collected by the agent through interaction with the environment or human users. This aspect is out of the scope of this paper. See \citep{liu2023ai,liu2022self} for details. 

This paper focuses only on the key learning aspects of the open world scenario: (1) OOD/novelty detection and (2) continual learning, more specifically \textit{class incremental learning} (CIL) (see the definition below). In the research community, (1) and (2) are regarded as two completely different problems, but this paper theoretically unifies them by proving that good OOD detection for each task within the set of learned tasks, which we call \textit{closed-world OOD detection}, is in fact \textit{necessary} for CIL. 
Below, we define the concepts of OOD detection and continual learning. 

\vspace{2mm}
\textit{\textbf{Out-of-distribution (OOD) detection}}: Given the training data
$\mathcal{D}=\{(x^i, y^i)_{i=1}^{n}\}$, where $n$ is the number of data samples, and $x^i \in \mathbf{X}$ is an input sample and $y^i \in \mathbf{Y}$ (the set of all class labels in $\mathcal{D}$) is its class label, our goal is to build a classifier $f : \mathbf{X} \rightarrow \mathbf{Y} \cup \{O\}$ that can detect test instances that do not belong to any classes in $\mathbf{Y}$ (called \textit{OOD detection}), which are assigned to the class $O$. $\mathbf{Y}$ is often called the \textit{in-distribution} (IND) \textit{classes}.
\vspace{2mm}

We also call this \textit{open-world OOD detection}. As we can see from the definition, an OOD detection algorithm can also classify test instances belonging to $\mathbf{Y}$ to their respective classes, which is called IND classification, although most OOD detection papers do not report the IND classification results.

Continual learning (CL) aims to incrementally learn a sequence of tasks. Each task consists of a set of classes to be learned together (the set may contain only a single class). Once a task is learned, its training data (at least a majority of it) is no longer accessible. Thus, unlike multitask learning, in learning a new task, CL will not be able to use the data of the previous tasks.  A major challenge of CL is \textit{catastrophic forgetting} (CF), which refers to the phenomenon that in learning a new task, the neural network model parameters need to be modified, which may corrupt the knowledge learned for previous tasks in the network and cause performance degradation for the previous tasks~\citep{McCloskey1989}.  Although many CL techniques have been proposed, they are mainly empirical. Limited theoretical work has been done on how to solve CL. This paper performs such a theoretical study about the necessary and sufficient conditions for effective CL. Two main CL settings have been extensively studied: \textit{class incremental learning} (CIL) and \textit{task incremental learning} (TIL)~\citep{van2019three}. In CIL, the learning process builds a single classifier for all tasks/classes  
learned so far. In testing, a test instance from any class may be presented for the model to classify. No prior task information (e.g., task-id) of the test instance is provided. Formally, CIL is defined as follows.  

\vspace{2mm}
\textit{\textbf{Class incremental learning}} (CIL). CIL learns a sequence of tasks, $1, 2, \cdots$. Let $T$ be the number of tasks that have been learned so far. Each task $k$ ($1 \le k \le T$) has a training dataset
$\mathcal{D}_k=\{(x_k^i, y_k^i)_{i=1}^{n_k}\}$, where $n_k$ is the number of data samples in task $k$, and $x_k^i \in \mathbf{X}$ is an input sample and $y_k^i \in \mathbf{Y}_k$ (the set of all classes of task $k$) is its class label. All $\mathbf{Y}_k$'s are disjoint ($\mathbf{Y}_k \cap \mathbf{Y}_{k'} = \emptyset,\, \forall k \neq k'$) and $\bigcup_{k=1}^T \mathbf{Y}_k = \mathbf{Y}$. 
The goal of CIL is to construct a single predictive function or classifier $f : \mathbf{X} \rightarrow \mathbf{Y}$ that can identify the class label $y$ of each given test instance $x$ from the $T$ tasks.
\vspace{2mm}

Based on CIL, we can also define the term \textit{close-world OOD detection}. 

\vspace{2mm}
\textit{\textbf{Closed-world OOD detection}}: \textit{Closed-world OOD detection} for a given task $k$ among the $T$ tasks that have been learned so far is \textit{OOD detection} regarding the classes of task $k$ as the IND classes and those of the other $T-1$ tasks as the OOD classes. 
\vspace{2mm}

From now on when we refer to \textit{OOD detection} on its own (which is \textit{open-world OOD detection}), we mean 
it is not limited to the $T$ learned tasks, as opposed to the \textit{closed-world OOD detection}. Clearly, (open-world) OOD detection encompasses closed-world OOD detection, but not vice versa.

Unlike CIL, each task in TIL is a separate or independent classification problem. {For example, one task could be to classify different breeds of dogs and another task could be to classify different types of animals (the tasks may not be disjoint).} One model is built for each task in a shared network. In testing, the task-id of each test instance is provided and the system uses only the specific model for the task (dog or animal classification) to classify the test instance. Formally, TIL is defined as follows.

\vspace{2mm}
\textit{\textbf{Task incremental learning}} (TIL). TIL learns a sequence of tasks, $1, 2, \cdots$. Let $T$ be the number of tasks that have been learned so far. Each task $k$ ($1 \le k \le T$) has a training dataset 
$\mathcal{D}_k=\{((x_k^i, k), y_k^i)_{i=1}^{n_k}\}$, 
where $n_k$ is the number of data samples in task $k \in \mathbf{T} = \{1, 2, ..., T\}$, and $x_k^i \in \mathbf{X}$ is an input sample and $y_k^i \in \mathbf{Y}_k \subset \mathbf{Y}$ is its class label.
The goal of TIL is to construct a predictor $f: \mathbf{X} \times \mathbf{T} \rightarrow \mathbf{Y}$ to identify the class label $y \in \mathbf{Y}_k$ for $(x, k)$ (the given test instance $x$ from task $k$). 
\vspace{2mm}

This paper focuses on CIL, which involves incrementally learning new or novel object classes—a key aspect of open-world learning. While the proposed methods are also applicable to TIL, we do not address it in this paper. TIL is generally simpler, and several existing techniques can achieve it without CF~\citep{Serra2018overcoming,supsup2020}. In contrast, CIL remains highly challenging due to the difficulty of Inter-task Class Separation (ICS), i.e., establishing decision boundaries between classes from the new task and those from previous tasks in learning the new task without accessing the training data of previous tasks.

\vspace{2mm}
\textbf{Problem Statement} (\textit{open-world continual learning}): \textit{Open-world continual learning} (OWCL) is defined as CIL with the OOD detection capability. We also call it \textit{open-world CIL} or CIL$^+$. At any time, the resulting open-world CIL model can classify test instances belonging to the classes in the $T$ tasks that have been learned so far to their respective classes and also detect OOD instances that do not belong to any of the learned classes so far. 
\vspace{2mm}

Note that OOD detection in CIL$^+$ is different from traditional OOD detection (which sees the full IND data together) because, in CIL$^+$, the model does not see all the IND data together. Instead, the IND data comes in a sequence of tasks incrementally, and in learning each task, the model does not see any data (or only a very small sample) of the old or previous tasks. 

\vspace{2mm}
\textbf{Main contributions:} This paper makes three main contributions. First, it theoretically proves the necessary and sufficient conditions for solving the CIL problem. A good closed-world OOD detection performance is one of the necessary conditions, which connects or unifies OOD detection and CIL. Since in this traditional CIL, the test instances are assumed to be from one of the $T$ tasks that have been learned, we call the existing CIL the \textit{closed-world CIL}. Second, we prove that the theory can naturally be generalized or extended to the \textit{open-world CIL}, which is the proposed \textit{open-world continual learning}. Open-world CIL can perform CIL in the open world and detect OOD test data that do not belong to any of the $T$ tasks learned so far. Third, based on the theory, several new CIL algorithms are designed, which are also able to detect novel (or OOD) instances for the open-world continual learning (OWCL) setting. Note that from here onward, when we do not explicitly say open-world CIL, CIL means the traditional CIL.

\vspace{2mm}
\textbf{Theory.} We conduct a theoretical study of CIL, which is applicable to any CIL classification model. Instead of focusing on the traditional PAC generalization bound~\citep{pentina2014pac} or neural tangent kernel (NTK)~\citep{karakida2022learning}, we focus on how to solve the CIL problem. We first decompose the CIL problem into two sub-problems in a probabilistic framework: \textit{Within-task Prediction} (WP) and \textit{Task-id Prediction} (TP). WP means that the prediction for a test instance is only made within the classes of the task to which the test instance belongs, which is basically the TIL problem. TP predicts the task-id. TP is needed because, in CIL, task-id is not provided at test time. This paper then proves based on cross-entropy loss that (i) the CIL performance is bounded by WP and TP performances, and (ii) TP and closed-world OOD detection performance bound each other. This paper further generalizes the result to open-world CIL (or CIL$^+$). These results unify CIL and OOD detection.

\vspace{+2mm}
\noindent
\textbf{\textit{Key theoretical results}}: Regardless of whether WP and TP or OOD detection are defined explicitly or implicitly by a closed-world or open-world CIL algorithm, (1) good WP and good TP or closed-world OOD detection are
\textit{necessary} and \textit{sufficient} conditions for good closed-world CIL performances and (2) good WP and good TP or open-world OOD detection are
\textit{necessary} and \textit{sufficient} conditions for good open-world CIL performances.\footnote{This result applies to both batch/offline and online/stream CIL and to CIL problems with blurry task boundaries which means that some training data of a task may come later together with a future task.} 
\vspace{2mm}

The intuition of the theory is simple because if a closed-world or open-world CIL model is perfect at detecting OOD samples for each task, 
which solves the ICS problem, then closed-world or open-world CIL is reduced to WP, which is the traditional single-task supervised learning for each task. Note that many OOD detection algorithms can also perform IND classification, which is WP. 

\vspace{2mm}
\textbf{New CIL Algorithms for OWCL.} The theoretical result provides principled guidance for solving the (closed-world or open-world) 
CIL problem. 
Based on the theory,  several new CIL methods are designed. (1) The first few methods integrate a TIL method and an OOD detection method, which outperform strong baselines in both the CIL and TIL settings by a large margin. This combination is attractive because TIL has achieved no forgetting, and we only need a strong OOD detection technique that can perform both IND prediction and OOD detection to learn each task to achieve strong CIL results. We do not propose a new OOD detection method as there are numerous such methods in the literature. We use two existing ones. (2) Another method is based on a pre-trained model and an OOD replay technique, which performs even better, outperforming existing baselines markedly in both CIL and OOD detection in the OWCL setting.  

\section{Related Work}
\label{sec.related}

Although a large number of algorithms have been proposed to solve the CIL problem, they are mainly empirical. Two papers have focused on studying the traditional PAC generalization bound~\citep{pentina2014pac} or NTK~\citep{karakida2022learning}, but they do not tell how to solve the CL problem. This paper focuses on how to solve the CIL problem. To the best of our knowledge, we are not aware of any work that has proposed a theory on how to solve CIL. Also, none of the existing work has connected CIL and OOD detection. Our work shows that a good CIL algorithm can naturally perform OOD detection in the open world. Below, we first survey four popular families of CL approaches, which are mainly for overcoming catastrophic forgetting (CF). We then discuss related works about open-world learning. 

\textbf{Regularization-based methods} prevent forgetting by restricting the learned parameters for previous tasks from being modified significantly by using a regularization term to penalize such changes \citep{Kirkpatrick2017overcoming,Zenke2017continual} or to regularize the learned representations or outputs so that they are not far from those of the previously learned network~\citep{Li2016LwF,Zhu_2021_CVPR_pass}.

\textbf{Replay-based methods}~\citep{Rebuffi2017,chaudhry2018efficient,castro2018end_eeil,NEURIPS2020_b704ea2c_derpp,yan2021dynamically} mitigate forgetting by saving a small amount of training data from previous tasks in a memory buffer and jointly train the network using the current data and the previous task data saved in the memory. Some methods in this family also study which samples in the memory should be used in replaying~\citep{NIPS2019_9357_mir} or which samples in the training data should be saved for later replaying \citep{Rebuffi2017,Liu_2020_CVPR}.

\textbf{Generative methods} construct a generative network to generate raw training data~\citep{Shin2017continual,ostapenko2019learning,ayub2021eec}. The generated data are used with the current task training data to jointly train the classification network. \cite{Zhu_2021_CVPR_pass} generates features instead of raw data. The generated samples in these methods are used to prevent forgetting in both the generative network and the classification network.

\textbf{Parameter-isolation methods}~\citep{Serra2018overcoming,supsup2020} train a set of task-specific parameters to effectively protect the important parameters of each task from being updated, which thus has almost no forgetting. A limitation of the approach is that the correct task-id of each test instance must be known to the system to select the corresponding task-specific parameters at inference. These methods are thus mainly used for \textit{task incremental learning} (TIL). Some CIL methods also used these methods~\citep{abati2020conditional,von2019continual,rajasegaran2020itaml,henning2021posterior} and they have separate mechanisms to predict the task-id (more on this below). However, their CIL performances are far below that of recent replay-based counterparts (see Sections~\ref{sec.experiment1} and~\ref{sec.experiment2} for details).
Two of our proposed CIL methods also use two parameter-isolation methods (HAT \citep{Serra2018overcoming} and SupSup~\citep{supsup2020}) for TIL as one of the components. 

Using a TIL method for CIL means that CIL is decomposed into WP and TP. Task-id prediction (TP) is the key challenge. For example, CCG~\citep{abati2020conditional} constructs an additional network to predict the task-id. 
iTAML~\citep{rajasegaran2020itaml} identifies the task-id of the test data in a batch. A serious limitation of this is that it requires the test data in a batch to belong to the same task. Our methods are different as they can predict for a single test instance at a time. HyperNet~\citep{von2019continual} and PR~\citep{henning2021posterior} propose an entropy-based task-id prediction method. SupSup~\citep{supsup2020} predicts task-id by finding the optimal superpositions at inference. However, these methods perform poorly because they either do not know that OOD detection is the key for accurate task-id prediction or their task models are not built for OOD detection. It is also important to note that our theory does not explicitly predict task-id. Instead, it uses the TP probability and WP probability for test prediction. 

Several papers have explicitly or implicitly indicated the use of OOD detection for task-id prediction in continual learning. For example, the CIL method in \cite{hu2021continual} is based on one-class classification, which is OOD detection with only a single class as the in-distribution (IND) class. 
In \cite{henning2021posterior}, the authors proposed an uncertainty-based OOD detection framework for task-id prediction. Two specific methods were presented. One uses \textit{entropy} to quantify the uncertainty (which has also been used in some other systems discussed above) and the other is called \textit{agree}, which selects the task that leads to the highest agreement in predictions across task models. There are also related works that did not explicitly make a connection between CIL and OOD detection, their methods implicitly imply it. For example, \cite{yan2021dynamically} uses a regularization similar to OOD detection, which employs the replay data from previous tasks as OOD samples. \cite{van2021class} proposed to train a VAE model for each class to be learned. It then estimates the likelihood $p(x|y)$ and uses the Bayes rule to predict the class ($y$) of each test instance ($x$). Our work makes a theoretical contribution by formally connecting CIL and OOD detection and proving that for a good CIL performance, a good OOD detection capability for each task is \textit{necessary}.

\textit{Open world learning} has been studied by many researchers~\citep{scholkopf1999support,scheirer2014probability,bendale2015towards,fei2016learning,liu2020learning,liu2022self}. However, the existing research mainly focused on novelty detection, also called \textit{open set recognition} or \textit{out-of-distribution} (OOD) \textit{detection}. 
Some researchers have also studied learning the novel objects after they are detected and manually labeled \citep{bendale2015towards,fei2016learning,xu2019open,jafarzadeh2020review}. However, none of them perform continual learning, which has additional challenges of catastrophic forgetting (CF) and inter-task class separation (ICS). Several researchers also studied other related tasks in addition to novelty detection, e.g., characterization of novelties and adaptation of novelties to maximize the performance task~\cite{loyall2022integrated,thai2022architecture}. Again, these works are not about continual learning. 
Excellent surveys of novelty detection or OOD detection and open-world learning can be found in~\citep{yang2021generalized,pang2021deep,parmar2021open,jafarzadeh2020review}. \cite{gummadi2022shels} did novelty detection and also continual learning, but its continual learning uses the regularization-based method. It is quite weak because it has serious forgetting. A position paper~\citep{langley2020open} recently presented some nice blue-sky ideas about open-world learning, but it does not propose or implement any algorithm. 

Our proposed algorithms are quite different. 
In training, based on our theory, we use two existing OOD detection methods to verify that our theory can guide us to design new and much more effective CIL algorithms. In testing, our OOD detection is in the open-world continual learning (OWCL) setting, which has been described in the introduction section.

Several researchers have studied \textit{novel class discovery} \citep{Han2019learning}, which is defined as discovering the hidden classes in the detected novel or OOD instances. Our work does not perform this function. We assume that the training data for each new task is given. Performing automatic class label discovery is still very challenging as in many cases, the class assignments can be subjective and are determined by human users. For example, for a dog, whether it should just be labeled as a dog or a specific breed of dog is a subjective decision and depends on applications. 

Some existing works have combined OOD detection and continual learning \cite{gummadi2022shels,he2022out,rios2022incdfm}. These papers use OOD score thresholds to determine OOD instances and also do continual learning afterward. However, their continual learning still assumes that the training data are given as it is hard to do real-time detection and learning. This is because, without human involvement, it is impossible to obtain novel class labels in general and verify the correctness of OOD detection results. Any error in OOD detection will propagate to the continual learning phase. \cite{liu2023ai,liu2022self} reported a continual learning chatbot that can detect novel user utterances that the system does not understand and chat with the user through its novelty characterization mechanism to get the ground truth. However, this system is based on saving new/novel utterances and performing matching and retrieval in subsequent chatting. \cite{loyall2022integrated} proposed an integrated online architecture that combines and extends probabilistic programming and planning to (1) detect novelty, (2) incrementally characterize the novelty, and (3) continually adapt its task-based reasoning to the evolving understanding of the novelty to maximize task performance. However, this work is not about continual learning. \cite{palash2023continuous} also reported a system for continuous emotion novelty detection. 

\section{A Theoretical Study on Solving CIL}\label{sec.theorem}
This section presents our theory for solving CIL, which also covers novelty or OOD detection. 
It first shows that the CIL performance improves if the within-task prediction (WP) performance and/or the task-id prediction (TP) performance improve, and then shows that TP and OOD detection bound each other, which indicates that CIL performance is controlled by WP and OOD detection. This connects CIL and OOD detection. After that, we study the necessary conditions for a good CIL model, which includes a good WP, and a good TP or OOD detection. In the first four sub-sections, we focus on the traditional CIL that is limited to the number of tasks $T$ that have been learned so far, which we also call \textit{closed-world CIL}. OOD detection in this context is also within the $T$ learned tasks and is called \textit{closed-world OOD detection} (see Section~\ref{sec.intro}). For simplicity in presentation, we will not add \textit{closed-world} before CIL or OOD detection below. In Section~\ref{sec.ow-CIL}, we generalize/extend the theory to open-world CIL or open-world continual learning, which will also detect OOD data that do not belong to any of the $T$ tasks learned so far. Table~\ref{Tab:acronyms} gives the list of acronyms used in the paper.

\begin{table}[h]
\centering
\caption{Acronyms used in the paper}
\resizebox{0.4\columnwidth}{!}{
\begin{tabular}{l c}
\toprule
CL & Continual learning \\
CIL & Class incremental learning \\
TIL & Task incremental learning \\
OOD & Out-of-distribution \\
IND & In-distribution \\
WP & Within-task prediction \\
TP & Task-id prediction \\
CF & Catastrophic forgetting \\
OWCL & Open-world continual learning \\
NTK & Neural tangent kernel \\
AUC & Area Under the ROC Curve \\
\bottomrule
\end{tabular}
}
\label{Tab:acronyms}
\vspace{-2mm}
\end{table}

\subsection{CIL Problem Decomposition} \label{sec:decomposition}
This sub-section first presents the assumptions made by CIL based on its definition and then proposes a decomposition of the CIL problem into two sub-problems. Assume that a CIL system has learned a sequence of $T$ tasks $\{(\mathbf{X}_k, \mathbf{Y}_k)\}_{k=1,\dots,T}$ so far, where $\mathbf{X}_{k}$ is the domain of task $k$ and $\mathbf{Y}_k$ are classes of task $k$ as $\mathbf{Y}_k = \{\mathbf{Y}_{k, j}\}$, 
where $j$ indicates the $j$th class in task $k$. 
Let $\mathbf{X}_{k, j}$ to be the domain of $j$th class of task $k$, where $\mathbf{X}_{k} = \bigcup_j \mathbf{X}_{k, j}$. 
For accuracy, we will use $x \in \mathbf{X}_{k, j}$ instead of $\mathbf{Y}_{k, j}$ in probabilistic analysis. 
Based on the definition of class incremental learning (CIL) (Section~\ref{sec.intro}), the following assumptions are implied,
\begin{assumption}
    The domains of classes of the same task are disjoint, i.e., $\mathbf{X}_{k, j} \cap \mathbf{X}_{k, j'} = \emptyset,\, \forall j \neq j'$.
\end{assumption}
\begin{assumption}
    The domains of tasks are disjoint, i.e., $\mathbf{X}_k \cap \mathbf{X}_{k'} = \emptyset,\, \forall k \neq k'$.
\end{assumption}
For any ground event $D$, the goal of a CIL problem is to learn $\mathbf{P}(x \in \mathbf{X}_{k, j} | D)$. This can be decomposed into two probabilities, \textit{within-task IND prediction} (WP) probability and \textit{task-id prediction} (TP) probability. WP probability is
$\mathbf{P} (x \in \mathbf{X}_{k, j} | x \in \mathbf{X}_{k}, D)$ and TP probability is $\mathbf{P}(x \in \mathbf{X}_k | D)$. 
We can rewrite the CIL problem using WP and TP based on the two assumptions, 
\begin{align}
     \mathbf{P}(x \in \mathbf{X}_{k_0, j_0} | D)
     &= \sum_{k=1,\dots,n} \mathbf{P} (x \in \mathbf{X}_{k, j_0} | x \in \mathbf{X}_k, D) \mathbf{P}(x \in \mathbf{X}_k | D) \label{eq:prob_sum} \\
     &= \mathbf{P} (x \in \mathbf{X}_{k_0, j_0} | x \in \mathbf{X}_{k_0}, D) \mathbf{P}(x \in \mathbf{X}_{k_0} | D) \label{eq:cil_in_til_and_tp}
\end{align}
where $k_0$ means a particular task and $j_0$ is a particular class in the task. 

Some remarks are in order about Eq.~\ref{eq:cil_in_til_and_tp} and our subsequent analysis to set the stage. 

\begin{remark}\label{remark1}
Eq.~\ref{eq:cil_in_til_and_tp} shows that if we can improve either the WP or TP performance, or both, we can improve the CIL performance.   
\end{remark}
\begin{remark} \label{remark:alg}
It is important to note that our theory is not concerned with the learning algorithm or training process. But we will propose some concrete CIL algorithms based on the theoretical result in the experiment section. 
\end{remark}
\begin{remark}\label{remark:disjoint}
We note that the CIL definition and the subsequent analysis apply to tasks with any number of classes (including only one class per task) and to online CIL where the training data for each task or class comes gradually in a data stream and may also cross task boundaries (blurry tasks~\citep{bang2021rainbow}) because our analysis is based on an already-built CIL model after training. Regarding blurry task boundaries, suppose dataset 1 has classes \{dog, cat, tiger\} and dataset 2 has classes \{dog, computer, car\}. We can define task 1 as \{dog, cat, tiger\} and task 2 as \{computer, car\}. The shared class \textit{dog} in dataset 2 is just additional training data of \textit{dog} appeared after task 1. 
\end{remark}
\begin{remark}\label{remark:meaningofeq2}
CIL = WP * TP in Eq.~\ref{eq:cil_in_til_and_tp} means that when we have WP and TP (defined either explicitly or implicitly by implementation), we can find a corresponding CIL model defined by WP * TP. Similarly, when we have a CIL model, we can find the corresponding underlying WP and TP defined by their probabilistic definitions.
\end{remark}
In the following sub-sections, we develop this further concretely to derive the sufficient and necessary conditions for solving the CIL problem in the context of cross-entropy loss as it is used in almost all supervised CIL systems.

\subsection{CIL Improves as WP and/or TP Improve}\label{sec:cil_improve_by_til_and_tp}

As stated in Remark~\ref{remark:alg} above, the study here is based on a \textit{trained CIL model} and not concerned with the algorithm used in training the model. We use cross-entropy as the performance measure of a trained model as it is the most popular loss function used in supervised CL. For experimental evaluation, we use \textit{accuracy} following CL papers. Denote the cross-entropy of two probability distributions $p$ and $q$ as 
\begin{align}
    H(p, q) \overset{def}{=} - \mathbb{E}_p [\log q] = - \sum_i p_i \log q_i.
\end{align}
For any $x \in \mathbf{X}$, let $y$ to be the CIL ground truth label of $x$, where $y_{k_0, j_0} = 1$ if $x \in \mathbf{X}_{k_0, j_0}$ otherwise $y_{k, j} = 0$, $\forall (k, j) \neq (k_0, j_0)$.
Let $\Tilde{y}$ be the WP ground truth label of $x$, where $\Tilde{y}_{k_0, j_0} = 1$ if $x \in \mathbf{X}_{k_0, j_0}$ otherwise $\Tilde{y}_{k_0, j} = 0$, $\forall j \neq j_0$.
Let $\Bar{y}$ be the TP ground truth label of $x$, where $\Bar{y}_{k_0} = 1$ if $x \in \mathbf{X}_{k_0}$ otherwise $\Bar{y}_k = 0$, $\forall k \neq k_0$.
Denote
\begin{align}
    H_{WP} (x) &= H(\Tilde{y}, \{\mathbf{P}(x \in \mathbf{X}_{k_0, j} | x \in \mathbf{X}_{k_0}, D)\}_{j}), \\
    H_{CIL} (x) &= H(y, \{\mathbf{P}(x \in \mathbf{X}_{k, j} | D)\}_{k, j}), \\
    H_{TP} (x) &= H(\Bar{y}, \{\mathbf{P}(x \in \mathbf{X}_k | D)\}_{k})
\end{align}
where $H_{WP}$, $H_{CIL}$, and $H_{TP}$ are the cross-entropy values of WP, CIL, and TP, respectively.
We now present our first theorem. The theorem connects CIL to WP and TP and suggests that by having a good WP or TP, the CIL performance improves as the upper bound for the CIL loss decreases.
\begin{theorem}
\label{thm:ce}
If $H_{TP}(x) \leq \delta$ and $H_{WP}(x) \leq \epsilon$, we have
$
H_{CIL} (x) \leq \epsilon + \delta.
$ 
\end{theorem}

The detailed proof is given in \ref{prf:ce}. This theorem holds regardless of whether WP and TP are trained together or separately.
When they are trained separately, if WP is fixed and we let  $\epsilon=H_{WP}(x)$,  $H_{CIL}(x) \leq H_{WP}(x) + \delta$, which means if TP is better, CIL is better. 
Similarly, if TP is fixed, we have $H_{CIL}(x) \leq \epsilon + H_{TP} (x)$.
When they are trained concurrently, there exists a functional relationship between $\epsilon$ and $\delta$ depending on implementation. 
But no matter what it is, when $\epsilon + \delta$ decreases, CIL gets better.

Theorem~\ref{thm:ce} holds for any $x \in \mathbf{X} = \bigcup_k \mathbf{X}_k$ that satisfies $H_{TP} (x) \leq \delta$ or $H_{WP} (x) \leq \epsilon$. To measure the overall performance under expectation, we present the following corollary.
\begin{corollary}\label{cor:expectation}
Let $U(\mathbf{X})$ represent the uniform distribution on $\mathbf{X}$. i) If $\mathbb{E}_{x \sim U(\mathbf{X})} [H_{TP}(x)] \leq \delta$, then $\mathbb{E}_{x\sim U(\mathbf{X})} [H_{CIL} (x)] \leq \mathbb{E}_{x \sim U(\mathbf{X})} [H_{WP} (x)] + \delta$. Similarly, ii) $\mathbb{E}_{x \sim U(\mathbf{X})} [H_{WP} (x)] \leq \epsilon$, then $\mathbb{E}_{x\sim U(\mathbf{X})} [H_{CIL} (x)] \leq \epsilon + \mathbb{E}_{x \sim U(\mathbf{X})} [H_{TP}(x)]$.
\end{corollary}
The proof is given in \ref{prf:expectation}. The corollary is a direct extension of Theorem~\ref{thm:ce} in expectation. The implication is that given TP performance, CIL is positively related to WP. The better the WP is, the better the CIL is as the upper bound of the CIL loss decreases. Similarly, given WP performance, a better TP performance results in a better CIL performance. Due to the positive relation, we can improve CIL by improving either WP or TP using their respective methods developed in each area.

\subsection{Task Prediction (TP) to OOD Detection} 

Building on Eq.~\ref{eq:cil_in_til_and_tp}, we have studied the relationship of CIL, WP, and TP in Theorem \ref{thm:ce}. We now connect TP and OOD detection. They are shown to be dominated by each other to a constant factor.

We again use cross-entropy $H$ to measure the performance of TP and OOD detection of a trained network as in Section \ref{sec:cil_improve_by_til_and_tp}. 
To build the connection between $H_{TP}(x)$ and OOD detection of each task, we first define the notations of OOD detection. We use $\mathbf{P}'_k (x \in \mathbf{X}_k | D)$ to represent the probability distribution predicted by the $k$th task's OOD detector. Notice that the task prediction (TP) probability distribution $\mathbf{P} (x \in \mathbf{X}_k | D)$ is a categorical distribution over $T$ tasks, while the OOD detection probability distribution $\mathbf{P}'_k (x \in \mathbf{X}_k | D)$ is a Bernoulli distribution. For any $x \in \mathbf{X}$, define 
\begin{align}
    H_{OOD, k} (x) = \left\{ 
    \begin{aligned}
    H(1, \mathbf{P}'_k (x \in \mathbf{X}_k | D)) =& - \log \mathbf{P}'_k (x \in \mathbf{X}_k | D),& &x \in \mathbf{X}_k, \\
    H(0, \mathbf{P}'_k (x \in \mathbf{X}_k | D)) =& - \log \mathbf{P}'_k (x \in \mathbf{X} \backslash \mathbf{X}_k | D),& &x \in \mathbf{X} \backslash \mathbf{X}_k.
    \end{aligned}
    \right.
\label{eq:h_ood}
\end{align}
Note that the OOD detection here is the \textit{closed-world OOD detection}. But for presentation simplicity, we still use OOD detection below. In CIL, the term OOD detection probability for a task can be defined using the output values corresponding to the classes of the task. Some examples of the function are the sigmoid of maximum logit value and the maximum softmax probability after re-scaling to 0 to 1. 
{It is also possible to define the OOD detector directly as a function of tasks instead of a function of the output values of all classes of tasks, i.e. Mahalanobis distance.} The following theorem shows that TP and OOD detection bound each other. 
\begin{theorem}
\label{thm:tp_ood}
i) If $H_{TP} (x) \leq \delta$, let $\mathbf{P}'_k (x \in \mathbf{X}_k | D) = \mathbf{P} (x \in \mathbf{X}_k | D)$, then $H_{OOD, k} (x) \leq \delta, \forall\, k = 1, \dots, T$. ii) If $H_{OOD, k} (x) \leq \delta_k, k=1,\dots,T$, let $\mathbf{P} (x \in \mathbf{X}_k | D) = \frac{\mathbf{P}_k' (x \in \mathbf{X}_k |D)}{\sum_k \mathbf{P}_k' (x \in \mathbf{X}_k |D)}$, then $H_{TP} (x) \leq (\sum_k \mathbf{1}_{x \in \mathbf{X}_k} e^{\delta_{k}}) (\sum_k 1 - e^{-\delta_k})$, where $\mathbf{1}_{x \in \mathbf{X}_k}$ is an indicator function.
\end{theorem}
See \ref{prf:tp_ood} for the proof. 
As we use cross-entropy, the lower the bound, the better the performance is. The first statement (i) says that the OOD detection performance improves if the TP performance gets better (i.e., lower $\delta$). Similarly, the second statement (ii) says that the TP performance improves if the OOD detection performance on each task improves (i.e., lower $\delta_k$). Besides, since $(\sum_k \mathbf{1}_{x \in \mathbf{X}_k} e^{\delta_{k}}) (\sum_k 1 - e^{-\delta_k})$ converges to $0$ as $\delta_k$'s converge to $0$ in order of $O(|\sum_k \delta_k|)$, we further know that $H_{TP}$ and $\sum_k H_{OOD,k}$ are equivalent in quantity up to a constant factor.

For the traditional CIL, Theorem \ref{thm:ce} studied how CIL is related to WP and TP. Theorem \ref{thm:tp_ood} showed that TP and OOD detection bound each other. Now we explicitly give the upper bound of CIL in relation to WP and OOD detection of each task. The detailed proof can be found in \ref{prf:cil_with_op_and_ood}. 
\begin{theorem}
\label{thm:cil_with_op_and_ood}
If $H_{OOD, k}(x) \leq \delta_k,\, k = 1, \dots, T$ and $H_{WP}(x) \leq \epsilon$, we have
$$
H_{CIL} (x) \leq \epsilon + (\sum_k \mathbf{1}_{x \in \mathbf{X}_k} e^{\delta_{k}}) (\sum_k 1 - e^{-\delta_k}),$$ 
where $\mathbf{1}_{x \in \mathbf{X}_k}$ is an indicator function.
\end{theorem}

\subsection{Necessary Conditions for Improving CIL}

In Theorem \ref{thm:ce}, we showed that good performances of WP and TP are sufficient to guarantee a good performance of CIL.
In Theorem \ref{thm:cil_with_op_and_ood}, we showed that good performances of WP and  OOD detection are sufficient to guarantee a good performance of CIL. Again, for simplicity, OOD detection here refers to the closed-world OOD detection.
For completeness, we study the necessary conditions of a well-performed CIL in this sub-section.

\begin{theorem}
\label{thm:necessary_condition}
If $H_{CIL} (x) \leq \eta$, then there exist
i) a WP, s.t. $H_{WP} (x) \leq \eta$, 
ii) a TP, s.t. $H_{TP} (x) \leq \eta$, and
iii) an OOD detector for each task, s.t. $H_{OOD, k} \leq \eta,\, k = 1, \dots, T$. 
\end{theorem}

The detailed proof is given in \ref{prf:necessary_condition}. This theorem 
tells that if a good CIL model is trained, then a good WP, a good TP, and a good OOD detector for each task are always implied. 
More importantly, by transforming Theorem \ref{thm:necessary_condition} into its contraposition, we have the following statements:
If for any WP, $H_{WP} (x) > \eta$, then $H_{CIL} (x) > \eta$.
If for any TP, $H_{TP} (x) > \eta$, then $H_{CIL} (x) > \eta$.
If for any OOD detector, $H_{OOD, k} (x) > \eta,\, k=1,\dots,T$, then $H_{CIL} (x) > \eta$.
Regardless of whether WP and TP (or OOD detection) are defined explicitly or implicitly by a CIL algorithm, 
the existence of a good WP and the existence of a good TP or 
OOD detection are necessary conditions for good CIL performance. Note that the OOD detection here is closed-world OOD detection.

\begin{remark}
It is important to note again that our study in this section is based on a CIL model that has already been built. In other words, our study tells the CIL designers what should be achieved in the final model. Clearly, one would also like to know how to design a strong CIL model based on the theoretical results, which also considers catastrophic forgetting (CF). One effective method is to make use of a strong existing TIL algorithm, which can already achieve no or little forgetting (CF), and combine it with a strong OOD detection algorithm (as mentioned earlier, most OOD detection methods can also perform WP). Thus, any improved method from the OOD detection community can be applied to CIL to produce improved CIL systems (see Sections~\ref{sec.betterOOD} and \ref{sec.HAT+CSI}).
\end{remark}

\subsection{Generalization to Open-World Continual Learning}
\label{sec.ow-CIL}
As mentioned at the beginning of this section, the first four subsections focused on the traditional closed-world CIL. This subsection generalizes or extends the theory to the open-world CIL, denoted by CIL$^+$. CIL$^+$ is CIL with an additional pseudo-task on top of the $T$ learned tasks representing OOD detection beyond the $T$ tasks, which we call the \textit{OOD task} with a single pseudo-class (called \textit{OOD class}) as we cannot predict the unseen class of an OOD sample because it is unknown. In this context, OOD detection is referred to as \textit{open-world OOD detection}. For simplicity, we will continue using the term OOD detection.

We first note that Eq.~\ref{eq:cil_in_til_and_tp} still applies because CIL$^+$ only adds a new OOD task with one OOD class.  Theorem~\ref{thm:ce} for the closed-world CIL can be extended to the open-world CIL (CIL$^+$) by replacing $H_{TP}$ with $H_{TP^+}$, and $H_{CIL}$ with $H_{CIL^+}$. $H_{WP}$ stays the same as the WP definition has no change in CIL$^+$. 
The proof is trivially identical to the proof of Theorem~\ref{thm:ce}. 
The key extension is to Theorem~\ref{thm:tp_ood} of the traditional \textit{closed-world} CIL so that test samples that do not belong to any of the $T$ already-learned tasks (i.e., OOD to the $T$ tasks) can also be detected. 

Theorem~\ref{thm:tp_ood} can be generalized to CIL$^+$ by changing the closed-world TP to \textit{open-world TP}, denoted by TP$^+$, which must now predict the additional OOD task.

We denote $\mathbf{X}^+$ as the open-world OOD domain beyond $\mathbf{X}$. For any $x \in \mathbf{X} \cup \mathbf{X}^+$, define 
\begin{align}
 H_{CIL^+} (x) &= H(y, \{\mathbf{P}(x \in \mathbf{X}_{k, j} | D)\}_{k, j} \cup \{\mathbf{P}(x \in \mathbf{X^+} | D)\}), \\
 H_{TP^+} (x) &= H(\Bar{y}, \{\mathbf{P}(x \in \mathbf{X}_k | D)\}_{k} \cup \{\mathbf{P}(x \in \mathbf{X^+} | D)\}). 
\end{align}
For any $x \in \mathbf{X} \cup \mathbf{X}^+$, define 
\begin{align}
    H_{OOD^+, k} (x) = \left\{
    \begin{aligned}
    H(1, \mathbf{P}'_k (x \in \mathbf{X}_k | D)) =& - \log \mathbf{P}'_k (x \in \mathbf{X}_k | D), \\ &~~~~~~~~~~~~~~~~~~~~~~~~x \in \mathbf{X}_k, \\
    H(0, \mathbf{P}'_k (x \in \mathbf{X}_k | D)) =& - \log \mathbf{P}'_k (x \in (\mathbf{X} \cup \mathbf{X}^+) \backslash \mathbf{X}_k | D), \\ &~~~~~~~~~~~~~~~~x \in (\mathbf{X} \cup \mathbf{X}^+) \backslash \mathbf{X}_k. \\
    \end{aligned}
    \right.
\end{align} 
where OOD$^+$ denotes the \textit{open-world OOD detection}. 

It is clear that open-world OOD detection implies closed-world OOD detection, but the reverse is not true. Since the classification in the closed-world CIL is limited to the $T$ tasks learned so far, it cannot derive open-world OOD detection but only closed-world OOD detection. Thus, only closed-world OOD detection is \textit{necessary} for the traditional closed-world CIL.

We now generalize Theorem~\ref{thm:tp_ood} to 
the open-world CIL (CIL$^+$) setting with the following Corollary. The proof is given in~\ref{prf:tp_ood_to_open}.

\vspace{+2mm}
\begin{corollary}
\label{thm:tp_ood_to_open}
i) If $H_{TP^+} (x) \leq \delta$, let $\mathbf{P}'_k (x \in \mathbf{X}_k | D) = \mathbf{P} (x \in \mathbf{X}_k | D)$, then $H_{OOD^+, k} (x) \leq \delta, \forall\, k = 1, \dots, T$. 
ii) If $H_{OOD^+, k} (x) \leq \delta_k, k=1,\dots,T$, let 
$\mathbf{P} (x \in \mathbf{X}_k | D) = \frac{\mathbf{P}'_{k} (x \in \mathbf{X}_{k} |D)}{\sum_{k} \mathbf{P}'_{k} (x \in \mathbf{X}_{k} |D) + \prod_{k} (1 - \mathbf{P}'_{k} (x \in \mathbf{X}_{k} |D))}$
and 
$\mathbf{P} (x \in \mathbf{X}^+ | D) = \frac{\prod_{k} (1 - \mathbf{P}'_{k} (x \in \mathbf{X}_{k} |D))}{\sum_{k} \mathbf{P}'_{k} (x \in \mathbf{X}_{k} |D) + \prod_{k} (1 - \mathbf{P}'_{k} (x \in \mathbf{X}_{k} |D))}$,
then $H_{TP^+} (x) \leq \max ( (\sum_k \mathbf{1}_{x \in \mathbf{X}_k} e^{\delta_{k}}) (\sum_k (1 + \mathbf{1}_{x \in \mathbf{X}_k})(1 - e^{-\delta_k})), \prod_{k} e^{\delta_{k}} \sum_{k} 1 - e^{-\delta_{k}})$,
where $\mathbf{1}_{x \in \mathbf{X}_k}$ is an indicator function.
\end{corollary}
\vspace{+2mm}

For Corollary~\ref{thm:tp_ood_to_open}, we have that $(\sum_k \mathbf{1}_{x \in \mathbf{X}_k} e^{\delta_{k}}) (\sum_k (1 + \mathbf{1}_{x \in \mathbf{X}_k})(1 - e^{-\delta_k}))$ converges to $0$ as $\delta_k$'s converges to $0$ in order of $O(|\sum_k \delta_k| + \max_k |\delta_k|) = O(|\sum_k \delta_k|) $, and $\prod_{k} e^{\delta_{k}} \sum_{k} 1 - e^{-\delta_{k}}$ converges to $0$ in order of $O(|\sum_k \delta_k|)$.
Therefore, we know that $H_{TP^+}$ and $\sum_k H_{OOD^+,k}$ are equivalent in quantity up to a constant factor.

We can extend Theorem~\ref{thm:cil_with_op_and_ood} to the open-world CIL$^+$ using Corollary~\ref{thm:tp_ood_to_open}. By substituting $H_{OOD,k}$  with $H_{OOD^+,k}$ and $H_{CIL}$ with $H_{CIL^+}$, we obtain a new upper bound for the open-world CIL, $\epsilon + \max ( (\sum_k \mathbf{1}_{x \in \mathbf{X}_k} e^{\delta_{k}}) (\sum_k (1 + \mathbf{1}_{x \in \mathbf{X}_k})(1 - e^{-\delta_k})), \prod_{k} e^{\delta_{k}} \sum_{k} 1 - e^{-\delta_{k}})$. The proof is trivially identical to the original proof of Theorem~\ref{thm:cil_with_op_and_ood}.

We can establish the same theorem as Theorem~\ref{thm:necessary_condition} for CIL$^+$ by replacing $H_{OOD,k}$ with $H_{OOD^+,k}$ and $H_{CIL}$ with $H_{CIL^+}$. Again, the proof is trivially identical to the original proof of Theorem~\ref{thm:necessary_condition}.

The new theorems establish that a good TP$^+$ or OOD$^+$ (open-world OOD detection) and a good WP are necessary and sufficient for a good CIL$^+$.

\section{Proposed Approach 1: Combining TIL and OOD Detection} \label{sec.clom}

{Based on the above theoretical result, we have designed two approaches to solving CIL that employ OOD detection methods, more precisely \textit{open-world OOD detection} methods. Although theoretically speaking, open-world OOD detection implies closed-world OOD detection, 
in practical applications, 
we often do not need to distinguish whether an OOD detection method is a closed-world or an open-world method as they usually can be used for either closed-world or open-world CIL. We just want them to be as accurate as possible for the applications. 

This section presents the first approach, which combines a task incremental learning (TIL) method and an OOD detection method. The approach does not save any training data from previous tasks. The OOD detection method here is an open-world method as it does not use any information from the other tasks learned in the CIL process. The next section presents the second approach, which is based on replay and needs to save some training data from previous tasks.\footnote{Note that this paper focuses on establishing a theoretical connection between novelty (or OOD) detection and class incremental learning (CIL). Our experiments show the validity of the theory. We also report the OOD detection results using AUC but this paper does not focus on the problem of real-time decision-making and learning using OOD detection to detect each novel instance, acquire its class label, and incrementally learn it. The reason is that this will involve setting an OOD score threshold to decide each OOD instance and interacting with human users to acquire the class label to learn. 
Such user interactions wouldn't give the system a large number of labeled training data. Then the highly challenging few-shot continual learning is required. We leave this to future work.} The OOD detection method used there is a closed-world OOD detection method as it treats the replay data from previous tasks as the OOD data in the model building, but this method can also be used for open-world CIL.

\begin{figure*}
\centering
\subfigure[]{\includegraphics[width=67mm]{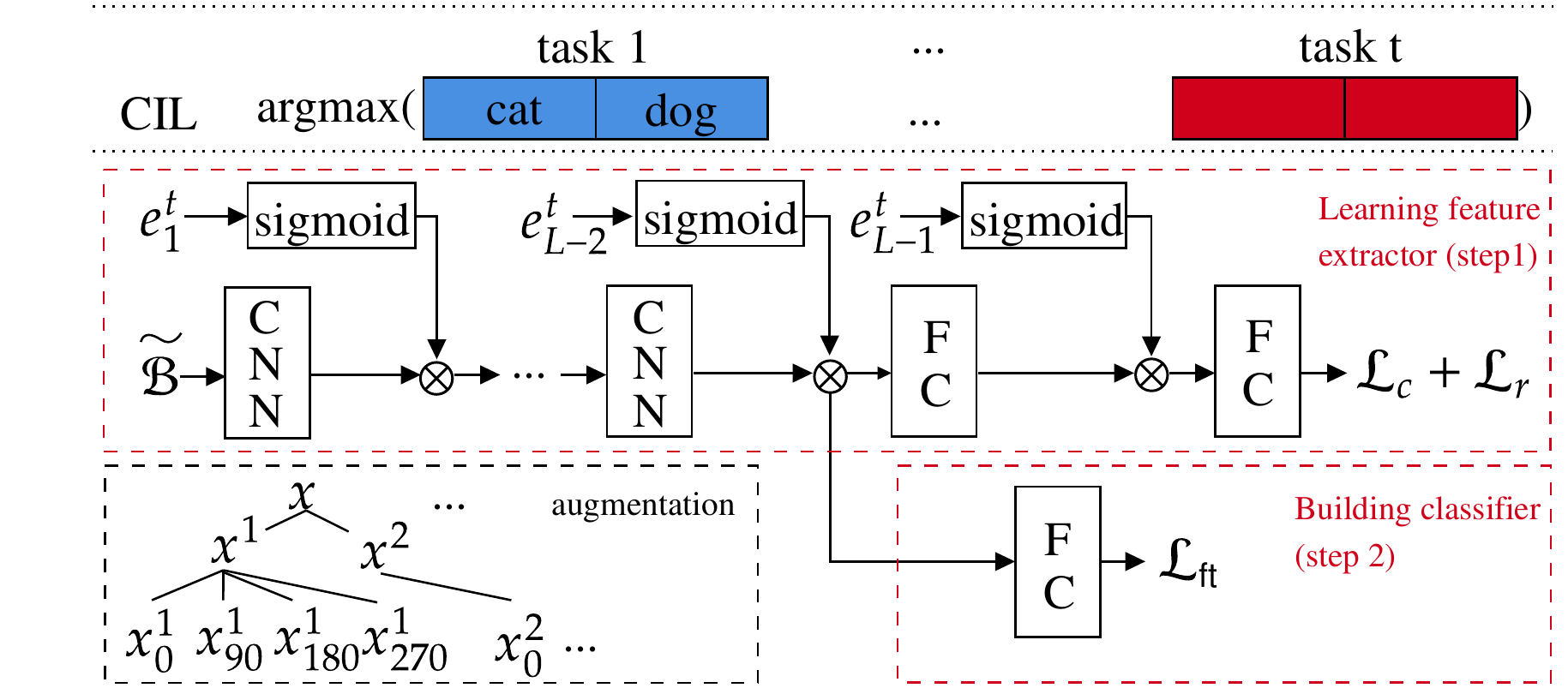}}
\subfigure[]{\includegraphics[width=67mm]{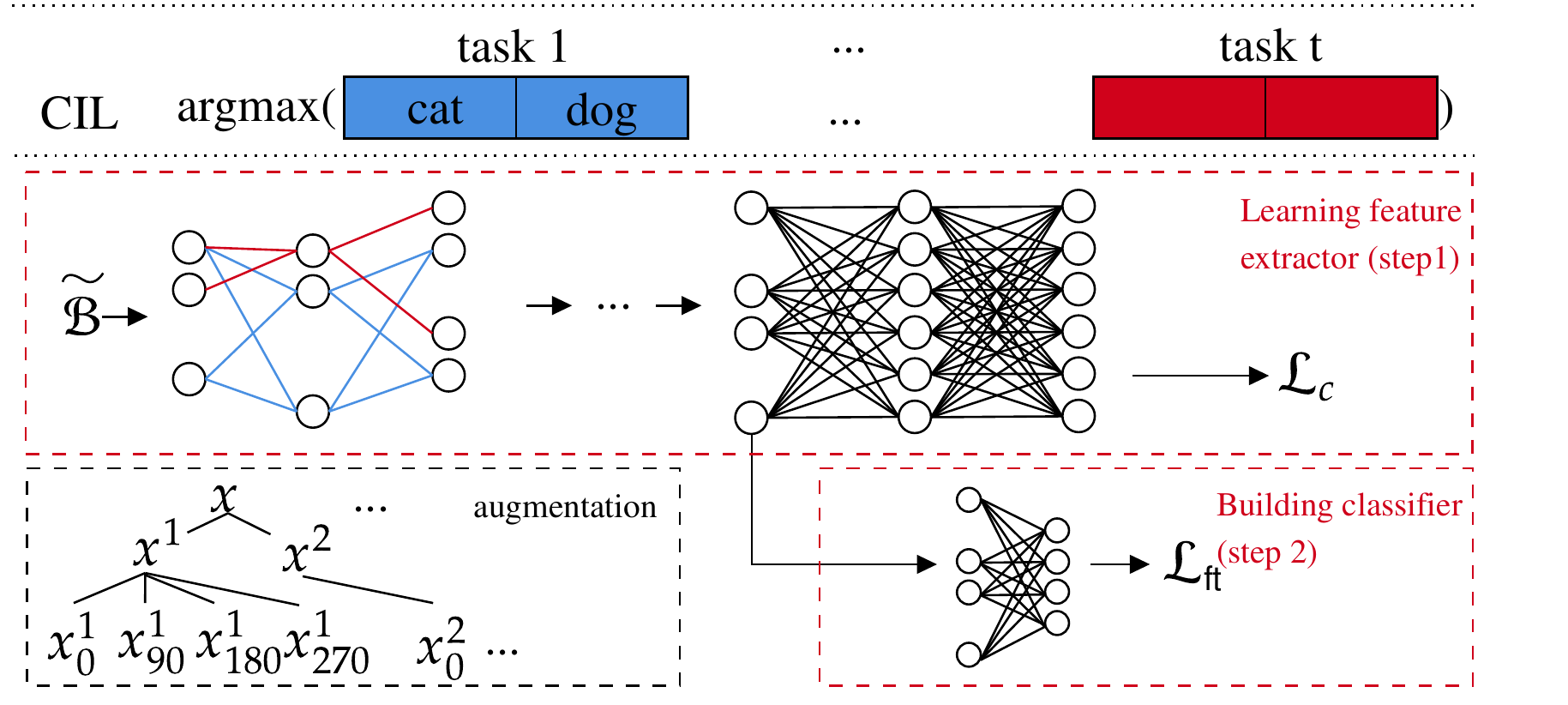}}
\caption{
Overview of prediction and training framework of HAT+CSI and Sup+CSI. \textbf{(a)} HAT+CSI: The CIL prediction is made by argmax over the concatenated output from each task. The training of each task uses CSI. That is, the training batch is augmented to give different views of the samples for contrastive training. The training consists of two steps following CSI. The first step learns the feature extractor by 
using the hard attention algorithm~\citep{Serra2018overcoming}, which applies task embeddings to find hard masks at each layer. Then given the learned feature representations, it fine-tunes the classifier in step 2. \textbf{(b)} Sup+CSI: The CIL prediction is also made by taking argmax over the concatenated output values from each task as HAT+CSI. The model training for each task is similar to HAT+CSI except that it uses 
the Edge Popup algorithm of SupSup~\citep{ramanujan2020s} for finding a sparse network for each task. The sparse networks are indicated by edges of different colors in the diagram. The second step fine-tunes the classifier only with the fixed feature extractor.
}
\label{fig.diagrams}
\end{figure*}

\subsection{Combining a TIL Method and an OOD Detection Method}

As mentioned earlier, several existing TIL methods can overcome CF. This proposed approach basically leverages the CF prevention ability in two TIL methods (HAT~\citep{Serra2018overcoming} and SupSup (Sup)~\citep{supsup2020}) and replaces their task learning methods with an OOD detection technique, called CSI~\citep{tack2020csi}, which can perform both within-task or IND prediction (WP) and OOD instance detection. Below, we first introduce the two TIL methods, HAT and SupSup, and the OOD detection method, CSI. The combinations give two new CIL methods, HAT+CSI and Sup+CSI. None of these methods needs to save any data from previous tasks. 

Figure~\ref{fig.diagrams} shows the overall training frameworks of HAT+CSI and Sup+CSI. Note that both HAT and Sup are multi-head methods (one head for each task) designed for task incremental learning (TIL).

\subsubsection{HAT: Hard Attention Masks} \label{sec.hat}
To prevent forgetting the trained OOD detection model $f^{k} \circ h^k$ for each task $k$ in subsequent task learning, 
the hard attention mask (HAT)~\citep{Serra2018overcoming} for TIL is employed (which prevents forgetting in the feature extractor). Specifically, 
in learning a task, a set of embeddings is trained to protect the important neurons so that the corresponding parameters are not interfered with by subsequent tasks. The importance of a neuron is measured by the 0-1 pseudo-step function, where 0 indicates not important and 1 indicates important (and thus protected).

The hard attention mask is an output of sigmoid function $u$ with a hyper-parameter $s$
\begin{align}
    a_{l}^{k} = u(s e_{l}^{k}), \label{eq.smax}
\end{align}
where
$e_{l}^{k}$ is a learnable embedding at layer $l$ of task $k$. Since the step function is not differentiable, a sigmoid function with a large $s$ is used to approximate it. Sigmoid is approximately a 0-1 step function with a large $s$. The attention is multiplied to the output $h_{l} = \text{ReLU}(W_l h_{l-1} + b_l)$ of layer $l$,
\begin{align}
    h'_{l} = a_{l}^{k} \otimes h_{l}
\end{align}
The $j$th element $a_{j, l}^{k}$ in the attention mask blocks (or unblocks) the information flow from neuron $j$ at layer $l$ if its value is $0$ (or $1$). 
With 0 value of $a_{j, l}^{k}$, the corresponding parameters in $W_{l}$ and $b_{l}$
can be freely changed as the output values $h'_{l}$ are not affected. The neurons with non-zero mask values are necessary to perform the task and thus need protection from catastrophic forgetting.

We modify the gradients of parameters that are important in performing the previous tasks $(1, \cdots, k-1)$ during training task $k$ so they are not interfered with. Denote the accumulated mask by
\begin{align}
    a_{l}^{<k} = \max(a_{l}^{<k-1}, a_{l}^{k-1})
\end{align}
where $\max$ is element-wise maximum and the initial mask $a_{l}^{0}$ is a zero vector. 
It is a collection of mask values at layer $l$ where a neuron has value 1 if it has ever been activated previously.
The gradient of parameter $w_{ij, l}$ is modified as
\begin{align}
    \nabla w_{ij, l}' = \left( 1 - \min\left( a_{i,l}^{<k}, a_{j, l-1}^{<k} \right) \right) \nabla w_{ij, l} \label{eq:grad_mod}
\end{align}
where $a_{i,l}^{<k}$ is the $i$th unit of $a_{l}^{<k}$. The gradient flow is blocked if both neurons $i$ in the current layer and $j$ in the previous layer have been activated. We apply the mask for all layers except the last layer. The parameters in the last layer do not need to be protected as they are task-specific parameters.

A regularization is introduced to encourage sparsity in $a_{l}^{k}$ and parameter sharing with $a_{l}^{<k}$. The capacity of a network depletes when $a_{l}^{<k}$ becomes 1-vector in all layers. Despite a set of new neurons can be added to the network at any point in training for more capacity, we utilize resources more efficiently by minimizing the loss
\begin{align}
    \mathcal{L}_r = \lambda \frac{\sum_l \sum_i a_{i,l}^{k}\left( 1 - a_{i,l}^{<k} \right)}{\sum_{l} \sum_{i} \left(1 - a_{i, l}^{<k} \right)} \label{eq.hat_reg}
\end{align}
where $\lambda$ is a hyper-parameter. The final objective of training a comprehensive task network without forgetting is
\begin{align}
    \mathcal{L} = \mathcal{L}_{ce} + \mathcal{L}_{r} \label{eq.loss_hat}
\end{align}
where $\mathcal{L}_{ce}$ is the cross-entropy loss. The overall framework of the algorithm is shown in Figure~\ref{fig.diagrams}(a).

Note that for TIL, HAT needs the task-id for each test instance in order to choose the right task model for prediction or classification. However, by replacing the original model building method for each task in HAT with the OOD detection
method in CSI (more specifically, $\mathcal{L}_{c}$)
during training, HAT+CSI does not require to know the task-id
of each test instance at inference, which makes HAT+CSI suitable for CIL (class incremental learning). We will see the detailed prediction/classification method in Section~\ref{sec.experiment1}.

\subsubsection{SupSup: Supermasks in Superposition} \label{sec.sup}
SupSup (Sup)~\citep{supsup2020} is also a highly effective method that can overcome forgetting in the TIL setting. Sup trains supermasks by Edge Popup algorithm in~\citep{ramanujan2020s}. Specifically, given the initial weights of a base network $\mathbf{W}$, find binary masks $\mathbf{M}_k$ for task $k$ to minimize the cross-entropy loss.
\begin{align}
    \mathcal{L} = - \frac{1}{|\mathbf{X}_{k}|} \sum \log p(y | x, k), \label{eq.loss_sup}
\end{align}
where $\mathbf{X}_k$ is the training data for task $k$, and
\begin{align}
    p(y | x, k) = f(h(x; \mathbf{W} \otimes \mathbf{M}_{k})),
\end{align}
where $\otimes$ indicates element-wise product. The masks are obtained by selecting the top $p$\% of entries in the score matrices $\mathbf{V}$. The $p$ value determines the sparsity of the mask $\mathbf{M}_{k}$. The subnetwork found by Edge Popup algorithm is indicated by different colors in Figure~\ref{fig.diagrams}(b).

Like HAT, Sup is also for TIL and needs the task-id $k$ of each test instance at inference. With $k$, the system (which is referred to as Sup GG in the original Sup paper) uses the task-specific mask $\mathbf{M}_k$ to obtain the classification output. Like HAT+CSI, by replacing the cross-entropy loss in mask finding with the OOD detection loss in CSI, Sup+CSI also does not require the task-id of each test instance, which makes Sup+CSI applicable to CIL (class incremental learning). We will discuss the detailed prediction/classification method in Section~\ref{sec.experiment1}.

}

\subsubsection{CSI: Contrasting Shifted Instances for OOD Detection} \label{sec.csi}

The OOD detection method CSI is based on contrastive learning~\citep{chen2020simple,khosla2020supervised}, data and class augmentations, and results ensembling~\citep{tack2020csi}. 
The OOD training process is similar to that of contrastive learning. It consists of two steps: Step 1 learns {the feature representation by} the composite $g \circ h$, where $h$ is the feature extractor and $g$ is the projection to contrastive representation, and Step 2 learns/fine-tunes the linear classifier $f$, mapping the feature representation of $h$ to the label space (the classifier is the OOD model)). This two-step training process is outlined in Figure~\ref{fig.diagrams}(b). In the following,
we first describe the two-step training process and then explain how to make
a prediction based on an ensemble method to further improve the prediction. 

\textbf{Step 1 (Contrastive Loss for Feature Learning).} 
Supervised contrastive learning is used to try to repel data of different classes and
align data of the same class more closely to make it easier to classify them. A key operation is data
augmentation via transformations.

{Given a batch of $N$ samples, each sample ${x}$ is first duplicated. Each version then goes through \textit{three initial augmentations}
(horizontal flip, color changes, and Inception crop~\citep{inception}) to generate two different views ${x}^{1}$ and ${x}^{2}$ (they keep the same class label as ${x}$).}
Denote the augmented batch by $\mathcal{B}$, which now has $2N$ samples. In~\citep{hendrycks2019using} and \citep{tack2020csi}, it was shown that using image rotations is effective in learning OOD detection models because such rotations can effectively serve as out-of-distribution (OOD) training data.  
For each augmented sample ${x} \in \mathcal{B}$ with class $y$ of a task, we rotate ${x}$ by $90^{\circ}, 180^{\circ}, 270^{\circ}$ to create three images, which are assigned \textit{three new classes} $y_1, y_2$, and $y_3$, respectively.
{This results in a larger augmented batch $\tilde{\mathcal{B}}$. Since we generate three new images from each ${x}$, 
the size of $\tilde{\mathcal{B}}$ is $8N$. For each original class, we now have 4 classes. For a sample ${x} \in \tilde{\mathcal{B}}$, let $\mathcal{\tilde{B}}({x}) = \mathcal{\tilde{B}} \backslash \{ {x} \}$
and 
let $P({x}) \subset \tilde{\mathcal{B}} \backslash \{ {x} \}$ 
be a set consisting of the data of the same class as ${x}$ distinct from ${x}$.
The contrastive representation of a sample ${x}$ is ${z}_{x} = g(h({x, k})) / \| g(h({x, k})) \|$, where $k$ is the current task.
In learning, we minimize the supervised contrastive loss.
\begin{align}
    \mathcal{L}_{c}
    &= \frac{1}{8N} \sum_{ {x} \in \tilde{\mathcal{B}}} \frac{-1}{| P({x}) |} \sum_{{p} \in P({x})} \log{ \frac{ \text{exp}( {z}_{{x}} \cdot {z}_{{p}} / \tau)}{\sum_{{x}'  \in \tilde{\mathcal{B}}({x}) } \text{exp}( {z}_{{x}} \cdot {z}_{{x}'} / \tau) } }, \label{eq.modsupclr}
\end{align}
where 
$\tau$ is a scalar temperature, $\cdot$ is dot product, and $\times$ is  multiplication. 
The loss is reduced by repelling ${z}$ of different classes and aligning ${z}$ of the same class more closely.
$\mathcal{L}_{c}$ basically trains a feature extractor with good 
representations for learning an  OOD classifier.} 

Since the feature extractor is shared across tasks in continual learning, protection is needed to prevent catastrophic forgetting. HAT and Sup use their respective techniques to protect their feature extractor from forgetting. Therefore, the losses $\mathcal{L}$ of Eq.~\ref{eq.loss_sup} and $\mathcal{L}_{ce}$ of Eq.~\ref{eq.loss_hat} are replaced by Eq.~\ref{eq.modsupclr} while the forgetting prevention mechanisms still hold.

\textbf{Step 2 (Fine-tuning the Classifier).} {Given the feature extractor $h$ trained with the loss in Eq.~\ref{eq.modsupclr}, we {\textit{freeze $h$} and} only \textit{fine-tune} the linear classifier $f$, which is trained to predict the classes of task $k$ \textit{and} the augmented rotation classes.} $f$ maps the feature representation to {the label space in} $\mathcal{R}^{4|\mathcal{C}^{k}|}$, where $4$ is the number of rotation classes including the original data with $0^{\circ}$ rotation and $|\mathcal{C}^{k}|$ is the number of {original} classes in task $k$. We minimize the cross-entropy loss,
\begin{align}
    \mathcal{L}_{\text{ft}} = - \frac{1}{|\tilde{\mathcal{B}} |} \sum_{({x}, y) \in \tilde{\mathcal{B}}} 
    \log \tilde{p}(y | {x}, k), 
    \label{3obj}
\end{align}
where $\text{ft}$ indicates fine-tune,
and 
\begin{align}
    \tilde{p}(y | {x}, k) = \text{softmax} \left( f(h({x}, k))
    \right) \label{probrotation}
\end{align}
where
$f(h({x, k})) \in \mathcal{
R}^{4|\mathcal{C}^{k}|}$. The output $f(h({x, k}))$
includes the rotation classes. The linear classifier is trained to predict the original \textit{and} the rotation classes. Since an individual classifier is trained for each task and the feature extractor is frozen, no protection is necessary.

\vspace{2mm}
\noindent
\textbf{Ensemble Class Prediction.} We now discuss the prediction of class label $y$ for a test sample ${x}$.
Note that the network $f\circ h$ in Eq.~\ref{probrotation} returns logits for rotation classes (including the original task classes). Note also for each original class label $j_k \in \mathcal{C}^{k}$ (original classes) of a task $k$, we created three additional rotation classes. For class $j_k$, the classifier $f$ will produce four output values from its four rotation class logits, i.e., $f_{j_k,0}(h({x_0, k}))$, $f_{j_k,90}(h({x_{90}, k}))$, $f_{j_k,180}(h({x_{180}, k}))$, and $f_{j_k,270}(h({x_{270}, k}))$, where 0, 90, 180, and 270 represent $0^{\circ}, 90^{\circ}, 180^{\circ}$, and $270^{\circ}$ rotations respectively and ${x}_0$ is the original ${x}$. 
We compute an ensemble output $f_{j_k}(h({x}))$ for each class $j_k \in \mathcal{C}^{k}$ of task $k$, 
\begin{align}
    f(h({x, k}))_{j_k} = \frac{1}{4} \sum_{\text{deg}} f (h({x}_{\text{deg}}, k))_{j_k,\text{deg}} \label{eq:ensemblelogit}.
\end{align}

\subsection{Experiments}
\label{sec.experiment1}
We now present the experimental results of the combination techniques HAT+CSI and Sup+CSI for \textit{class incremental learning} (CIL). We will also use another OOD detection method ODIN~\citep{liang2018enhancing} to show that a better OOD detection method leads to better CIL results. We do not conduct extensive experiments on ODIN as it is much weaker than CSI in terms of ODD detection. Note that we will not report the ODD detection results for HAT+CSI and Sup+CSI in the open world in the continual learning process as the proposed method MORE in the next section performs better. 

\subsubsection{Experimental Datasets and Baselines}\label{sec.baselines}

\vspace{+2mm}
\noindent
\textbf{Datasets and CIL tasks}: We use three standard image classification benchmark datasets and construct five different CIL experiments.

\textbf{1.} \textbf{CIFAR-10}~\citep{Krizhevsky2009learning}: This dataset consists of 32x32 color images of 10 classes with 50,000 training and 10,000 testing samples. We construct an experiment (\textbf{C10-5T)} of 5 tasks with 2 classes per task.

\textbf{2.} \textbf{CIFAR-100}~\citep{Krizhevsky2009learning}: This dataset consists of 32x32 color images of 100 classes with 50,000 training and 10,000 testing samples. We construct two experiments of 10 tasks (\textbf{C100-10T}) and 20 tasks (\textbf{C100-20T}), where each task has 10 classes and 5 classes, respectively.

\textbf{3.} \textbf{Tiny-ImageNet}~\citep{Le2015TinyIV}: This is an image classification dataset with 64x64 color images of 200 classes with 100,000 training and 10,000 validation samples. Since the dataset does not provide labels for testing data, we use the validation data for testing. We construct two experiments of 5 tasks (\textbf{T-5T}) and 10 tasks (\textbf{T-10T}) with 40 classes per task and 20 classes per task, respectively.

\vspace{+2mm}
\noindent
\textbf{Baselines}: We use 18 diverse continual learning baselines: 

\textbf{1.} One projection method (\textbf{OWM}~\citep{zeng2019continuous}).

\textbf{2.}~Two exemplar-free (no replay data is saved) regularization methods (\textbf{MUC}~
\citep{Liu2020} and \textbf{PASS}~
\citep{Zhu_2021_CVPR_pass}). 

\textbf{3.} Nine replay-based methods (\textbf{LwF}~\citep{Li2016LwF}, \textbf{iCaRL}~\citep{Rebuffi2017}, \textbf{A-GEM}~\citep{chaudhry2018efficient}, \textbf{EEIL}~\citep{castro2018end_eeil}, \textbf{GD}~\citep{lee2019overcoming_gd}, \textbf{Mnemonics}~\citep{Liu_2020_CVPR}, \textbf{BiC}~\citep{wu2019large}, \textbf{DER++}~\citep{NEURIPS2020_b704ea2c_derpp}], and \textbf{HAL}~\citep{Chaudhry_Gordo_Dokania_Torr_Lopez-Paz_2021_hal}).

\textbf{4.} Three parameter-isolation methods (\textbf{HAT}~\citep{Serra2018overcoming}, \textbf{HyperNet}~\citep{von2019continual}, and \textbf{SupSup}~\citep{supsup2020}). 

\textbf{5.} Additionally, we report the accuracies of replay-based method \textbf{Co$^2$L} \citep{Cha_2021_ICCV_co2l} and parameter isolation methods \textbf{CCG} \citep{abati2020conditional} and \textbf{PR-Ent} \citep{henning2021posterior} from their original papers as CCG has not released the code and we are unable to run Co$^2$L and PR-Ent on our machines.

\subsubsection{Training Details and Evaluation Metrics}
\label{sec.training}

\textbf{Training Details.} \label{sec:training_details}
For the backbone structure, we follow \citep{supsup2020,Zhu_2021_CVPR_pass,NEURIPS2020_b704ea2c_derpp} and use
ResNet-18~\citep{he2016deep}. For CIFAR-100 and Tiny-ImageNet, the number of channels is doubled to fit more classes. For all baselines, the same ResNet-18 backbone architecture is employed except for OWM and HyperNet, for which we use their original architectures. OWM uses AlexNet. It is not obvious how to apply its orthogonal projection technique to the ResNet structure. HyperNet uses ResNet-32 and we are unable to replace it due to model initialization arguments unexplained in the original paper.
For the replay methods, we use the memory buffer of 200 for 
CIFAR-10 and 2000 for CIFAR-100 and Tiny-ImageNet as in \citep{Rebuffi2017,NEURIPS2020_b704ea2c_derpp}. We use the hyper-parameters suggested by the authors. If we cannot reproduce any result, we use 10\% of the training data as a validation set to grid-search for good hyper-parameters. For our proposed methods, we report the hyper-parameters in \ref{apx:hyper_params}.
All the results are averages over 5 runs with random seeds.

\textbf{Evaluation Metrics.} 

\textbf{1.} \textit{Average classification accuracy} over all classes after learning the last task. The final class prediction depends on \textit{prediction methods} (see below). We also report \textit{forgetting rate} 
in \ref{apx:forgetting}.

\textbf{2.} \textit{Average AUC} (Area Under the ROC Curve) over all task models for the evaluation of OOD detection. AUC is the main measure used in OOD detection papers. Using this measure, we show that a better OOD detection method will result in a better CIL performance. Let $\textit{AUC}_{k}$ be the AUC score of task $k$. It is computed by using only the model (or classes) of task $k$ to score the test data of task $k$ as the in-distribution (IND) data and the test data from other tasks as the out-of-distribution (OOD) data. The average AUC score is: $AUC = \sum_{k} \textit{AUC}_{k}/n$, where $n$ is the number of tasks.

It is not straightforward to change existing CL algorithms to include a new OOD detection method that needs training, e.g., CSI, except for TIL (task incremental learning) methods like HAT and Sup. For HAT and Sup, we can simply switch their methods for learning each task with CSI (see Section~\ref{sec.hat} and Section~\ref{sec.sup}).

\textbf{Prediction Methods.} The theoretical result in Section~\ref{sec.theorem} states that we use Eq.~\ref{eq:cil_in_til_and_tp} to perform the final prediction. The first probability (WP) in Eq.~\ref{eq:cil_in_til_and_tp} is easy to get as we can simply use the softmax values of the classes in each task. However, the second probability (TP) in Eq.~\ref{eq:cil_in_til_and_tp} is tricky as each task is learned without the data of other tasks. There can be many options. 
We take the following approaches for prediction (which are a special case of Eq.~\ref{eq:cil_in_til_and_tp}, see below):

\textbf{1.} For those approaches that use a single classification head to include all classes learned so far, we predict as follows (which is also the approach taken by the existing papers.) 
\begin{align}
    \hat{y} = \argmax f(x)
\end{align}
where $f(x)$ is the logit output of the network. 

\textbf{2.} For multi-head methods (e.g., HAT, HyperNet, and Sup), which use one head for each task, we use the concatenated output as
\begin{align}
    \hat{y} = \argmax \bigoplus_{k} f(x)_{k} \label{eq:cil_pred}
\end{align}
where $\bigoplus$ indicate concatenation and $f(x)_k$ is the output of task $k$.\footnote{The {Sup} paper proposed a one-shot task-id prediction assuming that the test instances come in a batch and all belong to the same task like iTAML. We assume a single test instance per batch. Its task-id prediction results in an accuracy of 50.2 on C10-5T, which is much lower than 62.6 by using Eq.~\ref{eq:cil_pred}. The task-id prediction of {HyperNet} also works poorly. The accuracy of its task-id prediction is 49.34 on C10-5T while it is 53.4 using Eq.~\ref{eq:cil_pred}. {PR} uses entropy to find task-id. Among many variations of PR, we use the variations that perform the best for each dataset with exemplar-free and single sample per batch at testing (i.e., no PR-BW).}

These methods (in fact, they are the same method used in two different settings) are a special case of Eq.~\ref{eq:cil_in_til_and_tp} if we define $OOD_k$ as $\sigma(\max f(x)_k )$, where $\sigma$ is the sigmoid. Hence, the theoretical results in Section~\ref{sec.theorem} are still applicable. We present a detailed explanation of this prediction method and some other options in \ref{apx:diff_tp}. These two approaches work quite well.

\subsubsection{Better OOD Detection Produces Better CIL Performance}\label{sec.betterOOD}
The key theoretical result in Section \ref{sec.theorem} is that better OOD detection will produce better CIL performance. We compare a weaker OOD method ODIN with the strong CSI. ODIN is a post-processing method for OOD detection \citep{liang2018enhancing}. Note that it does not always improve the OOD detection performance compared to without the ODIN post-processing (see below).

\begin{table}[t]
\centering
\caption{Performance Comparison between the Original Output and ODIN. Note that ODIN does not apply to iCaRL and Mnemonics as they are not based on softmax but some distance functions. As mentioned earlier, for \textbf{Co$^2$L}, \textbf{CCG}, and \textbf{PR-Ent}, they either have no code, or their codes do not run on our machine.
The results for other datasets are in \ref{apx:additional_odin}.}
\resizebox{0.55\columnwidth}{!}{
\begin{tabular}{lcccc}
\toprule
& \multicolumn{2}{c}{AUC} & \multicolumn{2}{c}{CIL} \\
Method & Original & ODIN & Original & ODIN \\
\midrule
OWM & 71.31 & 70.06 & 28.91 & 28.88 \\
\hline
MUC & 72.69 & 72.53 & 30.42 & 29.79 \\
PASS & 69.89 & 69.60 & 33.00 & 31.00 \\
\hline
LwF & 88.30 & 87.11 & 45.26 & 51.82 \\
A-GEM & 78.01 & 79.00 & 9.29 & 13.48 \\
EEIL & 83.37 & 79.73 & 48.99 & 41.74 \\
GD & 85.37 & 82.98 & 49.67 & 47.28 \\
BiC & 87.89 & 86.73 & 52.92 & 48.65 \\
DER++ & 85.99 & 88.21 & 53.71 & 55.29 \\
HAL & 64.21 & 64.83 & 15.59 & 21.01 \\
\hline
HAT & 77.72 & 77.80 & 41.06 & 41.21 \\
HyperNet & 71.82 & 72.32 & 30.23 & 30.83 \\
Sup & 79.16 & 80.58 & 44.58 & 46.74 \\
\bottomrule
\end{tabular}
}
\label{Tab:odin}
 \end{table}

\textbf{Applying ODIN.} 
We first train the baseline models using their original algorithms, and then apply temperature scaling and input noise of ODIN at testing for each task (no training data needed).
More precisely, the output of class $j$ in task $k$ changes by temperature scaling factor $\tau_{k}$ of task $k$ as
\begin{align}
    s(x; \tau_k)_j = e^{f(x)_{kj} / \tau_k } / \sum_{j} e^{f(x)_{kj} / \tau_{k}} \label{eq:odin_softmax}
\end{align}
and the input changes by the noise factor $\epsilon_k$ as
\begin{align}
    \tilde{x} = x - \epsilon_k \text{sign} (-\nabla_x \log s (x; \tau_{k})_{\hat{y}} ) \label{eq:odin_perturbation}
\end{align}
where $\hat{y}$ is the class with the maximum output value in task $k$. This is a positive adversarial example inspired by \citep{goodfellow2015explaining}. The values $\tau_k$ and $\epsilon_k$ are hyper-parameters and we use the same values for all tasks except for PASS, for which we use a validation set to tune $\tau_k$ (see~\ref{apx:additional_odin}).

Table~\ref{Tab:odin} gives the results for C100-10T. The CIL results clearly show that the CIL performance increases if the AUC increases with ODIN. For instance, the CIL of DER++ and Sup improves from 53.71 to 55.29 and 44.58 to 46.74, respectively, as the AUC increases from 85.99 to 88.21 and 79.16 to 80.58.
It shows that when this method is incorporated into each task model in the existing trained CIL network, the CIL performance of the original method improves. We note that ODIN does not always improve the average AUC.
For those experiencing a decrease in AUC, the CIL performance also decreases except LwF. The inconsistency of LwF is due to its severe classification bias towards later tasks as discussed in BiC~\citep{wu2019large}. The temperature scaling in ODIN has a similar effect as 
the bias correction in BiC, and the CIL of LwF becomes close to that of BiC after the correction. Regardless of whether ODIN improves AUC or not, the positive correlation between AUC and CIL (except LwF) verifies the efficacy of Theorem~\ref{thm:cil_with_op_and_ood}, indicating better OOD detection results in better CIL performances.

\begin{table}[t]
\centering
\caption{Average CIL and AUC of HAT and Sup after applying OOD detection methods ODIN and CSI. ODIN is a traditional OOD detection method while CSI is a recent OOD detection method known to be better than ODIN. As CL methods produce better OOD detection performance by CSI, their CIL performances are better than the ODIN counterparts.}
\resizebox{0.95\columnwidth}{!}{
\begin{tabular}{l l c c c c c c c c c c}
\toprule
\multicolumn{1}{c}{CL} & \multicolumn{1}{c}{OOD} & \multicolumn{2}{c}{C10-5T} & \multicolumn{2}{c}{C100-10T} &  \multicolumn{2}{c}{C100-20T} & \multicolumn{2}{c}{T-5T} & \multicolumn{2}{c}{T-10T} \\
{} & {} & AUC & CIL & AUC & CIL & AUC & CIL & AUC & CIL & AUC & CIL \\
\midrule
\multirow{2}{*}{HAT} & 
ODIN &
82.5 & 62.6 &
77.8 & 41.2 &
75.4 & 25.8 &
72.3 & 38.6 &
71.8 & 30.0 \\
{} & 
CSI &
91.2 & 87.8 & %
84.5 & 63.3 & %
86.5 & 54.6 & %
76.5 & 45.7 &
78.5 & 47.1 \\ %
\hline
\multirow{2}{*}{Sup} & 
ODIN &
82.4 & 62.6 &
80.6 & 46.7 &
81.6 & 36.4 &
74.0 & 41.1 &
74.6 & 36.5 \Tstrut \\
{} & 
CSI &
91.6 & 86.0 & %
86.8 & 65.1 & %
88.3 & 60.2 & %
77.1 & 48.9 & 
79.4 & 45.7 \\ 
\bottomrule
\end{tabular}
}
\label{Tab:odin_csi}
\vspace{-3mm}
\end{table}

\textbf{Applying CSI.} We now apply the OOD detection method CSI. Due to its sophisticated data augmentation, supervised contrastive learning, and results ensemble, it is hard to apply CSI to other baselines without fundamentally changing them except for HAT and Sup (SupSup) as these methods are parameter isolation-based TIL methods. We can simply replace their model for training each task with CSI wholesale. As mentioned earlier, both HAT and Sup as TIL methods have almost no forgetting.

Table~\ref{Tab:odin_csi} reports the results of using CSI and ODIN. ODIN is a weaker OOD method than CSI. 
Both HAT and Sup improve greatly as the systems are equipped with a better OOD detection method CSI.
These experiment results empirically demonstrate the efficacy of Theorem~\ref{thm:cil_with_op_and_ood}, i.e., the CIL performance can be improved if a better OOD detection method is used.

\subsubsection{Full Comparison of HAT+CSI and Sup+CSI with Baselines} \label{sec.HAT+CSI}

\begin{table}[t]
\centering
\caption{Average accuracy (CIL) of all methods after all tasks are learned. The baselines are grouped into (a), (b), (c), and (d) for projection, regularization, replay, and parameter-isolation methods, respectively. Our proposed methods are grouped into (e). Exemplar-free methods are italicized. $\dagger$ indicates that in their original papers, PASS and Mnemonics are pre-trained with the first half of the classes. Their results with pre-train are 50.1 and 53.5 on C100-10T, respectively, which are still much lower than the proposed HAT+CSI and Sup+CSI without pre-training. We do not use pre-training in our experiment for fairness. $*$ indicates that iCaRL and Mnemonics report average incremental accuracy in their original papers. We report average accuracy over all classes after all tasks are learned. The last column Avg. shows the average CIL accuracy of each method over all datasets.
}
\resizebox{0.80\columnwidth}{!}{
\begin{tabular}{l l c c c c c c}
\toprule
& \multirow{1}{*}{Method}  &  \multicolumn{1}{c}{C10-5T}  &  \multicolumn{1}{c}{C100-10T} &  \multicolumn{1}{c}{C100-20T} &  \multicolumn{1}{c}{T-5T} & \multicolumn{1}{c}{T-10T} & \multicolumn{1}{c}{Avg.} \\
\midrule
\multirow{1}{*}{(a)} & \textit{OWM} & 51.8\scalebox{1.0}{$\pm$0.05} & 28.9\scalebox{1.0}{$\pm$0.60} & 24.1\scalebox{1.0}{$\pm$0.26} & 10.0\scalebox{1.0}{$\pm$0.55} & 8.6\scalebox{1.0}{$\pm$0.42} & 24.7 \\
\hline
\multirow{2}{*}{(b)} & \textit{MUC} & 52.9\scalebox{1.0}{$\pm$1.03} & 30.4\scalebox{1.0}{$\pm$1.18} & 14.2\scalebox{1.0}{$\pm$0.30} & 33.6\scalebox{1.0}{$\pm$0.19} & 17.4\scalebox{1.0}{$\pm$0.17} & 29.7 \\
& \textit{PASS}$^{\dagger}$ & 47.3\scalebox{1.0}{$\pm$0.98} & 33.0\scalebox{1.0}{$\pm$0.58} & 25.0\scalebox{1.0}{$\pm$0.69} & 28.4\scalebox{1.0}{$\pm$0.51} & 19.1\scalebox{1.0}{$\pm$0.46} & 30.6 \\
\hline
\multirow{10}{*}{(c)} & LwF & 54.7\scalebox{1.0}{$\pm$1.18} & 45.3\scalebox{1.0}{$\pm$0.75} & 44.3\scalebox{1.0}{$\pm$0.46} & 32.2\scalebox{1.0}{$\pm$0.50} & 24.3\scalebox{1.0}{$\pm$0.26} & 40.2 \\
& iCaRL$^*$  & 63.4\scalebox{1.0}{$\pm$1.11} & 51.4\scalebox{1.0}{$\pm$0.99} & 47.8\scalebox{1.0}{$\pm$0.48} & 37.0\scalebox{1.0}{$\pm$0.41} & 28.3\scalebox{1.0}{$\pm$0.18} & 45.6 \\
& A-GEM & 20.0\scalebox{1.0}{$\pm$0.37} & 9.3\scalebox{1.0}{$\pm$0.17} & 4.1\scalebox{1.0}{$\pm$0.89} & 13.5\scalebox{1.0}{$\pm$0.08} & 7.7\scalebox{1.0}{$\pm$0.07} & 10.9 \\
& EEIL & 57.1\scalebox{1.0}{$\pm$0.28} & 49.0\scalebox{1.0}{$\pm$1.27} & 33.5\scalebox{1.0}{$\pm$0.08} & 14.7\scalebox{1.0}{$\pm$0.40} & 9.8\scalebox{1.0}{$\pm$0.19} & 32.8 \\
& GD & 58.7\scalebox{1.0}{$\pm$0.31} & 49.7\scalebox{1.0}{$\pm$0.33} & 38.9\scalebox{1.0}{$\pm$0.02} & 16.4\scalebox{1.0}{$\pm$1.40} & 11.7\scalebox{1.0}{$\pm$0.25} & 35.1 \\
& Mnemonics$^{\dagger *}$ & 64.1\scalebox{1.0}{$\pm$1.47} & 51.0\scalebox{1.0}{$\pm$0.34} & 47.6\scalebox{1.0}{$\pm$0.74} & 37.1\scalebox{1.0}{$\pm$0.46} & 28.5\scalebox{1.0}{$\pm$0.72} & 45.7 \\
& BiC & 61.4\scalebox{1.0}{$\pm$1.74} & 52.9\scalebox{1.0}{$\pm$0.64} & 48.9\scalebox{1.0}{$\pm$0.54} & 41.7\scalebox{1.0}{$\pm$0.74} & 33.8\scalebox{1.0}{$\pm$0.40} & 47.7 \\
& DER++ & 66.0\scalebox{1.0}{$\pm$1.20} & 53.7\scalebox{1.0}{$\pm$1.30} & 46.6\scalebox{1.0}{$\pm$1.44} & 35.8\scalebox{1.0}{$\pm$0.77} & 30.5\scalebox{1.0}{$\pm$0.47} & 46.5 \\
& HAL & 32.8\scalebox{1.0}{$\pm$2.17} & 15.6\scalebox{1.0}{$\pm$0.31} & 13.5\scalebox{1.0}{$\pm$1.53} & 3.4\scalebox{1.0}{$\pm$0.35} & 3.4\scalebox{1.0}{$\pm$0.38} & 13.7 \\
& Co$^2$L & 65.6 &  &  &  &  \\
\hline
\multirow{6}{*}{(d)} & \textit{CCG}  & 70.1 &  &  &  &  & \\
& \textit{HAT} & 62.7\scalebox{1.0}{$\pm$1.45} & 41.1\scalebox{1.0}{$\pm$0.93} & 25.6\scalebox{1.0}{$\pm$0.51} & 38.5\scalebox{1.0}{$\pm$1.85} & 29.8\scalebox{1.0}{$\pm$0.65} & 39.5 \\
& \textit{HyperNet} & 53.4\scalebox{1.0}{$\pm$2.19} & 30.2\scalebox{1.0}{$\pm$1.54} & 18.7\scalebox{1.0}{$\pm$1.10} & 7.9\scalebox{1.0}{$\pm$0.69} & 5.3\scalebox{1.0}{$\pm$0.50} & 23.1 \\
& \textit{Sup} & 62.4\scalebox{1.0}{$\pm$1.45} & 44.6\scalebox{1.0}{$\pm$0.44} & 34.7\scalebox{1.0}{$\pm$0.30} & 41.8\scalebox{1.0}{$\pm$1.50} & 36.5\scalebox{1.0}{$\pm$0.36} & 44.0 \\
& \textit{PR-Ent} & 61.9 & 45.2 & & & \\
\hline
\multirow{4}{*}{(e)} &  \textit{HAT+CSI} & 87.8\scalebox{1.0}{$\pm$0.71} & 63.3\scalebox{1.0}{$\pm$1.00} & 54.6\scalebox{1.0}{$\pm$0.92} & 45.7\scalebox{1.0}{$\pm$0.26} & 47.1\scalebox{1.0}{$\pm$0.18} & 59.7 \\
&  \textit{Sup+CSI} & 86.0\scalebox{1.0}{$\pm$0.41} & 65.1\scalebox{1.0}{$\pm$0.39} & 60.2\scalebox{1.0}{$\pm$0.51} & 48.9\scalebox{1.0}{$\pm$0.25} & 45.7\scalebox{1.0}{$\pm$0.76} & 61.2 \\
& HAT+CSI+c & 88.0\scalebox{1.0}{$\pm$0.48} & 65.2\scalebox{1.0}{$\pm$0.71} & 58.0\scalebox{1.0}{$\pm$0.45} & 51.7\scalebox{1.0}{$\pm$0.37} & 47.6\scalebox{1.0}{$\pm$0.32} & 62.1 \\
& Sup+CSI+c & 87.3\scalebox{1.0}{$\pm$0.37} & 65.2\scalebox{1.0}{$\pm$0.37} & 60.5\scalebox{1.0}{$\pm$0.64} & 49.2\scalebox{1.0}{$\pm$0.28} & 46.2\scalebox{1.0}{$\pm$0.53} & 61.7  \\ 
\bottomrule
\end{tabular}
}
\label{Tab:maintable_resnet}
 \vspace{-3mm}
\end{table}

\begin{table}[t]
\centering
\caption{TIL (WP) accuracy results of 3 best-performing baselines and our methods. The full results are given in \ref{apx:til_results}. The calibrated versions (+c) of our methods are omitted as calibration does not affect TIL performances.}
\resizebox{0.9\columnwidth}{!}{
\begin{tabular}{l c c c c c c}
\toprule
\multirow{1}{*}{Method}  &  \multicolumn{1}{c}{C10-5T}  &  \multicolumn{1}{c}{C100-10T} &  \multicolumn{1}{c}{C100-20T} &  \multicolumn{1}{c}{T-5T} & \multicolumn{1}{c}{T-10T} & \multicolumn{1}{c}{Avg.} \\
\midrule
BiC & 95.4\scalebox{1.0}{$\pm$0.35} & 84.6\scalebox{1.0}{$\pm$0.48} & 88.7\scalebox{1.0}{$\pm$0.19} & 61.5\scalebox{1.0}{$\pm$0.60} & 62.2\scalebox{1.0}{$\pm$0.45} & 78.5 \\
HAT & 96.7\scalebox{1.0}{$\pm$0.18} & 84.0\scalebox{1.0}{$\pm$0.23} & 85.0\scalebox{1.0}{$\pm$0.98} & 61.2\scalebox{1.0}{$\pm$0.72} & 63.8\scalebox{1.0}{$\pm$0.41} & 78.1 \\
Sup & 96.6\scalebox{1.0}{$\pm$0.21} & 87.9\scalebox{1.0}{$\pm$0.27} & 91.6\scalebox{1.0}{$\pm$0.15} & 64.3\scalebox{1.0}{$\pm$0.24} & 68.4\scalebox{1.0}{$\pm$0.22} & 81.8 \\
\hline
HAT+CSI & 98.7\scalebox{1.0}{$\pm$0.06} & 92.0\scalebox{1.0}{$\pm$0.37} & 94.3\scalebox{1.0}{$\pm$0.06} & 68.4\scalebox{1.0}{$\pm$0.16} & 72.4\scalebox{1.0}{$\pm$0.21} & 85.2 \\
Sup+CSI & 98.7\scalebox{1.0}{$\pm$0.07} & 93.0\scalebox{1.0}{$\pm$0.13} & 95.3\scalebox{1.0}{$\pm$0.20} & 65.9\scalebox{1.0}{$\pm$0.25} & 74.1\scalebox{1.0}{$\pm$0.28} & 85.4 \\
\bottomrule
\end{tabular}
}
\label{Tab:til}
\vspace{-2mm}
\end{table}

We now make a full comparison of the two proposed systems HAT+CSI and Sup+CSI designed based on the theoretical results with baselines. 
Since HAT and Sup are exemplar-free CL methods, HAT+CSI and Sup+CSI do not need to save any previous task data for replaying. Table~\ref{Tab:maintable_resnet} shows that HAT and Sup equipped with CSI outperform the baselines by large margins.
DER++, the best replay method, achieves 66.0 and 53.7 on C10-5T and C100-10T, respectively, while HAT+CSI achieves 87.8 and 63.3 {and Sup+CSI achieves 86.0 and 65.1}. The large performance gap remains consistent in more challenging problems, T-5T and T-10T. 

Due to the definition of OOD in the prediction method and the fact that each task is trained separately in HAT and Sup, the outputs $f(x)_k$ from different tasks can be in different scales, which will result in incorrect predictions. To deal with the problem, we can calibrate the output as $\alpha_k f(x)_k + \beta_k$ and use $OOD_k = \sigma ( \alpha_k f(x)_k + \beta_k )$. The optimal $\alpha_k^*$ and $\beta_k^*$ for each task $k$ can be found by optimization with a memory buffer to save a very small number of training examples from previous tasks like that in the replay-based methods. We refer to the calibrated methods as HAT+CSI+c and Sup+CSI+c. They are trained by using a memory buffer of the same size as the replay methods (see Section \ref{sec.training}). Table~\ref{Tab:maintable_resnet} shows that the calibration improves from their memory-free versions, i.e., without calibration. We provide the details about how to train the calibration parameters $\alpha_k$ and $\beta_k$ in \ref{apx:calibration}.

We note that CSI uses extensive data augmentations in its OOD detection. However, the baseline systems do not. To be fair, we added the same data augmentations to the three top-performing baselines, Mnemonics, BiC, and DER++. The average accuracy values over the five CIL experiments are 36.66, 35.75, and 18.43 for Mnemonics, BiC, and DER++, respectively. Our methods, HAT+CSI, Sup+CSI, HAT+CSI+c, and Sup+CSI+c, achieve accuracy values of 59.7, 61.2, 62.1, and 61.7, respectively, which are significantly better. In fact, with the augmentations, the three baselines perform worse than their original versions.
We believe the reason is that while augmentations improve the performance of the current task as they help learn finer-grained and more task-specific features, they also cause more model updates in learning a task due to the additional augmented data, which leads to significantly more forgetting of prior tasks. However, our technique incorporating robust TIL mechanisms prevents forgetting while also concurrently benefiting from the strong OOD detection performance for each task model of CSI, which exploits data augmentations.

Finally, as shown in Theorem~\ref{thm:ce}, the CIL performance also depends on the TIL (WP) performance. We compare the TIL accuracies of the baselines and our methods in Table~\ref{Tab:til}. Our systems again outperform the baselines by large margins on more challenging datasets (e.g., CIFAR100 and Tiny-ImageNet).

\section{Proposed Approach 2: Out-of-Distribution Replay} 
\label{sec.more}
{The approach presented above does not save any training data from previous tasks except for the optional step of calibration. 
The method presented in this section is based on the replay approach to solving CIL, which saves a small number of training data from each previous task. The proposed method is called \textit{M}ulti-head model for continual learning via \textit{O}OD \textit{RE}play (MORE). As mentioned in Section~\ref{sec.clom}, the OOD detection method used in this section is a closed-world method as it uses the saved samples from previous tasks as OOD samples in learning each new task.

\subsection{The Proposed MORE Technique}
\label{sec.ood-reply}

Recall a replay-based method for continual learning works by memorizing or saving a small subset of the training samples from each previous task in a memory buffer. The saved data is called the \textit{replay data}. In learning a new task, the new task data and the replay data are trained jointly to update the model. Clearly, using the replay data can partially deal with the inter-class separation (ICS) problem because the model sees some data from all classes learned so far. However, it cannot solve the ICS problem completely because the amount of replay data is often very small. 

Unlike existing replay-based CIL methods, which simply use the replay data to update the decision boundaries between the old class and the new classes (in the new task), the proposed method uses the replay data to build an OOD detection model for each task in continual learning, which gives the name of the proposed method, i.e., \textit{out-of-distribution replay}. Further, unlike existing OOD detection methods, which usually do not use any OOD data in training, the proposed method uses the replay data from previous tasks as the OOD data for the current new task in building its OOD detection model. 

Unlike HAT+CSI and Sup+CSI, which do not use a pre-trained network, MORE trains a multi-head network as an adapter \citep{houlsby2019parameter_adapter} to a pre-trained network (see Figure~\ref{illustration}(b)). 
Note that using a pre-trained transformer network and adapter modules is a common practice in existing continual learning methods in the natural language processing community \citep{ke2021achieving,ke2022continual}. Here we also leverage this approach for image classification tasks.
In continual learning, the pre-trained network is frozen, only the adapters and the norm layers are trainable. Similar to HAT+CSI, a hard attention mask (HAT) is again employed to protect each task model or classifier to avoid forgetting. Each head is also an OOD detection model for a task, but, as mentioned above, MORE uses the replay data as the OOD data to build an OOD detection model.  Since HAT has been described in Section~\ref{sec.hat}, we will not discuss it further except to state that we need to use  $\mathcal{L}_{ood}$ in Eq.~\ref{ood_obj} to replace $\mathcal{L}_{ce}$  in Eq.~\ref{eq.loss_hat} after incorporating the trainable embedding $\ve^{k}$. We describe the whole training and prediction process in \ref{appendix:pseudo}.

\subsubsection{Training an OOD Detection Model}\label{sec.training_ood}

At task $k$, the system receives the training data $\mathcal{D}_k=\{(x_k^i, y_k^i)_{i=1}^{n_k}\}$, where $n_k$ is the number of samples, and $x_k^i \in \mathbf{X}_k$ is an input sample and $y_k^i \in \mathbf{Y}_k$ (the set of all classes of task $k$) is its class label.
We train the feature extractor $z = h(x, k; \theta)$ and task-specific classifier $f(z; \phi_k)$ using $\mathcal{D}_{k}$ and the samples in the memory buffer $\mathcal{M}$. We treat the buffer data as OOD data to encourage the network to learn the current task and also detect ODD samples (the models or classifiers of the previous tasks are not touched). We achieve it by maximizing $p(y | x, k) = \text{softmax} f(h(x, k; \theta); \phi_k)$ for an IND sample $x \in \mathbf{X}_k$ and maximizing $p(ood|x, k)$ for an OOD sample $x \in \mathcal{M}$. The additional label $ood$ is reserved for previous and possible future unseen classes.
Figure~\ref{illustration}(a) shows the overall idea of the proposed approach. We formulate the problem as follows.
\begin{figure}
\centering
\includegraphics[width=1.00\linewidth]{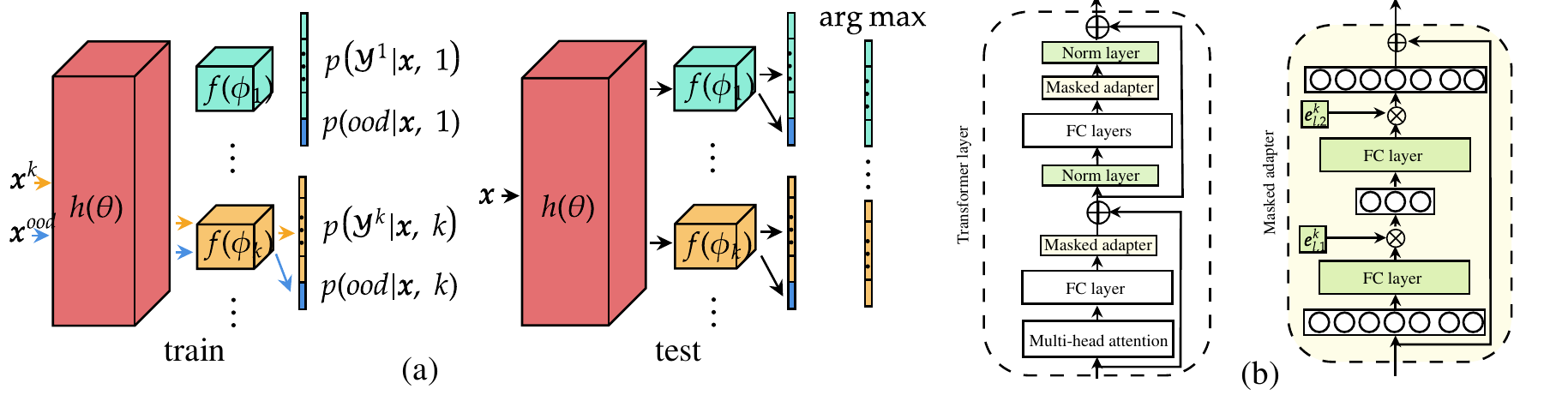}
\caption{
(a) We train the feature extractor and the task classifier $k$ at task $k$.
The output values of the classifier correspond to $|\mathbf{Y}_k| + 1$ classes, in which the last class is for OOD (i.e., representing previous and unseen future classes).
At inference/testing, the probability values of each task model without the OOD class are concatenated and the system chooses the class with the maximum score. (b) Transformer and adapter module. The masked adapter network consists of 2 fully connected layers and task-specific masks. During training, only the masked adapters and norm layers are updated and the other parts in the transformer layers remain unchanged.}
\label{illustration}
\end{figure}

Given the training data $\mathbf{X}_k$
of size $n_k$ at task $k$ and the memory buffer $\mathcal{M}$ of size $M$, we minimize the loss
\begin{align}
    \begin{split}
        \mathcal{L}_{ood} (\theta, \phi_k) = - \frac{1}{M + N} \left( \sum_{(x, y) \in \mathcal{M}} \log p(ood | x, k) + \sum_{(x, y) \in \mathcal{D}^{k}} \log p(y | x, k) \right) \label{ood_obj}
    \end{split}
\end{align}
It is the sum of two cross-entropy losses. The first loss is for learning OOD samples while the second loss is for learning the classes from the current task. We optimize the shared parameter $\theta$ in the feature extractor. The task-specific classification parameters $\phi_k$ are independent of other tasks. 
The learned representation of the current data should be robust to OOD data. The classifier thus can classify both IND and OOD data.

In testing, we perform prediction by comparing the softmax probability output values using all the task classifiers from task 1 to $k$ without the OOD class as
\begin{align}
    \hat{y} = \argmax \bigoplus_{1 \leq j \leq k} p(\mathbf{Y}_j| x, j) \label{base_prediction}
\end{align}
where $\bigoplus$ is the concatenation over the output space. Figure~\ref{illustration}(a) shows the prediction rule. 
We are basically choosing the class with the highest softmax probability over all classes from all learned tasks.

\subsubsection{Back-Updating the Previous OOD Models}\label{sec.backward}
Each task model works better if more diverse OOD data is provided during training. As in a replay-based approach, MORE saves an equal number of samples per class after each task~\citep{chaudhry2019continual_er}. The saved samples in the memory are used as OOD samples for each new task. 
Thus, in the beginning of continual learning when the system is trained on only a small number of tasks, the classes of samples in the memory are less diverse than after more tasks are learned. This makes the performance of OOD detection stronger for later tasks, but weaker in earlier tasks. To prevent this asymmetry, we update the model of each previous task so that it can also identify the samples from subsequent classes (which were unseen during the training of the previous task) as OOD samples.

At task $k$, we update the previous task models $(j=1, \cdots, k-1)$ as follows. Denote the samples of task $j$ in memory $\mathcal{M}$ by $\tilde{\mathbf{X}}_j$. We construct a new dataset using the current task dataset and the samples in the memory buffer. We randomly select $|\mathcal{M}|$ samples from the training data $\mathcal{D}_{k}$ and pool them with the remaining samples in $\mathcal{M}$ after removing the IND samples $\tilde{\mathcal{D}_{j}}$ of task $j$ from $\mathcal{M}$. We do not use the entire training data $\mathcal{D}_{k}$ as we do not want a large sample imbalance between IND and OOD.
Denote the new dataset by $\tilde{\mathcal{M}}$.
Using the data, we update only the parameters $\phi_j$ of the classifier for task $j$ with the feature representations frozen by minimizing the loss
\begin{align}
    \mathcal{L}(\phi_{j}) = - \frac{1}{2M} \left( \sum_{(x, y) \in \tilde{\mathcal{M}}} \log p(ood | x, j) + \sum_{(x, y) \in \tilde{\mathcal{D}}_{j}} \log p(y | x, j) \right) \label{backward}
\end{align}
We reduce the loss by updating the parameters of classifier $j$ to maximize the probability of the class if the sample belongs to task $j$ and maximize the OOD probability otherwise. 

\subsubsection{Improving Prediction Performance by a Distance Based Technique}\label{sec.ensemble_scores}
We further improve the prediction in Eq.~\ref{base_prediction} by introducing a distance-based factor used as a coefficient to the softmax probabilities in Eq.~\ref{base_prediction}. It is quite intuitive that if a test instance is close to a class, it is more likely to belong to the class. We thus propose to combine this new distance factor and the softmax probability output of the task $j$ model to make the final prediction decision. In some sense, this can be considered as an ensemble of the two methods.

We define the distance-based coefficient $s_{j}(x)$ of task $j$ for the test instance $x$ by the maximum of inverse Mahalanobis distance~\citep{lee2018simple} between the feature of $x$ and the Gaussian distributions of the classes in task $j$ parameterized by the mean $\mu_j^i$ of the class $i$ in task $j$ and the sample covariance $S_{j}$. They are estimated by the features of class $i$'s training data for each class $i$ in task $j$.
If a test instance is from the task, its feature should be close to the distribution that the instance belongs to. Conversely, if the instance is OOD to the task, its feature should not be close to any of the distributions of the classes in the task.
More precisely, for task $j$ with class $y_1, \cdots, y_{|\mathbf{Y}_{j}|}$ (where ${|\mathbf{Y}_{j}|}$ represents the number of classes in task $k$), we define the coefficient $s_{j}(x)$ as 
\begin{align}
    s_{j}(x) = \max \left[ 1/\text{MD}(x; \mu_j^{y_{1}}, S_{j}), \cdots, 1/\text{MD}(x; \mu_j^{y_{|\mathbf{Y}_{j}|}}, S_{j}) \right] \label{eq.ensemble}
\end{align}
{$\text{MD}(x; \mu_j^{i}, S_{j})$ is the Mahalanobis distance.
The coefficient is large if at least one of the Mahalanobis distances is small but the coefficient is small if all the distances are large (i.e. the feature is far from all the distributions of the task).
The parameters $\mu_j^{i}$ and $S_{j}$ can be computed and saved when each task is learned. The mean $\mu_j^{i}$ 
is computed using the training samples $\mathcal{D}^i_{j}$ of class $i$ as follows,
\begin{align}
    \mu_j^{i} = \sum_{x \in \mathcal{D}^{i}_{j}} h(x, i) / |\mathcal{D}^{i}_{j}| \label{eq.compute_mu}
\end{align}
and the covariance $S_{j}$ of task $j$ is the mean of covariances of the classes in task $j$,
\begin{align}
    S_{j} = \sum_{i \in \mathbf{Y}_{j}} S^{i}_{j} / |\mathbf{Y}_{j}| \label{eq.compute_sigma}
\end{align}
where $S^{i}_{j} = \sum_{x \in \mathcal{D}^{i}_{j}} (x - \mu_j^{i})^{T}(x - \mu_j^{i}) / |\mathcal{D}^{i}_{j}|$ is the sample covariance of class $i$. 
By multiplying the coefficient $s_{j}(x)$ to the original softmax probabilities $p(\mathbf{Y}_{j} | x, j)$, the task output $p(\mathbf{Y}_{j} | x, j) s_{j}(x)$ increases if $x$ is from task $j$ and decreases otherwise. The final prediction is made by (which replaces Eq.~\ref{base_prediction})
\begin{align}
    y = \argmax \bigoplus_{1\leq j \leq k}  p(\mathbf{Y}_{j} | x, j) s_{j}(x), \label{final_prediction}
\end{align}
where $k$ is the last task that we have learned.  
}

\subsection{Experiments}
\label{sec.experiment2}
We now report the experiment results of the proposed method MORE. For \textbf{experimental datasets}, we use the same three image classification benchmark datasets as in Section~\ref{sec.baselines}.
For \textbf{baselines}, the same systems are used as well (see Section~\ref{sec.baselines}) except Mnemonics, HyperNet, CCG, {Co$^2$L}, and {PR-Ent}. Mnemonics requires optimization of training instances and it is not clear how to implement it for images after interpolation for a given input size of a pre-trained model (see below). For HyperNet, it is due to the reason explained in Training Details in Section~\ref{sec.training}. For CCG,  {Co$^2$L}, and {PR-Ent}, CCG has no code and the codes of {Co$^2$L}, and {PR-Ent} do not run in our environment and thus we could not convert their codes to use a pre-trained model. Finally, we are left with 13 baselines. Note that 
HAT-CSI and Sup+CSI are not included as they are much weaker (up to 15\% lower than MORE in accuracy) as CSI's approach of using contrastive learning and data augmentations does not work well with a pre-trained
model.

\textbf{Evaluation Metrics.} We still use the same evaluation measures as we used in Section~\ref{sec.training}. 
\textbf{(1).} \textit{Average classification accuracy} over all classes after learning the last task.
\textbf{(2).} \textit{Average AUC} (Area Under the ROC Curve) for evaluating OOD detection performance of continual learning in the open world. See Section \ref{sec.ood} for more details. 
}

\subsubsection{Pre-trained Network}
\label{sec.pre-train}

We pre-train a vision transformer~\citep{touvron2021training_deit} using a subset of the ImageNet data~\citep{russakovsky2015imagenet} and apply the pre-trained network/model to all baselines and our method. To ensure that there is no overlapping of data between ImageNet and our experimental datasets, we manually removed 389 classes from the original 1000 classes in ImageNet that are similar/identical to the classes in CIFAR-10, CIFAR-100, or Tiny-ImageNet. We pre-train the network with the remaining subset of 611 classes of ImageNet.

Using the pre-trained network, both our system and the baselines improve dramatically compared to their versions without using the pre-trained network. For instance, the two best baselines (DER++ and PASS) in our experiments achieved the average classification accuracy of 66.89 and 68.25 (after the final task) with the pre-trained network over 5 experiments while they achieved only 46.88 and 32.42 without using the pre-train network.

We insert an adapter module at each transformer layer to exploit the pre-trained transformer network in continual learning. During training, the adapter module and the layer norm are trained while the transformer parameters are unchanged to prevent forgetting in the pre-trained network. 

\subsubsection{Training Details}

For all experiments, {we use the same backbone architecture DeiT-S/16 \citep{touvron2021training_deit} with a 2-layer adapter \citep{houlsby2019parameter_adapter} at each transformer layer}, and the same class order for both baselines and our method. The first fully connected layer in the adapter maps from dimension 384 to the bottleneck. 
The second fully connected layer following ReLU activation function maps from bottleneck to 384. The bottleneck dimension is the same for all adapters in a model. For our method, we use SGD with a momentum value 0.9.  The back-updating method in Section \ref{sec.backward} is also a hyper-parameter choice. If we apply it, we train each classifier for 10 epochs by SGD with a learning rate 0.01, batch size 16, and momentum value 0.9. We choose 500 for $s$ in Eq.~\ref{eq.smax} and 0.75 for $\lambda$ in Eq.~\ref{eq.hat_reg} 
as recommended in \citep{Serra2018overcoming}. 
We find a good set of learning rates and number of epochs on the validation set made of 10\% of the training data.
We follow~\citep{chaudhry2019continual_er} and save an equal number of random samples per class in the replay memory. Following the experiment settings in~\citep{Rebuffi2017, Zhu_2021_CVPR_pass}, we fix the size of the memory buffer and reduce the saved samples to accommodate a new set of samples after a new task is learned. We use the class order protocol in \citep{Rebuffi2017,NEURIPS2020_b704ea2c_derpp} by generating random class orders for the experiments. The baselines and our method use the same class ordering. We also report the size of memory required for each experiment in \ref{appendix:size_of_memory}.

For CIFAR-10, we split 10 classes into 5 tasks (2 classes per task).
The bottleneck size in each adapter is 64. Following \citep{NEURIPS2020_b704ea2c_derpp}, we use the memory size 200, and train for 20 epochs with a learning rate 0.005, and apply the back-updating method in Section~\ref{sec.backward}.

For CIFAR-100, we conducted 10 tasks and 20 tasks experiments, where each task has 10 classes and 5 classes, respectively.
We double the bottleneck size of the adapter to learn more classes. We use the memory size 2000 {following \citep{Rebuffi2017}} and train for 40 epochs with learning rates 0.001 and 0.005 for 10 tasks and 20 tasks, respectively, and apply the back-updating method in Section~\ref{sec.backward}.

For Tiny-ImageNet, two experiments are conducted. We split 200 classes into 5 and 10 tasks, where each task has 40 classes and 20 classes per task, respectively. We use the bottleneck size 128, and save 2000 samples in memory. We train with the learning rate 0.005 for 15 and 10 epochs for 5 tasks and 10 tasks, respectively.
There is no need to use the back-updating method as the earlier tasks already have diverse OOD classes.

\subsubsection{Accuracy and Forgetting Rate Results and Analysis}\label{sec:result_comparison}

\begin{table*}[t]
\centering
\caption{Average accuracy after the final task. `-XT' means X number of tasks. Our system MORE and all baselines use the pre-trained network. The baselines are grouped into (a), (b), (c), and (d) for projection, regularization, replay, and parameter-isolation methods, respectively. The last column shows the average accuracy 
of each method over all datasets and experiments. We highlight the best results in each column in bold. }
\resizebox{1.0\columnwidth}{!}{
\begin{tabular}{l l c c c c c c}
\toprule
& \multirow{1}{*}{Method}  & \multicolumn{1}{c}{C10-5T}  &  \multicolumn{1}{c}{C100-10T} &  \multicolumn{1}{c}{C100-20T} &  \multicolumn{1}{c}{T-5T} & \multicolumn{1}{c}{T-10T} & \multicolumn{1}{c}{Avg.}\\
\midrule
\multirow{1}{*}{(a)} & OWM             &  41.69\scalebox{1.0}{$\pm$6.34}  & 21.39\scalebox{1.0}{$\pm$3.18} & 16.98\scalebox{1.0}{$\pm$4.44} & 24.55\scalebox{1.0}{$\pm$2.48} & 17.52\scalebox{1.0}{$\pm$3.45} & 24.43 \\
\hline
\multirow{2}{*}{(b)} & MUC & 73.95\scalebox{1.0}{$\pm$7.24} & 57.87\scalebox{1.0}{$\pm$1.11} & 43.98\scalebox{1.0}{$\pm$2.68} & 62.47\scalebox{1.0}{$\pm$0.34} & 55.79\scalebox{1.0}{$\pm$0.49} & 58.81 \\
& PASS          &  86.21\scalebox{1.0}{$\pm$1.10}  & 68.90\scalebox{1.0}{$\pm$0.94} & 66.77\scalebox{1.0}{$\pm$1.18} & 61.03\scalebox{1.0}{$\pm$0.38} & 58.34\scalebox{1.0}{$\pm$0.42} & 68.25 \\ 
\hline
\multirow{8}{*}{(c)} & LwF & 67.59\scalebox{1.0}{$\pm$4.27} & 66.50\scalebox{1.0}{$\pm$1.93} & 67.54\scalebox{1.0}{$\pm$0.97} & 33.51\scalebox{1.0}{$\pm$4.36} & 36.85\scalebox{1.0}{$\pm$4.46} & 54.40 \\
& iCaRL         & 87.55\scalebox{1.0}{$\pm$0.99} & 68.90\scalebox{1.0}{$\pm$0.47} & 69.15\scalebox{1.0}{$\pm$0.99} & 53.13\scalebox{1.0}{$\pm$1.04} & 51.88\scalebox{1.0}{$\pm$2.36} & 66.12 \\
& A-GEM         & 56.33\scalebox{1.0}{$\pm$7.77} & 25.21\scalebox{1.0}{$\pm$4.00} & 21.99\scalebox{1.0}{$\pm$4.01} & 30.53\scalebox{1.0}{$\pm$3.99} & 21.90\scalebox{1.0}{$\pm$5.52} & 31.20 \\ 
& EEIL & 82.34\scalebox{1.0}{$\pm$3.13} & 68.08\scalebox{1.0}{$\pm$0.51} & 63.79\scalebox{1.0}{$\pm$0.66} & 53.34\scalebox{1.0}{$\pm$0.54} & 50.38\scalebox{1.0}{$\pm$0.97} & 63.59 \\
& GD & \textbf{89.16}\scalebox{1.0}{$\pm$0.53} & 64.36\scalebox{1.0}{$\pm$0.57} & 60.10\scalebox{1.0}{$\pm$0.74} & 53.01\scalebox{1.0}{$\pm$0.97} & 42.48\scalebox{1.0}{$\pm$2.53} & 61.82 \\
& BiC & 67.44\scalebox{1.0}{$\pm$3.93} & 64.47\scalebox{1.0}{$\pm$1.30} & 67.69\scalebox{1.0}{$\pm$1.97} & 38.78\scalebox{1.0}{$\pm$1.26} & 40.98\scalebox{1.0}{$\pm$2.39} & 55.87 \\
& DER++         & 84.63\scalebox{1.0}{$\pm$2.91} & 69.73\scalebox{1.0}{$\pm$0.99} & 70.03\scalebox{1.0}{$\pm$1.46} & 55.84\scalebox{1.0}{$\pm$2.21} & 54.20\scalebox{1.0}{$\pm$3.28} & 66.89 \\
& HAL           & 84.38\scalebox{1.0}{$\pm$2.70} & 67.17\scalebox{1.0}{$\pm$1.50} & 67.37\scalebox{1.0}{$\pm$1.45} & 52.80\scalebox{1.0}{$\pm$2.37} & 55.25\scalebox{1.0}{$\pm$3.60} & 65.39 \\
\hline
\multirow{2}{*}{(d)} & HAT          & 83.30\scalebox{1.0}{$\pm$1.54} & 62.34\scalebox{1.0}{$\pm$0.93} & 56.72\scalebox{1.0}{$\pm$0.44} & 57.91\scalebox{1.0}{$\pm$0.72} & 53.12\scalebox{1.0}{$\pm$0.94} & 62.68 \\ 
& Sup & 80.91\scalebox{1.0}{$\pm$2.99} & 62.49\scalebox{1.0}{$\pm$0.49} & 57.32\scalebox{1.0}{$\pm$1.11} & 58.43\scalebox{1.0}{$\pm$0.67} & 54.52\scalebox{1.0}{$\pm$0.45} & 62.74 \\
\hline
& MORE & \textbf{89.16}\scalebox{1.0}{$\pm$0.96} & \textbf{70.23}\scalebox{1.0}{$\pm$2.27} & \textbf{70.53}\scalebox{1.0}{$\pm$1.09} & \textbf{64.97}\scalebox{1.0}{$\pm$1.28}  & \textbf{63.06}\scalebox{1.0}{$\pm$1.26} & \textbf{71.59} \Tstrut \\  
\bottomrule
\end{tabular}
}
\label{Tab:maintable}
\vspace{-3mm}
\end{table*}

\textbf{Average Accuracy.}
Table~\ref{Tab:maintable} shows that our method MORE consistently outperforms the baselines. All the reported results are the averages of 5 runs. The last column gives the average of each row. We compare with the replay-based methods first. The best replay-based method on average over all the datasets is DER++. Our method MORE achieves an accuracy of 71.59, much better than 66.89 of DER++. This demonstrates that the existing replay-based methods utilizing the replay samples to update all learned classes are inferior to our MORE method using samples for OOD learning. The best baseline is the generative method PASS. Its average accuracy over all the datasets is 68.25, which is still poorer than our method's performance of 71.59. The performance of the multi-head method HAT~\citep{Serra2018overcoming} using task-id prediction is only 62.68, which is lower than many other baselines. Its performance is particularly low in experiments where the number of classes per task is small. For instance, its accuracy on C100-20T is 56.72, much lower than our method of 70.53 trained based on OOD detection. 

\textbf{Accuracy with Smaller Memory Sizes.}
For all the datasets, we run additional experiments with half of the original memory size and show that our method is even stronger with a smaller memory. The new memory sizes are 100, 1000, and 1000 for CIFAR-10, CIFAR-100, and Tiny-ImageNet, respectively. Table~\ref{Tab:smaller_memory} shows that
MORE has experienced almost no performance drop
with the reduced memory size while the memory-based baselines suffer from major performance reduction. The accuracy of the best memory-based baseline (DER++) has decreased the accuracy from 66.89 to 62.16, while MORE only decreases from 71.59 to 71.44, which shows that a small number of OOD samples is enough to enable the system to produce a robust OOD detection model. 
\begin{table*}[t]
\centering
\caption{Average accuracy of the baselines and our method MORE with smaller memory sizes. We reduce the size of the memory buffer by half. The new sizes are 100, 1000, and 1000 for CIFAR10, CIFAR100, and Tiny-ImageNet. Numbers in bold are the best results in each column.}
\resizebox{1.0\columnwidth}{!}{
\begin{tabular}{l l c c c c c c}
\toprule
& \multirow{1}{*}{Method}  & \multicolumn{1}{c}{C10-5T}  &  \multicolumn{1}{c}{C100-10T} &  \multicolumn{1}{c}{C100-20T} &  \multicolumn{1}{c}{T-5T} & \multicolumn{1}{c}{T-10T} & \multicolumn{1}{c}{Avg.}\\
\midrule
\multirow{1}{*}{(a)} & OWM             & 41.69\scalebox{1.0}{$\pm$6.34} & 21.39\scalebox{1.0}{$\pm$3.18} & 16.98\scalebox{1.0}{$\pm$4.44} & 24.55\scalebox{1.0}{$\pm$2.48} & 17.52\scalebox{1.0}{$\pm$3.45} & 24.43 \\
\hline
\multirow{2}{*}{(b)} & MUC & 73.95\scalebox{1.0}{$\pm$7.24} & 57.87\scalebox{1.0}{$\pm$1.11} & 43.98\scalebox{1.0}{$\pm$2.68} & 62.47\scalebox{1.0}{$\pm$0.34} & 55.79\scalebox{1.0}{$\pm$0.49} & 58.81 \\
& PASS          & 86.21\scalebox{1.0}{$\pm$1.10} & 68.90\scalebox{1.0}{$\pm$0.94} & 66.77\scalebox{1.0}{$\pm$1.18} & 61.03\scalebox{1.0}{$\pm$0.38} & 58.34\scalebox{1.0}{$\pm$0.42} & 68.25 \\
\hline
\multirow{8}{*}{(c)} & LwF & 63.01\scalebox{1.0}{$\pm$4.19} & 56.76\scalebox{1.0}{$\pm$3.72} & 63.53\scalebox{1.0}{$\pm$2.86} & 26.79\scalebox{1.0}{$\pm$2.36} & 28.08\scalebox{1.0}{$\pm$4.88} & 47.63 \\
& iCaRL         & 86.08\scalebox{1.0}{$\pm$1.19}  &  66.96\scalebox{1.0}{$\pm$2.08} & 68.16\scalebox{1.0}{$\pm$0.71} & 47.27\scalebox{1.0}{$\pm$3.22} & 49.51\scalebox{1.0}{$\pm$1.87} & 63.60 \\ 
& A-GEM         & 56.64\scalebox{1.0}{$\pm$4.29} & 23.18\scalebox{1.0}{$\pm$2.54} & 20.76\scalebox{1.0}{$\pm$2.88} & 31.44\scalebox{1.0}{$\pm$3.84} & 23.73\scalebox{1.0}{$\pm$6.27} & 31.15 \\
& EEIL & 77.44\scalebox{1.0}{$\pm$3.04} & 62.95\scalebox{1.0}{$\pm$0.68} & 57.86\scalebox{1.0}{$\pm$0.74} & 48.36\scalebox{1.0}{$\pm$1.38} & 44.59\scalebox{1.0}{$\pm$1.72} & 58.24 \\
& GD & 85.96\scalebox{1.0}{$\pm$1.64} & 57.17\scalebox{1.0}{$\pm$1.06} & 50.30\scalebox{1.0}{$\pm$0.58} & 46.09\scalebox{1.0}{$\pm$1.77} & 32.41\scalebox{1.0}{$\pm$2.75} & 54.39 \\
& BiC & 56.28\scalebox{1.0}{$\pm$3.31} & 58.42\scalebox{1.0}{$\pm$2.48} & 62.19\scalebox{1.0}{$\pm$1.20} & 33.29\scalebox{1.0}{$\pm$2.65} & 28.44\scalebox{1.0}{$\pm$2.41} & 47.72 \\
& DER++         & 80.09\scalebox{1.0}{$\pm$3.00} & 64.89\scalebox{1.0}{$\pm$2.48} & 65.84\scalebox{1.0}{$\pm$1.46} & 50.74\scalebox{1.0}{$\pm$2.41} & 49.24\scalebox{1.0}{$\pm$5.01} & 62.16 \\
& HAL           & 79.16\scalebox{1.0}{$\pm$4.56} & 62.65\scalebox{1.0}{$\pm$0.83} & 63.96\scalebox{1.0}{$\pm$1.49} & 48.17\scalebox{1.0}{$\pm$2.94} & 47.11\scalebox{1.0}{$\pm$6.00} & 60.21 \\
\hline
\multirow{2}{*}{(d)} & HAT          & 83.30\scalebox{1.0}{$\pm$1.54}  & 62.34\scalebox{1.0}{$\pm$0.93} & 56.72\scalebox{1.0}{$\pm$0.44} & 57.91\scalebox{1.0}{$\pm$0.72} & 53.12\scalebox{1.0}{$\pm$0.94} & 62.68 \\ 
& Sup & 80.91\scalebox{1.0}{$\pm$2.99} & 62.49\scalebox{1.0}{$\pm$0.49} & 57.32\scalebox{1.0}{$\pm$1.11} & 58.43\scalebox{1.0}{$\pm$0.67} & 54.52\scalebox{1.0}{$\pm$0.45} & 62.74 \\
\hline
& MORE  & \textbf{88.13}\scalebox{1.0}{$\pm$1.16} & \textbf{71.69}\scalebox{1.0}{$\pm$0.11} & \textbf{71.29}\scalebox{1.0}{$\pm$0.55} & \textbf{64.17}\scalebox{1.0}{$\pm$0.77} & \textbf{61.90}\scalebox{1.0}{$\pm$0.90} & \textbf{71.44} \Tstrut \\  
\bottomrule
\end{tabular}
}
\label{Tab:smaller_memory}
\end{table*}

\begin{figure}
    \centering
    \includegraphics[width=2.5in]{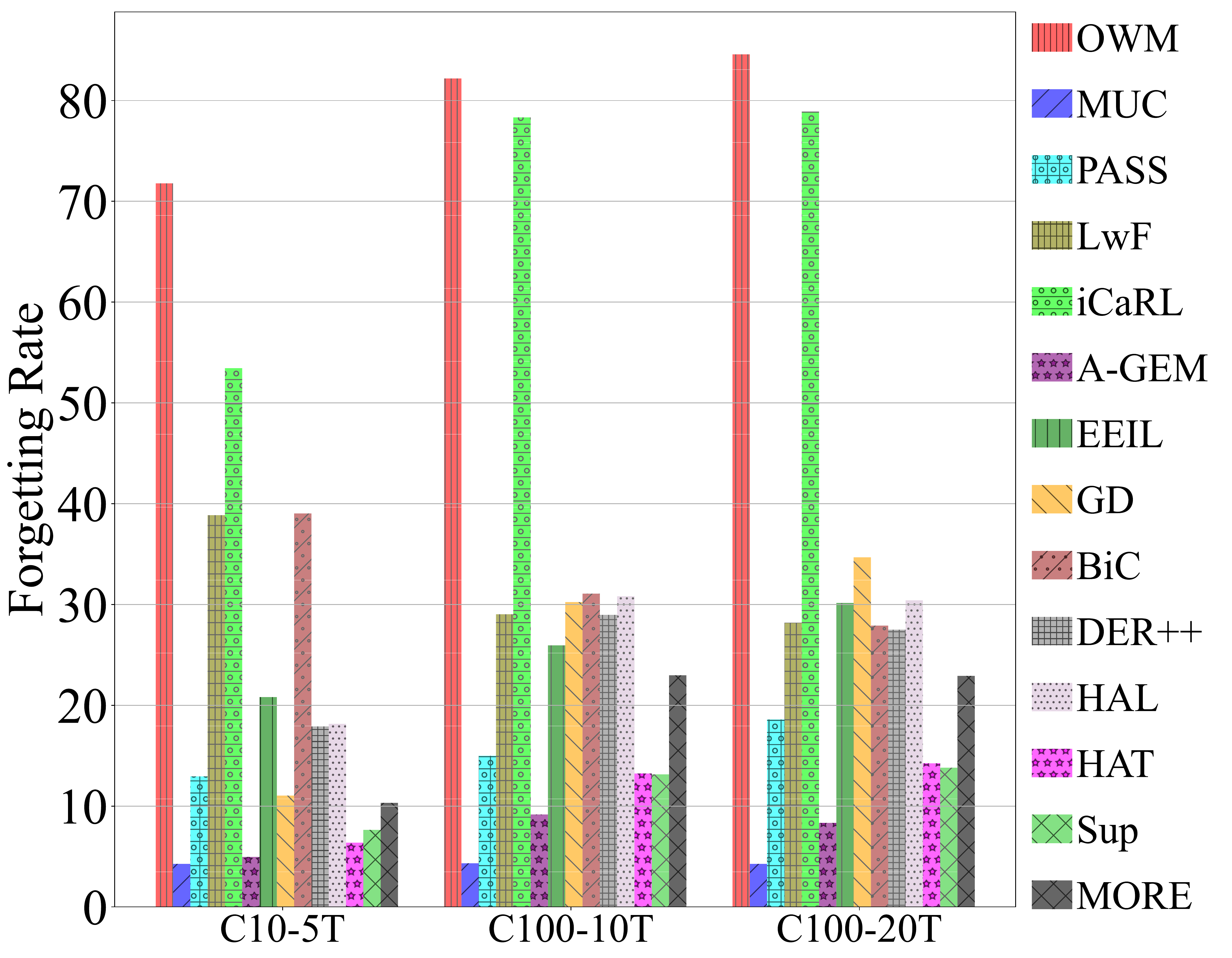}
    \caption{Average forgetting rate (\%). The lower the rate, the better the method is.}
    \label{forgetting}
\end{figure}
\textbf{Average Forgetting Rate.} 
The average forgetting rate is defined as follows~\citep{Liu_2020_CVPR}: $\mathcal{F}^{t} = \sum_{k=1}^{t-1} (\mathcal{A}^{\text{init}}_{k} - \mathcal{A}^{t}_{k}) / (t-1)$, where $\mathcal{A}_{k}^{\text{init}}$ is the classification accuracy on samples of task $k$ right after learning the task $k$. We do not consider the task $t$ as it is the last task. 

We compare the forgetting rate of our method MORE against the baselines using C10-5T, C100-10T, and C100-20T. Figure~\ref{forgetting} shows that the performance drop of our method as more tasks are learned is relatively lower than many baselines. MUC, A-GEM, HAT, and Sup achieve lower drops than our method. {However, they are not able to adapt to new tasks well as the accuracy values of the four methods on C100-10T are 57.87, 25.21, 62.34, and 62.49, respectively, while our method MORE achieves 70.23. The performance gaps remain consistent in the other dataset.}
PASS experiences a smaller drop in performance on C100-10T and C100-20T than our method, but its accuracy of 68.90 and 66.77 are significantly lower than 70.23 and 70.53 of our MORE.

\subsubsection{Out-of-Distribution Detection Results}\label{sec.ood}
As we explained in the introduction section, since our method MORE is based on OOD detection in building each task model, our method can naturally be used to detect test samples that are out-of-distribution for all the classes or tasks learned thus far. We are not aware of any existing continual learning system that has done such an evaluation. We evaluate the performance of the baselines and our system in this out-of-distribution scenario, which is also called the open set setting.

This OOD detection ability is highly desirable for a continual learning system
because in a real-life environment in the open world, the system can be exposed to not only seen classes but also unseen classes. When the test sample is from one of the seen classes, the system should be able to predict its class. If the sample does not belong to any of the training classes seen so far (i.e., the sample is out-of-distribution), the system should detect it.

We formulate the performance of OOD detection of a continual learning system as follows. A continual learning system accepts and classifies a test sample after training task $k$ if the test sample is from one of the classes in tasks $1, \cdots, k$. If it is from one of the classes of the future tasks $k + 1, \cdots, t$, it should be rejected as OOD (where $t$ is the last task in each evaluation). 

\begin{table*}[t]
\centering
\caption{AI-AUC Results. Numbers in bold are the best results in each column}
\resizebox{1.0\columnwidth}{!}{
\begin{tabular}{l l c c c c c c}
\toprule
& \multirow{1}{*}{Method}  & \multicolumn{1}{c}{C10-5T}  &  \multicolumn{1}{c}{C100-10T} &  \multicolumn{1}{c}{C100-20T} &  \multicolumn{1}{c}{T-5T} & \multicolumn{1}{c}{T-10T} & \multicolumn{1}{c}{Avg.}\\
\midrule
\multirow{1}{*}{(a)} & OWM & 70.02\scalebox{1.0}{$\pm$3.59} & 63.17\scalebox{1.0}{$\pm$1.06} & 59.42\scalebox{1.0}{$\pm$1.26} & 67.24\scalebox{1.0}{$\pm$0.92} & 62.17\scalebox{1.0}{$\pm$0.35} & 64.41 \\
\hline
\multirow{2}{*}{(b)} & MUC & 85.47\scalebox{1.0}{$\pm$3.97} & 79.28\scalebox{1.0}{$\pm$1.15} & 74.82\scalebox{1.0}{$\pm$1.91} & \textbf{83.91}\scalebox{1.0}{$\pm$0.54} & \textbf{81.42}\scalebox{1.0}{$\pm$0.47} & 80.98 \\
& PASS & 84.57\scalebox{1.0}{$\pm$1.54} & 77.74\scalebox{1.0}{$\pm$1.40} & 77.42\scalebox{1.0}{$\pm$1.44} & 77.07\scalebox{1.0}{$\pm$2.14} & 74.79\scalebox{1.0}{$\pm$2.36} & 78.32 \\
\hline
\multirow{8}{*}{(c)} & LwF & 72.18\scalebox{1.0}{$\pm$4.15} & 74.95\scalebox{1.0}{$\pm$0.39} & 75.40\scalebox{1.0}{$\pm$0.64} & 66.44\scalebox{1.0}{$\pm$1.14} & 65.52\scalebox{1.0}{$\pm$0.64} & 70.90 \\
& iCaRL & 82.12\scalebox{1.0}{$\pm$5.38} & 77.42\scalebox{1.0}{$\pm$0.45} & 76.91\scalebox{1.0}{$\pm$1.30} & 71.86\scalebox{1.0}{$\pm$1.57} & 74.24\scalebox{1.0}{$\pm$1.66} & 76.06 \\ 
& A-GEM & 74.92\scalebox{1.0}{$\pm$5.62} & 64.19\scalebox{1.0}{$\pm$0.86} & 60.23\scalebox{1.0}{$\pm$0.95} & 67.88\scalebox{1.0}{$\pm$1.28} & 63.08\scalebox{1.0}{$\pm$1.12} & 66.06 \\ 
& EEIL & 87.19\scalebox{1.0}{$\pm$2.31} & 78.89\scalebox{1.0}{$\pm$1.32} & 77.69\scalebox{1.0}{$\pm$1.40} & 74.82\scalebox{1.0}{$\pm$0.79} & 73.45\scalebox{1.0}{$\pm$1.33} & 78.39 \\
& GD & \textbf{89.71}\scalebox{1.0}{$\pm$1.85} & 77.31\scalebox{1.0}{$\pm$1.03} & 75.19\scalebox{1.0}{$\pm$0.87} & 75.36\scalebox{1.0}{$\pm$0.78} & 70.90\scalebox{1.0}{$\pm$1.75} & 77.69 \\
& BiC & 71.29\scalebox{1.0}{$\pm$3.57} & 74.49\scalebox{1.0}{$\pm$0.72} & 75.71\scalebox{1.0}{$\pm$0.60} & 67.45\scalebox{1.0}{$\pm$0.89} & 66.63\scalebox{1.0}{$\pm$0.77} & 71.11 \\
& DER++ & 84.61\scalebox{1.0}{$\pm$2.64} & 78.42\scalebox{1.0}{$\pm$0.64} & 78.37\scalebox{1.0}{$\pm$0.42} & 74.80\scalebox{1.0}{$\pm$1.72} & 74.86\scalebox{1.0}{$\pm$1.93} & 78.09 \\
& HAL & 84.09\scalebox{1.0}{$\pm$3.30} & 77.37\scalebox{1.0}{$\pm$0.55} & 77.66\scalebox{1.0}{$\pm$0.31} & 74.52\scalebox{1.0}{$\pm$1.93} & 75.47\scalebox{1.0}{$\pm$2.35} & 77.82 \\
\hline
\multirow{2}{*}{(d)} & HAT & 87.83\scalebox{1.0}{$\pm$2.44} & 79.57\scalebox{1.0}{$\pm$0.29} & 77.20\scalebox{1.0}{$\pm$0.74} & 79.78\scalebox{1.0}{$\pm$1.59} & 78.25\scalebox{1.0}{$\pm$1.68} & 80.53 \\ 
& Sup & 87.06\scalebox{1.0}{$\pm$3.68} & 80.54\scalebox{1.0}{$\pm$0.12} & 77.81\scalebox{1.0}{$\pm$0.66} & 80.01\scalebox{1.0}{$\pm$0.71} & 78.96\scalebox{1.0}{$\pm$0.64} & 80.87 \\
\hline
& MORE & 88.06\scalebox{1.0}{$\pm$1.84} & \textbf{81.67}\scalebox{1.0}{$\pm$1.27} & \textbf{80.97}\scalebox{1.0}{$\pm$0.80} & 80.72\scalebox{1.0}{$\pm$3.38} & 79.73\scalebox{1.0}{$\pm$2.97} & \textbf{82.23} \Tstrut \\
\bottomrule
\end{tabular}
}
\label{Tab:maintable_auc}
 \vspace{-3mm}
\end{table*}

We use maximum softmax probability (MSP)~\citep{hendrycks2016baseline_msp} as the OOD score of a test sample for the baselines and use maximum output with coefficient in Eq.~\ref{final_prediction} for our method MORE. We employ
Area Under the ROC Curve (AUC) to measure the performance of OOD detection as AUC is the standard metric used in OOD detection papers~\citep{yang2021generalized}. 
We report \textit{average incremental AUC} (AI-AUC) which is the average AUC result at all tasks except the last one as there is no more OOD data after the last task.

Table~\ref{Tab:maintable_auc} shows that our method MORE outperforms all baselines consistently except MUC and GD. For MUC, it performs better than MORE on Tiny-ImageNet, but on average, it is poorer (80.98) than MORE (82.23). For GD, it outperforms MORE on C10-5T, but its overall performance is much lower as it achieves an average of only 77.69 over the 5 experiments.
The best baselines based on accuracy are DER++ and PASS. Their average AI-AUC scores over the experiments shown in the last column are 78.09 for DER++ and 78.32 for PASS, but our method MORE achieves 82.23. 
In~\ref{appendix:different_ood}, we report the OOD detection performances of the models on novel classes drawn from datasets that are completely different from the datasets used in the continual learning process.

\subsubsection{Ablation Study}
We conduct an ablation study to measure the performance gain by each proposed technique, \textbf{back-updating}
in previous models in Section \ref{sec.backward} and the \textbf{distance-based coefficient} in Section \ref{sec.ensemble_scores},  using three experiments. The back-updating by Eq.~\ref{backward} is to improve the earlier task models as they are trained with less diverse OOD data than later models. The modified output with coefficient in Eq.~\ref{eq.ensemble} is to improve the classification accuracy by combining the softmax probability of the task networks and the inverse Mahalanobis distances.

Table~\ref{ablation:tab} compares 
accuracy obtained after applying each method. Both distance-based coefficient and back-updating show large improvements from the original method without any of the two techniques. Although the performance is already competitive with either technique, the performance improves further after applying them together.

\begin{table}
\centering
\caption{Ablation Study}
\resizebox{3.3in}{!}{
\begin{tabular}{l c c c}
\toprule
{} & \multicolumn{1}{c}{C10-5T} & \multicolumn{1}{c}{C100-10T} & \multicolumn{1}{c}{C100-20T}\\
\midrule
\multicolumn{1}{c}{Original} & 91.01\scalebox{1.0}{$\pm$2.48} & 76.93\scalebox{1.0}{$\pm$1.58} & 75.76\scalebox{1.0}{$\pm$2.35}  \\
\multicolumn{1}{c}{Coefficient (C)} & 93.86\scalebox{1.0}{$\pm$1.12} & 80.31\scalebox{1.0}{$\pm$1.02} & 80.77\scalebox{1.0}{$\pm$1.36}  \\
\multicolumn{1}{c}{Back (B)} & 93.36\scalebox{1.0}{$\pm$0.79} & 80.35\scalebox{1.0}{$\pm$1.08} & 80.32\scalebox{1.0}{$\pm$0.82} \\
\midrule
\multicolumn{1}{c}{C + B} & 94.23\scalebox{1.0}{$\pm$0.82} & 81.24\scalebox{1.0}{$\pm$1.24} & 81.59\scalebox{1.0}{$\pm$0.98} \\
\bottomrule
\end{tabular}
}
\caption*{
Performance gains with the proposed techniques. The row Original indicates the method without the coefficient and back-updating and the row Back means the back-updating method.
}
\label{ablation:tab}
\end{table}

\section{Conclusion}
This paper studied open-world continual learning (OWCL), which is necessary for learning in the open world. An open-world learning algorithm first needs to detect novel or out-of-distribution (OOD) items and then learn them continually or incrementally. This incremental learning of new items or classes is referred to as \textit{class incremental learning} (CIL), which is a challenging setting of \textit{continual learning}. Since the traditional CIL does not do OOD detection, we call it the \textit{closed-world CIL}. 
In existing research, novelty/OOD detection and CIL are regarded as two completely different problems. This paper theoretically showed that the two problems can be unified. In particular, we 
first decomposed the closed-world CIL problem into \textit{within-task prediction} (WP) and \textit{task-id prediction} (TP), and then proved that TP is correlated with \textit{closed-world OOD detection}. It then showed that a good performance of the two is both \textit{necessary} and \textit{sufficient} for good CIL performances. We then generalized the traditional closed-world CIL to the open-world CIL (or CIL$^+$). CIL$^+$ does both (open-world) OOD detection and CIL. 
Our theory thus connects and unifies open-world OOD/novelty detection and CIL in continual learning. This combination gives us the paradigm of open-world continual learning or CIL$^+$. 
The theoretical result also provides principled guidance for designing better continual learning algorithms. Based on the result, several new CIL methods have been designed. They outperform strong CIL baselines by a large margin and also perform novelty (or OOD) detection well in the open-world continual learning setting. 

Several interesting future directions are worth pursuing. The central goal of these directions is to achieve learning autonomy. It means that the agent should discover new tasks to learn by itself and also take the initiative to acquire the ground-truth training data for learning. In the current research, the training data for each task is assumed to be provided by human engineers who have collected a large amount of labeled training data for each task. To achieve autonomy, the agent has to interact with human users, other agents, and/or the application environment. This requires an interactive module. To interact with humans, a natural language dialogue system is needed, which must also improve itself continually. Since new tasks are the discovered unknowns, an effective algorithm is also needed to decide whether and how to interact with human users to acquire the ground truth labels and training data when it is uncertain about an unfamiliar object that it sees. Another interesting direction is few-shot continual learning as human users usually cannot provide a large number of training examples. Few-shot continual learning is highly challenging as it has to deal with the difficulties of few-shot learning, catastrophic forgetting, and inter-task class separation. To address these challenges, a powerful foundation model with strong reasoning capabilities is likely to be necessary.  

\section*{Author Contributions}

Gyuhak Kim developed the ideas and algorithms, conducted the experiments, and led the writing of the paper. Changnan Xiao created the mathematical framework and derived the proofs. Tatsuya Konishi and Zixuan Ke contributed to discussions and assisted with editing. Bing Liu supervised the project, contributed to the development of the ideas and algorithms, and helped with writing and revising the paper.

\section*{Acknowledgements}
We would like to express our sincere gratitude to the reviewers for their valuable and constructive feedback, which has significantly improved the rigor and precision of the theory presented in the paper. The work of Gyuhak Kim, Zixuan Ke, and Bing Liu was supported in part by four NSF grants (IIS-2229876, IIS-1910424, IIS-1838770, and CNS-2225427), a research contract from DARPA (HR001120C0023), and a research contract from KDDI. 

\newpage

\appendix

\section{Proof of Theorems and Corollaries}
\subsection{Proof of Theorem~\ref{thm:ce}}
\label{prf:ce}
\begin{proof}
Since 
$$
\begin{aligned}
    H_{CIL}(x) &= H(y, \{\mathbf{P}(x \in \mathbf{X}_{k, j} | D)\}_{k, j}) \\
    &= - \sum_{k, j} y_{k, j} \log \mathbf{P}(x \in \mathbf{X}_{k, j} | D) \\
    &= - \log \mathbf{P}(x \in \mathbf{X}_{k_0, j_0} | D),
\end{aligned}
$$
$$
\begin{aligned}
    H_{WP}(x) &= H(\Tilde{y}, \{\mathbf{P}(x \in \mathbf{X}_{k_0, j} | x \in \mathbf{X}_{k_0}, D)\}_{j}) \\
    &= - \sum_{j} y_{k_0, j} \log \mathbf{P}(x \in \mathbf{X}_{k_0, j} | x \in \mathbf{X}_{k_0}, D) \\
    &= - \log \mathbf{P}(x \in \mathbf{X}_{k_0, j_0} | x \in \mathbf{X}_{k_0}, D),
\end{aligned}
$$
and 
$$
\begin{aligned}
    H_{TP}(x) &= H(\Bar{y}, \{\mathbf{P}(x \in \mathbf{X}_k | D)\}_{k}) \\
    &= - \sum_{k} \Bar{y}_{k} \log \mathbf{P}(x \in \mathbf{X}_{k} | D) \\
    &= - \log \mathbf{P}(x \in \mathbf{X}_{k_0} | D),
\end{aligned}
$$
we have 
$$
\begin{aligned}
    H_{CIL}(x) &= - \log \mathbf{P}(x \in \mathbf{X}_{k_0, j_0} | D) \\ 
    &= - \log \mathbf{P}(x \in \mathbf{X}_{k_0, j_0} | x \in \mathbf{X}_{k_0}, D) - \log \mathbf{P}(x \in \mathbf{X}_{k_0} | D) \\
    &= H_{WP}(x) + H_{TP}(x) \\
    &\leq \epsilon + \delta. 
\end{aligned}
$$
\end{proof}

\subsection{Proof of Corollary~\ref{cor:expectation}.}
\label{prf:expectation}
\begin{proof}
By proof of Theorem~\ref{thm:ce}, we have 
$$
H_{CIL} (x) = H_{WP} (x) + H_{TP} (x).
$$
Taking expectations on both sides, we have i)
$$
\begin{aligned}
    \mathbb{E}_{x\sim U(\mathbf{X})} [H_{CIL} (x)] &= \mathbb{E}_{x \sim U(\mathbf{X})} [H_{WP} (x)] +\mathbb{E}_{x\sim U(\mathbf{X})} [H_{TP} (x)] \\
    &\leq \mathbb{E}_{x \sim U(\mathbf{X})} [H_{WP} (x)] + \delta.
\end{aligned}
$$
and ii)
$$
\begin{aligned}
    \mathbb{E}_{x\sim U(\mathbf{X})} [H_{CIL} (x)] &= \mathbb{E}_{x \sim U(\mathbf{X})} [H_{WP} (x)] + \mathbb{E}_{x\sim U(\mathbf{X})} [H_{TP} (x)] \\
    &\leq \epsilon + \mathbb{E}_{x\sim U(\mathbf{X})} [H_{TP} (x)].
\end{aligned}
$$
\end{proof}

\subsection{Proof of Theorem~\ref{thm:tp_ood}.}
\label{prf:tp_ood}
\begin{proof}
i) Assume $x \in \mathbf{X}_{k_0}$. For $k = k_0$, we have
$$
\begin{aligned}
    H_{OOD, k_0} (x) &= - \log \mathbf{P}'_{k_0} (x \in \mathbf{X}_{k_0} | D) \\
    &= - \log \mathbf{P}(x \in \mathbf{X}_{k_0} | D) \\
    &= H_{TP} (x) \leq \delta.
\end{aligned}
$$
For $k \neq k_0$, we have
$$
\begin{aligned}
    H_{OOD, k} (x) &= - \log \mathbf{P}'_k (x \in \mathbf{X} \backslash \mathbf{X}_k | D) \\
    &= - \log (1 - \mathbf{P}'_k (x \in \mathbf{X}_k | D))\\
    &= - \log (1 - \mathbf{P} (x \in \mathbf{X}_k | D))\\
    &= - \log \mathbf{P}(x \in \cup_{k' \neq k} \mathbf{X}_{k'} | D) \\
    &\leq - \log \mathbf{P}(x \in \mathbf{X}_{k_0} | D) \\
    &= H_{TP} (x) \leq \delta.
\end{aligned}
$$

ii) Assume $x \in \mathbf{X}_{k_0}$. For $k = k_0$, by $H_{OOD, {k_0}} (x) \leq \delta_{k_0}$, we have 
$$- \log \mathbf{P}'_{k_0} (x \in \mathbf{X}_{k_0} | D) \leq \delta_{k_0},$$ 
which means 
$$\mathbf{P}'_{k_0} (x \in \mathbf{X}_{k_0} | D) \geq e^{-\delta_{k_0}}.$$
For $k \neq k_0$, by $H_{OOD, {k}} (x) \leq \delta_{k}$, we have 
$$- \log \mathbf{P}'_{k} (x \in \mathbf{X} \backslash \mathbf{X}_{k} | D) \leq \delta_{k},$$
which means 
$$\mathbf{P}'_{k} (x \in \mathbf{X}_{k} | D) \leq 1 - e^{-\delta_{k}}.$$
Therefore, we have
$$
\begin{aligned}
    \mathbf{P} (x \in \mathbf{X}_{k_0} | D) &= \frac{\mathbf{P}'_{k_0} (x \in \mathbf{X}_{k_0} |D)}{\sum_{k'} \mathbf{P}'_{k'} (x \in \mathbf{X}_{k'} |D)} \\
    &\geq \frac{e^{-\delta_{k_0}}}{1 + \sum_{k\neq k_0} 1 - e^{-\delta_k}} \\
    &= \frac{e^{-\delta_{k_0}}}{e^{-\delta_{k_0}} + \sum_{k} 1 - e^{-\delta_k}} \\
    &= \frac{1}{1 + e^{\delta_{k_0}}\sum_{k} 1 - e^{-\delta_k}}.
\end{aligned}
$$
Hence, 
$$
\begin{aligned}
    H_{TP} (x) &= -\log \mathbf{P} (x \in \mathbf{X}_{k_0} | D) \\
    &\leq - \log \frac{1}{1 + e^{\delta_{k_0}}\sum_{k} 1 - e^{-\delta_k}} \\
    &= \log (1 + e^{\delta_{k_0}}\sum_{k} 1 - e^{-\delta_k}) \\
    &\leq e^{\delta_{k_0}} (\sum_k 1 - e^{-\delta_k}) \\
    &= (\sum_k \mathbf{1}_{x \in \mathbf{X}_k} e^{\delta_{k}}) (\sum_k 1 - e^{-\delta_k}).
\end{aligned}
$$

\end{proof}

\subsection{Proof of Corollary~\ref{thm:tp_ood_to_open}.}
\label{prf:tp_ood_to_open}
\begin{proof}
i) Assume $x \in \mathbf{X}_{k_0}$. For $k = k_0$, we have
$$
\begin{aligned}
    H_{OOD^+, k_0} (x) &= - \log \mathbf{P}'_{k_0} (x \in \mathbf{X}_{k_0} | D) \\
    &= - \log \mathbf{P}(x \in \mathbf{X}_{k_0} | D) \\
    &= H_{TP^+} (x) \leq \delta.
\end{aligned}
$$
For $k \neq k_0$, we have
$$
\begin{aligned}
    H_{OOD^+, k} (x) &= - \log \mathbf{P}'_k (x \in (\mathbf{X} \cup \mathbf{X}^+) \backslash \mathbf{X}_k | D) \\
    &= - \log (1 - \mathbf{P}'_k (x \in \mathbf{X}_k | D))\\
    &= - \log (1 - \mathbf{P} (x \in \mathbf{X}_k | D))\\
    &= - \log \mathbf{P}(x \in (\cup_{k' \neq k} \mathbf{X}_{k'}) \cup X^+ | D) \\
    &\leq - \log \mathbf{P}(x \in \mathbf{X}_{k_0} | D) \\
    &= H_{TP^+} (x) \leq \delta.
\end{aligned}
$$

ii.a) Assume $x \in \mathbf{X}_{k_0}$. For $k = k_0$, by $H_{OOD^+, {k_0}} (x) \leq \delta_{k_0}$, we have 
$$- \log \mathbf{P}'_{k_0} (x \in \mathbf{X}_{k_0} | D) \leq \delta_{k_0},$$ 
which means 
$$\mathbf{P}'_{k_0} (x \in \mathbf{X}_{k_0} | D) \geq e^{-\delta_{k_0}}.$$
For $k \neq k_0$, by $H_{OOD^+, {k}} (x) \leq \delta_{k}$, we have 
$$- \log \mathbf{P}'_{k} (x \in (\mathbf{X} \cup \mathbf{X}^+) \backslash \mathbf{X}_{k} | D) \leq \delta_{k},$$
which means 
$$\mathbf{P}'_{k} (x \in \mathbf{X}_{k} | D) \leq 1 - e^{-\delta_{k}}.$$
Therefore, we have
$$
\begin{aligned}
    \mathbf{P} (x \in \mathbf{X}_{k_0} | D) &= \frac{\mathbf{P}'_{k_0} (x \in \mathbf{X}_{k_0} |D)}{\sum_{k'} \mathbf{P}'_{k'} (x \in \mathbf{X}_{k'} |D) + \prod_{k'} (1 - \mathbf{P}'_{k'} (x \in \mathbf{X}_{k'} |D))} \\
    &\geq \frac{e^{-\delta_{k_0}}}{1 + \sum_{k\neq k_0} 1 - e^{-\delta_k} + (1 - e^{-\delta_{k_0}}) \prod_{k' \neq k} 1} \\
    &= \frac{e^{-\delta_{k_0}}}{e^{-\delta_{k_0}} + \sum_{k} 1 - e^{-\delta_k} + 1 - e^{-\delta_{k_0}}} \\
    &= \frac{1}{1 + e^{\delta_{k_0}}(1 - e^{-\delta_{k_0}} + \sum_{k} 1 - e^{-\delta_k})}.
\end{aligned}
$$
Hence, 
$$
\begin{aligned}
    H_{TP^+} (x) &= -\log \mathbf{P} (x \in \mathbf{X}_{k_0} | D) \\
    &\leq - \log \frac{1}{1 + e^{\delta_{k_0}}(1 - e^{-\delta_{k_0}} + \sum_{k} 1 - e^{-\delta_k})} \\
    &= \log (1 + e^{\delta_{k_0}}(1 - e^{-\delta_{k_0}} + \sum_{k} 1 - e^{-\delta_k})) \\
    &\leq e^{\delta_{k_0}}(1 - e^{-\delta_{k_0}} + \sum_{k} 1 - e^{-\delta_k}) \\
    &= (\sum_k \mathbf{1}_{x \in \mathbf{X}_k} e^{\delta_{k}}) (\sum_k (1 + \mathbf{1}_{x \in \mathbf{X}_k})(1 - e^{-\delta_k})).
\end{aligned}
$$

ii.b) Assume $x \in \mathbf{X}^+$.
For $k = 1, \dots, T$, by $H_{OOD^+, {k}} (x) \leq \delta_{k}$, we have 
$$- \log \mathbf{P}'_{k} (x \in (\mathbf{X} \cup \mathbf{X}^+) \backslash \mathbf{X}_{k} | D) \leq \delta_{k},$$
which means 
$$\mathbf{P}'_{k} (x \in \mathbf{X}_{k} | D) \leq 1 - e^{-\delta_{k}}.$$
By definition, we have
$$
\begin{aligned}
    \mathbf{P} (x \in \mathbf{X}^+ | D) 
    &= \frac{\prod_{k'} (1 - \mathbf{P}'_{k'} (x \in \mathbf{X}_{k'} |D))}{\sum_{k'} \mathbf{P}'_{k'} (x \in \mathbf{X}_{k'} |D) + \prod_{k'} (1 - \mathbf{P}'_{k'} (x \in \mathbf{X}_{k'} |D))}. \\
\end{aligned}
$$
Hence, 
$$
\begin{aligned}
    H_{TP^+} (x) &= -\log \mathbf{P} (x \in \mathbf{X}^+ | D) \\
    &= \log (1 + \frac{\sum_{k'} \mathbf{P}'_{k'} (x \in \mathbf{X}_{k'} |D)}{\prod_{k'} (1 - \mathbf{P}'_{k'} (x \in \mathbf{X}_{k'} |D))}) \\
    &\leq \frac{\sum_{k'} \mathbf{P}'_{k'} (x \in \mathbf{X}_{k'} |D)}{\prod_{k'} (1 - \mathbf{P}'_{k'} (x \in \mathbf{X}_{k'} |D))} \\
    &\leq \frac{\sum_{k'} 1 - e^{-\delta_{k'}}}{\prod_{k'} e^{-\delta_{k'}}} \\
    &= \prod_{k} e^{\delta_{k}} \sum_{k} 1 - e^{-\delta_{k}}.
\end{aligned}
$$

\end{proof}

\subsection{Proof of Theorem~\ref{thm:cil_with_op_and_ood}.}
\label{prf:cil_with_op_and_ood}
\begin{proof}
Using Theorem~\ref{thm:ce} and \ref{thm:tp_ood},
$$
\begin{aligned}
    H_{CIL}(x) &= - \log \mathbf{P}(x \in \mathbf{X}_{k_0, j_0} | D) \\ 
    &= - \log \mathbf{P}(x \in \mathbf{X}_{k_0, j_0} | x \in \mathbf{X}_{k_0}, D) - \log \mathbf{P}(x \in \mathbf{X}_{k_0} | D) \\
    &= H_{WP}(x) + H_{TP}(x) \\
    &\leq \epsilon + H_{TP}(x) \\
    &\leq \epsilon + (\sum_k \mathbf{1}_{x \in \mathbf{X}_k}  e^{\delta_{k}}) (\sum_k 1 - e^{-\delta_k})
\end{aligned}
$$
\end{proof}

\subsection{Proof of Theorem~\ref{thm:necessary_condition}.}
\label{prf:necessary_condition}
\begin{proof}
i) Assume $x \in \mathbf{X}_{k_0, j_0} \subset \mathbf{X}_{k_0}$. Define $\mathbf{P} (x \in \mathbf{X}_{k, j} | x \in \mathbf{X}_k, D) = \mathbf{P} (x \in \mathbf{X}_{k, j} | D)$. According to proof of Theorem~\ref{thm:ce}, 
$$
\begin{aligned}
H_{WP} (x) &= -\log \mathbf{P} (x \in \mathbf{X}_{k_0, j_0} | x \in \mathbf{X}_{k_0}, D), \\
H_{CIL} (x) &= -\log \mathbf{P} (x \in \mathbf{X}_{k_0, j_0} | D).
\end{aligned}
$$
Hence, we have
$$
\begin{aligned}
H_{WP} (x) &= -\log \mathbf{P} (x \in \mathbf{X}_{k_0, j_0} | x \in \mathbf{X}_{k_0}, D) \\
&= -\log \mathbf{P} (x \in \mathbf{X}_{k_0, j_0} | D) \\ 
&= H_{CIL} (x) \leq \eta.
\end{aligned}
$$

ii) Assume $x \in \mathbf{X}_{k_0, j_0} \subset \mathbf{X}_{k_0}$. Define $\mathbf{P} (x \in \mathbf{X}_{k} | D) = \sum_j \mathbf{P} (x \in \mathbf{X}_{k, j} | D)$. According to proof of Theorem~\ref{thm:ce}, 
$$
\begin{aligned}
H_{TP} (x) &= -\log \mathbf{P} (x \in \mathbf{X}_{k_0} | D), \\
H_{CIL} (x) &= -\log \mathbf{P} (x \in \mathbf{X}_{k_0, j_0} | D).
\end{aligned}
$$
Hence, we have 
$$
\begin{aligned}
H_{TP} (x) &= -\log \mathbf{P} (x \in \mathbf{X}_{k_0} | D) \\
&= -\log \sum_j \mathbf{P} (x \in \mathbf{X}_{k_0, j} | D) \\
&\leq -\log \mathbf{P} (x \in \mathbf{X}_{k_0, j_0} | D) \\
&= H_{CIL} (x) \leq \eta.
\end{aligned}
$$

iii) Assume $x \in \mathbf{X}_{k_0, j_0} \subset \mathbf{X}_{k_0}$. Define $\mathbf{P}'_i (x \in \mathbf{X}_{k} | D) = \mathbf{P} (x \in \mathbf{X}_{k} | D) = \sum_j \mathbf{P} (x \in \mathbf{X}_{k, j} | D)$. According to proof of Theorem~\ref{thm:necessary_condition} ii), we have
$$
H_{TP} (x) \leq \eta.
$$
According to proof of Theorem~\ref{thm:tp_ood} i), we have 
$$
H_{OOD, i} (x) \leq H_{TP} (x).
$$
Therefore, 
$$
H_{OOD, i} (x) \leq H_{TP} (x) \leq \eta.
$$

\end{proof}

\section{Additional Results and Explanation Regarding Table 1 in the Main Paper} \label{apx:additional_odin}
In Section~\ref{sec.betterOOD}, we showed that a better OOD detection improves CIL performance. For the post-processing method ODIN, we only reported the results on C100-10T. Table~\ref{Tab:odin_additional} shows the results on the other datasets.
\begin{table}
    \centering
    \caption{Performance comparison between the original output and output post-processed with OOD detection technique ODIN. Note that ODIN does not apply to iCaRL and Mnemonics as they are not based on softmax but some distance functions. The results for C100-10T are reported in the main paper.}
    \resizebox{5.3in}{!}{
    \begin{tabular}{ll cc cc cc cc cc}
    \toprule
    {} & {} & \multicolumn{2}{c}{C10-5T} & \multicolumn{2}{c}{C100-20T} & \multicolumn{2}{c}{T-5T} & \multicolumn{2}{c}{T-10T} \\
    Method & CIL & AUC & CIL & AUC & CIL & AUC & CIL & AUC & CIL \\
    \midrule
    \multirow{2}{*}{OWM} & Original & 81.33 & 51.79 & 71.90 & 24.15 & 58.49 & 10.00 & 59.48 & 8.57 \\
    {} & ODIN & 71.72 & 40.65 & 68.52 & 23.05 & 58.46 & 10.77 & 59.38 & 9.52 \\
    \hline
    \multirow{2}{*}{MUC} & Original & 79.49 & 52.85 & 66.20 & 14.19 & 68.42 & 33.57 & 62.63 & 17.39 \\
    {} & ODIN & 79.54 & 53.22 & 65.72 & 14.11 & 68.32 & 33.45 & 62.17 & 17.27 \\
    \hline
    \multirow{2}{*}{PASS} & Original & 66.51 & 47.34 & 70.26 & 24.99 & 65.18 & 28.40 & 63.27 & 19.07 \\
    {} & ODIN & 63.08 & 35.20 & 69.81 & 21.83 & 65.93 & 29.03 & 62.73 & 17.78 \\
    \hline
    \multirow{2}{*}{LwF} & Original & 89.39 & 54.67 & 89.84 & 44.33 & 78.20 & 32.17 & 79.43 & 24.28 \\
    {} & ODIN & 88.94 & 63.04 & 88.68 & 47.56 & 76.83 & 36.20 & 77.02 & 28.29 \\
    \hline
    \multirow{2}{*}{A-GEM} & Original & 85.93 & 20.03 & 74.48 & 4.14 & 72.33 & 13.52 & 76.42 & 7.66 \\
    {} & ODIN & 86.43 & 34.03 & 75.12 & 6.99 & 72.46 & 14.69 & 76.75 & 8.50 \\
    \hline
    \multirow{2}{*}{EEIL} & Original & 89.72 & 57.09 & 85.96 & 33.46 & 64.82 & 14.67 & 64.87 & 9.79 \\
    {} & ODIN & 89.20 & 59.47 & 85.46 & 35.16 & 57.01 & 11.92 & 55.42 & 6.88 \\
    \hline
    \multirow{2}{*}{GD} & Original & 91.23 & 58.69 & 86.76 & 38.83 & 68.63 & 16.36 & 69.61 & 11.73 \\
    {} & ODIN & 90.39 & 60.53 & 86.64 & 42.33 & 60.75 & 13.43 & 63.92 & 11.83 \\
    \hline
    \multirow{2}{*}{BiC} & Original & 90.89 & 61.41 & 89.46 & 48.92 & 80.17 & 41.75 & 80.37 & 33.77 \\
    {} & ODIN & 91.86 & 64.29 & 87.89 & 47.40 & 74.54 & 37.40 & 76.27 & 29.06 \\
    \hline
    \multirow{2}{*}{DER++} & Original & 90.16 & 66.04 & 85.44 & 46.59 & 71.80 & 35.80 & 72.41 & 30.49 \\
    {} & ODIN & 87.08 & 63.07 & 87.72 & 49.26 & 73.92 & 37.87 & 72.91 & 32.52 \\
    \hline
    \multirow{2}{*}{HAL} & Original & 86.16 & 32.82 & 65.59 & 13.51 & 53.00 & 3.42 & 57.87 & 3.36 \\
    {} & ODIN & 76.27 & 44.75 & 64.46 & 17.40 & 53.26 & 4.80 & 58.13 & 4.74 \\
    \hline
    \multirow{2}{*}{HAT} & Original & 82.47 & 62.67 & 75.35 & 25.64 & 72.28 & 38.46 & 71.82 & 29.78 \\
    {} & ODIN & 82.45 & 62.60 & 75.36 & 25.84 & 72.31 & 38.61 & 71.83 & 30.01 \\
    \hline
    \multirow{2}{*}{HyperNet} & Original & 78.54 & 53.40 & 72.04 & 18.67 & 54.58 & 7.91 & 55.37 & 5.32 \\
    {} & ODIN & 79.39 & 56.72 & 73.89 & 23.8 & 54.60 & 8.64 & 55.53 & 6.91 \\
    \hline
    \multirow{2}{*}{Sup} & Original & 79.16 & 62.37 & 81.14 & 34.70 & 74.13 & 41.82 & 74.59 & 36.46 \\
    {} & ODIN & 82.38 & 62.63 & 81.48 & 36.35 & 73.96 & 41.10 & 74.61 & 36.46 \\
    \bottomrule
    \end{tabular}
    }
    \label{Tab:odin_additional}
\end{table}

A continual learning method with a better AUC shows a better CIL performance than other methods with lower AUC. 
For instance, original HAT achieves AUC of 82.47 while HyperNet achieves 78.54 on C10-5T. The CIL for HAT is 62.67 while it is 53.40 for HyperNet. However, there are some exceptions that this comparison does not hold. An example is LwF. Its AUC and CIL are 89.39 and 54.67 on C10-5T. Although its AUC is better than HAT, the CIL is lower. This is due to the fact that CIL improves with WP and TP according to Theorem~\ref{thm:ce}. The contraposition of Theorem~\ref{thm:necessary_condition} also says 
if the cross-entropy of TIL is large, that of CIL is also large. Indeed, the average within-task prediction (WP) accuracy for LwF on C10-5T is 95.2 while the same for HAT is 96.7. Improving WP is also important in achieving good CIL performances.

For PASS, we had to tune $\tau_k$ using a validation set. This is because the softmax in Eq.~\ref{eq:odin_softmax} 
improves AUC by making the IND (in-distribution) and OOD scores more separable within a task, but deteriorates the final scores across tasks. To be specific, the test instances are predicted as one of the classes in the first task after softmax because the relative values between classes in task 1 are larger than the other tasks in PASS. Therefore, larger $\tau_1$ and smaller $\tau_k$, for $k > 1$, are chosen to compensate for the relative values.

\section{Definitions of TP} \label{apx:diff_tp}
As noted in the main paper, the class prediction in Eq.~\ref{eq:cil_in_til_and_tp} varies by the definition of WP and TP. The precise definition of WP and TP depends on implementation. Due to this subjectivity, we follow the prediction method in the existing approaches in continual learning, which is the $\argmax$ over the output. In this section, we show that the $\argmax$ over output is a special case of Eq.~\ref{eq:cil_in_til_and_tp}. We also provide CIL results using different definitions of TP.

We first establish another theorem. This is an extension of Theorem~\ref{thm:tp_ood} and connects the standard prediction method to our analysis.
\begin{theorem}[Extension of Theorem~\ref{thm:tp_ood}]
\label{thm:tp_ood_extension}

i) If $H_{TP} (x) \leq \delta$, let $\mathbf{P}'_k (x \in \mathbf{X}_k | D) = \mathbf{P} (x \in \mathbf{X}_k | D)^{1 / \tau_k}$, $\forall \tau_k > 0$, then $H_{OOD, k} (x) \leq \max (\delta /\tau_k, - \log (1 - (1 - e^{-\delta})^{1 / \tau_k}), \forall\, k = 1, \dots, T$. 

ii) If $H_{OOD, k} (x) \leq \delta_k, k=1,\dots,T$, let $\mathbf{P} (x \in \mathbf{X}_k | D) = \frac{\mathbf{P}_k' (x \in \mathbf{X}_k |D)^{1 / \tau_k}}{\sum_j \mathbf{P}_j' (x \in \mathbf{X}_j |D)^{1 / \tau_j}}$, $\forall \tau_k > 0$, then $H_{TP} (x) \leq \sum_k \frac{\mathbf{1}_{x \in \mathbf{X}_k} \delta_{k}}{\tau_{k}} + \frac{\sum_{k} (1 - e^{-\delta_k})^{1 / \tau_k}}{\sum_k \mathbf{1}_{x \in \mathbf{X}_k} (1 - (1 - e^{-\delta_{k}})^{1 / \tau_{k}})}$, where $\mathbf{1}_{x \in \mathbf{X}_k}$ is an indicator function.
\end{theorem}

In Theorem \ref{thm:tp_ood_extension} (proof appears later), we can observe that $\delta / \tau_k$ decreases with the increase of $\tau_k$, while $- \log (1 - (1 - e^{-\delta})^{1 / \tau_k})$ increases. Hence, when TP is given, let $\delta = H_{TP} (x)$, we can find the optimal $\tau_i$ to define OOD by solving $\delta /\tau_k = - \log (1 - (1 - e^{-\delta})^{1 / \tau_k})$. 
Similarly, given OOD, let $\delta_k = H_{OOD, k} (x)$, we can find the optimal $\tau_1, \dots, \tau_T$ to define TP by finding the global minima of $\sum_k \frac{\mathbf{1}_{x \in \mathbf{X}_k} \delta_{k}}{\tau_{k}} + \frac{\sum_{k} (1 - e^{-\delta_k})^{1 / \tau_k}}{\sum_k \mathbf{1}_{x \in \mathbf{X}_k} (1 - (1 - e^{-\delta_{k}})^{1 / \tau_{k}})}$. The optimal $\tau_k$ can be found using a memory buffer to save a small number of previous data like that in a replay-based continual learning method.

In Theorem~\ref{thm:tp_ood_extension} (ii), let $\mathbf{P}'_k (x \in \mathbf{X}_k |D) = \sigma ( \max f(x)_k)$, where $\sigma$ is the sigmoid and $f(x)_k$ is the output of task $k$ and choose $\tau_k \approx 0$  for each $k$. Then $\mathbf{P}(x \in \mathbf{X}_k | D)$ becomes approximately 1 for the task $k$ where the maximum logit value appears and 0 for the rest tasks. Therefore, Eq.~\ref{eq:cil_in_til_and_tp} in the paper
$$
\begin{aligned}
    \mathbf{P}(x \in \mathbf{X}_{k, j} | D) = \mathbf{P}(x \in \mathbf{X}_{k, j} | x \in \mathbf{X}_k, D) \mathbf{P}(x \in \mathbf{X}_k | D)
\end{aligned}
$$
is zero for all classes in tasks $k' \neq k$. Since only the probabilities of classes in task $k$ are non-zero, taking $\argmax$ over all class probabilities gives the same class as $\argmax$ over output logits.

We have also tried another definition of WP and TP. The considered WP is
\begin{align}
    \mathbf{P}(x \in \mathbf{X}_{k, j} | x \in \mathbf{X}_k, D) = \frac{e^{f(x)_{kj} / \nu_k }}{\sum_j e^{f(x)_{kj} / \nu_k }}, \label{eq:wip_max_softmax}
\end{align}
where $\nu_k$ is a temperature scaling parameter for task $k$, and the TP is
\begin{align}
    \mathbf{P}(x \in \mathbf{X}_k |D) = \frac{\mathbf{P}_k'(x \in \mathbf{X}_k | D) }{\sum_k \mathbf{P}_k' (x \in \mathbf{X}_k | D)}, \label{eq:tp_max_softmax}
\end{align}
where $\mathbf{P}_k'(x \in \mathbf{X}_k | D) =  \max_j e^{ f(x)_{kj} / \tau_k } / \sum_j e^{f(x)_{kj} / \tau_k }$
and $\tau_k$ is a temperature scaling parameter. 
This is the maximum softmax of task $k$.
We choose $\nu_k=0.1$ and $\tau_k=5$ for all $k$. A good $\tau$ and $\nu$ can be found using grid search on a validation set. However, one can also find the optimal values by optimization {using some past data saved for the memory buffer.} The CIL results for the new prediction method is in Table~\ref{Tab:maintable_diff}.
\begin{table}[t]
\centering
\caption{Average Accuracy with a Different Prediction Method}
\resizebox{0.95\columnwidth}{!}{
\begin{tabular}{l c c c c c}
\toprule
\multirow{1}{*}{Method}  & \multicolumn{1}{c}{C10-5T} & \multicolumn{1}{c}{C100-10T} & \multicolumn{1}{c}{C100-20T} & \multicolumn{1}{c}{T-5T} & \multicolumn{1}{c}{T-10T} \\
\midrule
\textit{OWM} & 40.6\scalebox{1.0}{$\pm$0.47} & 28.6\scalebox{1.0}{$\pm$0.82} & 22.9\scalebox{1.0}{$\pm$0.32} & 10.4\scalebox{1.0}{$\pm$0.54} & 9.2\scalebox{1.0}{$\pm$0.35} \Tstrut \\
\textit{MUC} & 53.2\scalebox{1.0}{$\pm$1.32} & 30.6\scalebox{1.0}{$\pm$1.21} & 14.0\scalebox{1.0}{$\pm$0.12} & 33.1\scalebox{1.0}{$\pm$0.18} & 17.2\scalebox{1.0}{$\pm$0.13} \\
\textit{PASS}$^{\dagger}$ & 33.6\scalebox{1.0}{$\pm$0.71} & 18.5\scalebox{1.0}{$\pm$1.85} & 20.8\scalebox{1.0}{$\pm$0.85} & 21.4\scalebox{1.0}{$\pm$0.44} & 13.0\scalebox{1.0}{$\pm$0.55} \\
LwF & 63.0\scalebox{1.0}{$\pm$0.34} & 51.9\scalebox{1.0}{$\pm$0.88} & 47.5\scalebox{1.0}{$\pm$0.62} & 35.9\scalebox{1.0}{$\pm$0.32} & 27.8\scalebox{1.0}{$\pm$0.29} \\
iCaRL$^*$ & 65.3\scalebox{1.0}{$\pm$0.83} & 52.9\scalebox{1.0}{$\pm$0.39} & 48.2\scalebox{1.0}{$\pm$0.70} & 34.8\scalebox{1.0}{$\pm$0.34} & 27.3\scalebox{1.0}{$\pm$0.17} \\
A-GEM & 34.0\scalebox{1.0}{$\pm$1.86} & 14.5\scalebox{1.0}{$\pm$0.55} & 7.3\scalebox{1.0}{$\pm$1.78} & 15.4\scalebox{1.0}{$\pm$0.24} & 9.0\scalebox{1.0}{$\pm$0.30} \\
EEIL & 59.5\scalebox{1.0}{$\pm$0.41} & 41.8\scalebox{1.0}{$\pm$0.78} & 37.9\scalebox{1.0}{$\pm$6.11} & 15.1\scalebox{1.0}{$\pm$0.00} & 7.5\scalebox{1.0}{$\pm$0.19} \\
GD & 68.0\scalebox{1.0}{$\pm$0.75} & 47.2\scalebox{1.0}{$\pm$0.33} & 41.8\scalebox{1.0}{$\pm$0.25} & 15.7\scalebox{1.0}{$\pm$2.08} & 12.2\scalebox{1.0}{$\pm$0.14} \\
Mnemonics$^{\dagger *}$ & 65.6\scalebox{1.0}{$\pm$1.55} & 50.7\scalebox{1.0}{$\pm$0.72} & 47.9\scalebox{1.0}{$\pm$0.71} & 36.3\scalebox{1.0}{$\pm$0.30} & 27.7\scalebox{1.0}{$\pm$0.78} \\
BiC & 65.5\scalebox{1.0}{$\pm$0.81} & 50.8\scalebox{1.0}{$\pm$0.69} & 47.2\scalebox{1.0}{$\pm$0.71} & 37.0\scalebox{1.0}{$\pm$0.58} & 29.1\scalebox{1.0}{$\pm$0.34} \\
DER++ & 63.1\scalebox{1.0}{$\pm$1.12} & 54.6\scalebox{1.0}{$\pm$1.21} & 48.9\scalebox{1.0}{$\pm$1.18} & 37.4\scalebox{1.0}{$\pm$0.72} & 32.1\scalebox{1.0}{$\pm$0.44} \\
HAL & 43.0\scalebox{1.0}{$\pm$3.10} & 20.0\scalebox{1.0}{$\pm$1.15} & 17.0\scalebox{1.0}{$\pm$0.83} & 4.6\scalebox{1.0}{$\pm$0.58} & 4.8\scalebox{1.0}{$\pm$0.50} \\
\textit{HAT} & 62.6\scalebox{1.0}{$\pm$1.31} & 41.5\scalebox{1.0}{$\pm$0.80} & 25.9\scalebox{1.0}{$\pm$0.56} & 38.9\scalebox{1.0}{$\pm$1.62} & 30.1\scalebox{1.0}{$\pm$0.52} \Tstrut \\
\textit{HyperNet} & 56.7\scalebox{1.0}{$\pm$1.23} & 32.4\scalebox{1.0}{$\pm$1.07} & 24.5\scalebox{1.0}{$\pm$1.12} & 8.9\scalebox{1.0}{$\pm$0.58} & 7.0\scalebox{1.0}{$\pm$0.52} \\
\textit{Sup} & 62.6\scalebox{1.0}{$\pm$1.11} & 46.8\scalebox{1.0}{$\pm$0.34} & 36.0\scalebox{1.0}{$\pm$0.32} & 41.5\scalebox{1.0}{$\pm$1.17} & 35.7\scalebox{1.0}{$\pm$0.40} \\
\hline
\textit{HAT+CSI} & 85.2\scalebox{1.0}{$\pm$0.92} & 62.9\scalebox{1.0}{$\pm$1.07} & 53.6\scalebox{1.0}{$\pm$0.84} & 47.0\scalebox{1.0}{$\pm$0.38} & 46.2\scalebox{1.0}{$\pm$0.30} \\
\textit{Sup+CSI} & 87.4\scalebox{1.0}{$\pm$0.40} & 66.6\scalebox{1.0}{$\pm$0.23} & 60.5\scalebox{1.0}{$\pm$0.89} & 47.7\scalebox{1.0}{$\pm$0.30} & 46.3\scalebox{1.0}{$\pm$0.30} \\
HAT+CSI+c & 85.2\scalebox{1.0}{$\pm$0.94} & 63.6\scalebox{1.0}{$\pm$0.69} & 55.4\scalebox{1.0}{$\pm$0.79} & 51.4\scalebox{1.0}{$\pm$0.38} & 46.5\scalebox{1.0}{$\pm$0.26} \\
Sup+CSI+c & 86.2\scalebox{1.0}{$\pm$0.79} & 67.0\scalebox{1.0}{$\pm$0.14} & 60.4\scalebox{1.0}{$\pm$1.04} & 48.2\scalebox{1.0}{$\pm$0.35} & 46.1\scalebox{1.0}{$\pm$0.32} \\ 
\bottomrule
\end{tabular}
}
\caption*{
Average classification accuracy. The results are based on the class prediction method defined with WP and TP in Eq.~\ref{eq:wip_max_softmax} and Eq.~\ref{eq:tp_max_softmax}, respectively. The results can be improved by finding optimal temperature scaling parameters.
}
\label{Tab:maintable_diff}
\end{table}

\begin{proof}
[Proof of Theorem \ref{thm:tp_ood_extension}.]
\label{prf:tp_ood_extension}
i) Assume $x \in \mathbf{X}_{k_0}$. 

For $k = k_0$, we have
$$
\begin{aligned}
    H_{OOD, k_0} (x) &= - \log \mathbf{P}'_{k_0} (x \in \mathbf{X}_{k_0} | D) \\ 
    &= - \frac{1}{\tau_{k_0}} \log \mathbf{P}(x \in \mathbf{X}_{k_0} | D) \\
    &= \frac{1}{\tau_{k_0}} H_{TP} (x) \leq \frac{\delta}{\tau_{k_0}}.
\end{aligned}
$$
For $k \neq k_0$, we have
$$
\begin{aligned}
    H_{OOD, k} (x) &= - \log \mathbf{P}'_k (x \notin \mathbf{X}_k | D) \\
    &= - \log (1 - \mathbf{P}'_k (x \in \mathbf{X}_k | D))\\
    &= - \log (1 - \mathbf{P} (x \in \mathbf{X}_k | D)^{1 / \tau_k})\\
    &= - \log (1 - (1 - \mathbf{P}(x \in \cup_{k' \neq k} \mathbf{X}_{k'} | D))^{1 / \tau_k})\\ 
    &\leq - \log (1 - (1 - \mathbf{P}(x \in \mathbf{X}_{k_0} | D))^{1 / \tau_k})\\
    &= - \log (1 - (1 - e^{-H_{TP}(x)})^{1 / \tau_k})\\
    &\leq - \log (1 - (1 - e^{-\delta})^{1 / \tau_k}).\\
\end{aligned}
$$

ii) Assume $x \in \mathbf{X}_{k_0}$. 

For $k = k_0$, by $H_{OOD, {k_0}} (x) \leq \delta_{k_0}$, we have 
$$- \log \mathbf{P}'_{k_0} (x \in \mathbf{X}_{k_0} | D) \leq \delta_{k_0},$$ 
which means 
$$\mathbf{P}'_{k_0} (x \in \mathbf{X}_{k_0} | D) \geq e^{-\delta_{k_0}}.$$
For $k \neq k_0$, by $H_{OOD, {k}} (x) \leq \delta_{k}$, we have 
$$- \log \mathbf{P}'_{k} (x \notin \mathbf{X}_{k} | D) \leq \delta_{k},$$
which means 
$$\mathbf{P}'_{k} (x \in \mathbf{X}_{k} | D) \leq 1 - e^{-\delta_{k}}.$$

Therefore, we have
$$
\begin{aligned}
    \mathbf{P} (x \in \mathbf{X}_{k_0} | D) &= \frac{\mathbf{P}'_{k_0} (x \in \mathbf{X}_{k_0} |D)^{1/\tau_{k_0}}}{\sum_k \mathbf{P}'_k (x \in \mathbf{X}_k |D)^{1/\tau_{k}}} \\
    &\geq \frac{e^{-\delta_{k_0} / \tau_{k_0}}}{1 + \sum_{k\neq k_0} (1 - e^{-\delta_k})^{1 / \tau_k}} \\
    &= \frac{e^{-\delta_{k_0} / \tau_{k_0}}}{1 - (1 - e^{-\delta_{k_0}})^{1 / \tau_{k_0}} + \sum_{k} (1 - e^{-\delta_k})^{1 / \tau_k}} \\
    &= \frac{e^{-\delta_{k_0} / \tau_{k_0}}}{1 - (1 - e^{-\delta_{k_0}})^{1 / \tau_{k_0}}} \cdot
    \frac{1}{1 + \frac{\sum_{k} (1 - e^{-\delta_k})^{1 / \tau_k}}{1 - (1 - e^{-\delta_{k_0}})^{1 / \tau_{k_0}}}}.
\end{aligned}
$$

Hence, 
$$
\begin{aligned}
    H_{TP} (x) &= -\log \mathbf{P} (x \in \mathbf{X}_{k_0} | D) \\
    &\leq - \log \frac{e^{-\delta_{k_0} / \tau_{k_0}}}{1 - (1 - e^{-\delta_{k_0}})^{1 / \tau_{k_0}}} \cdot
    \frac{1}{1 + \frac{\sum_{k} (1 - e^{-\delta_k})^{1 / \tau_k}}{1 - (1 - e^{-\delta_{k_0}})^{1 / \tau_{k_0}}}} \\
    &= \frac{\delta_{k_0}}{\tau_{k_0}} + \log [1 - (1 - e^{-\delta_{k_0}})^{1 / \tau_{k_0}}] 
    + \log \left[1 + \frac{\sum_{k} (1 - e^{-\delta_k})^{1 / \tau_k}}{1 - (1 - e^{-\delta_{k_0}})^{1 / \tau_{k_0}}}\right]\\
    &\leq \frac{\delta_{k_0}}{\tau_{k_0}} + \frac{\sum_{k} (1 - e^{-\delta_k})^{1 / \tau_k}}{1 - (1 - e^{-\delta_{k_0}})^{1 / \tau_{k_0}}} \\
    &= \sum_k \frac{\mathbf{1}_{x \in \mathbf{X}_k} \delta_{k}}{\tau_{k}} + \frac{\sum_{k} (1 - e^{-\delta_k})^{1 / \tau_k}}{\sum_k \mathbf{1}_{x \in \mathbf{X}_k} (1 - (1 - e^{-\delta_{k}})^{1 / \tau_{k}})}.
\end{aligned}
$$
\end{proof}

\section{Output Calibration} \label{apx:calibration}
In this section, we discuss the output calibration technique used in Section~\ref{sec.HAT+CSI} to improve the final prediction accuracy. Even if an OOD detection of each task was perfect (i.e. the model accepts and rejects IND and OOD samples perfectly), the system could make an incorrect class prediction if the magnitudes of outputs across different tasks are different. To ensure that the output values are comparable, we calibrate the outputs by scaling $\alpha_k$ and shifting $\beta_k$ for each task. The optimal parameters $(\alpha_k, \beta_k) \in R \times R$ can be found by solving the optimization problem using samples in the memory buffer. More precisely, denote the memory buffer $\mathcal{M}$ and calibration parameters $( \alpha, \beta ) \in R^{T} \times R^{T}$, where $T$ is the number of learned tasks. After training $T$th task, we find optimal calibration parameters by minimizing the cross-entropy loss,
\begin{align}
    \mathcal{L} = - \frac{1}{|\mathcal{M}|} \sum_{(x, y) \in \mathcal{M}} \log p(y | x)
\end{align}
where $p(c | x)$ is computed using the softmax,
\begin{align}
    \text{softmax} \bigoplus [ \alpha_k f(x)_k + \beta_k ]
\end{align}
where $\bigoplus$ indicates the concatenation and $f(x)_k$ is the output of task $k$ as Eq.~\ref{eq:cil_pred}. Given the optimal parameters $(\alpha^*, \beta^*)$, we make final prediction as
\begin{align}
    \hat{y} = \argmax \bigoplus [ \alpha_k^* f(x)_k + \beta_k^* ]
\end{align}
If we use $OOD_k = \sigma ( \alpha_k^* f(x)_k + \beta_k^* )$, where $\sigma$ is the sigmoid, and $TP_k = OOD_k / \sum_{k'} OOD_{k'}$, the theoretical results in Section~\ref{sec.theorem} hold.

\section{TIL (WP) Results} \label{apx:til_results}
The TIL (WP) results of all the systems are reported in Table~\ref{Tab:til_full}. HAT and Sup show strong performances compared to the other baselines as they leverage task-specific parameters. However, as shown in Theorem~\ref{thm:ce}, the CIL depends on TP (or OOD). Without an OOD detection mechanism in HAT or Sup, they perform poorly in CIL as shown in the main paper. The contrastive learning in CSI also improves the IND prediction (i.e., WP), and this along with OOD detection results in strong CIL performance.

\begin{table}[t]
\centering
\caption{The TIL Results of All the Systems.}
\resizebox{0.9\columnwidth}{!}{
\begin{tabular}{l c c c c c c}
\toprule
\multirow{1}{*}{Method}  & \multicolumn{1}{c}{C10-5T}  &  \multicolumn{1}{c}{C100-10T} &  \multicolumn{1}{c}{C100-20T} &  \multicolumn{1}{c}{T-5T} & \multicolumn{1}{c}{T-10T} & \multicolumn{1}{c}{Avg.} \\
\midrule
\textit{OWM} & 85.0\scalebox{1.0}{$\pm$0.07} & 59.6\scalebox{1.0}{$\pm$0.83} & 65.4\scalebox{1.0}{$\pm$0.48} & 22.4\scalebox{1.0}{$\pm$0.87} & 28.1\scalebox{1.0}{$\pm$0.55} & 52.1 \\
\textit{MUC} & 95.1\scalebox{1.0}{$\pm$0.10} & 77.3\scalebox{1.0}{$\pm$0.83} & 73.4\scalebox{1.0}{$\pm$9.16} & 55.9\scalebox{1.0}{$\pm$0.26} & 47.2\scalebox{1.0}{$\pm$0.22} & 69.8 \\
\textit{PASS}$^{\dagger}$ & 83.8\scalebox{1.0}{$\pm$0.68} & 72.1\scalebox{1.0}{$\pm$0.70} & 76.8\scalebox{1.0}{$\pm$0.32} & 49.9\scalebox{1.0}{$\pm$0.56} & 46.5\scalebox{1.0}{$\pm$0.39} & 65.8 \\
LwF & 95.2\scalebox{1.0}{$\pm$0.30} & 86.2\scalebox{1.0}{$\pm$1.00} & 89.0\scalebox{1.0}{$\pm$0.45} & 56.4\scalebox{1.0}{$\pm$0.48} & 55.3\scalebox{1.0}{$\pm$0.35} & 76.4 \\
iCaRL & 94.9\scalebox{1.0}{$\pm$0.34} & 84.2\scalebox{1.0}{$\pm$1.04} & 85.7\scalebox{1.0}{$\pm$0.68} & 54.5\scalebox{1.0}{$\pm$0.29} & 52.7\scalebox{1.0}{$\pm$0.37} & 74.4 \\
A-GEM & 82.5\scalebox{1.0}{$\pm$4.19} & 58.9\scalebox{1.0}{$\pm$2.14} & 56.4\scalebox{1.0}{$\pm$7.03} & 32.1\scalebox{1.0}{$\pm$0.90} & 30.1\scalebox{1.0}{$\pm$0.29} & 52.0 \\
EEIL & 93.4\scalebox{1.0}{$\pm$0.02} & 83.1\scalebox{1.0}{$\pm$3.13} & 88.4\scalebox{1.0}{$\pm$2.07} & 30.3\scalebox{1.0}{$\pm$0.89} & 25.9\scalebox{1.0}{$\pm$0.04} & 64.2 \\
GD & 94.4\scalebox{1.0}{$\pm$0.09} & 82.2\scalebox{1.0}{$\pm$0.18} & 85.7\scalebox{1.0}{$\pm$0.20} & 30.7\scalebox{1.0}{$\pm$1.79} & 32.2\scalebox{1.0}{$\pm$0.37} & 65.0 \\
Mnemonics$^{\dagger *}$ & 94.5\scalebox{1.0}{$\pm$0.46} & 82.3\scalebox{1.0}{$\pm$0.30} & 86.2\scalebox{1.0}{$\pm$0.46} & 54.8\scalebox{1.0}{$\pm$0.16} & 52.9\scalebox{1.0}{$\pm$0.66} & 74.1 \\
BiC & 95.4\scalebox{1.0}{$\pm$0.35} & 84.6\scalebox{1.0}{$\pm$0.48} & 88.7\scalebox{1.0}{$\pm$0.19} & 61.5\scalebox{1.0}{$\pm$0.60} & 62.2\scalebox{1.0}{$\pm$0.45} & 78.5 \\
DER++ & 92.0\scalebox{1.0}{$\pm$0.54} & 84.0\scalebox{1.0}{$\pm$9.43} & 86.6\scalebox{1.0}{$\pm$9.44} & 57.4\scalebox{1.0}{$\pm$1.31} & 60.0\scalebox{1.0}{$\pm$0.74} & 76.0 \\
HAL & 82.8\scalebox{1.0}{$\pm$1.94} & 49.5\scalebox{1.0}{$\pm$1.51} & 61.1\scalebox{1.0}{$\pm$1.43} & 13.2\scalebox{1.0}{$\pm$0.77} & 21.2\scalebox{1.0}{$\pm$0.41} & 26.2 \\
\textit{HAT} & 96.7\scalebox{1.0}{$\pm$0.18} & 84.0\scalebox{1.0}{$\pm$0.23} & 85.0\scalebox{1.0}{$\pm$0.98} & 61.2\scalebox{1.0}{$\pm$0.72} & 63.8\scalebox{1.0}{$\pm$0.41} & 78.1 \\
\textit{HyperNet} & 94.6\scalebox{1.0}{$\pm$0.37} & 76.8\scalebox{1.0}{$\pm$1.22} & 83.5\scalebox{1.0}{$\pm$0.98} & 23.9\scalebox{1.0}{$\pm$0.60} & 28.0\scalebox{1.0}{$\pm$0.69} & 61.4 \\
\textit{Sup} & 96.6\scalebox{1.0}{$\pm$0.21} & 87.9\scalebox{1.0}{$\pm$0.27} & 91.6\scalebox{1.0}{$\pm$0.15} & 64.3\scalebox{1.0}{$\pm$0.24} & 68.4\scalebox{1.0}{$\pm$0.22} & 81.8 \\
\hline
\textit{HAT+CSI} & 98.7\scalebox{1.0}{$\pm$0.06} & 92.0\scalebox{1.0}{$\pm$0.37} & 94.3\scalebox{1.0}{$\pm$0.06} & 68.4\scalebox{1.0}{$\pm$0.16} & 72.4\scalebox{1.0}{$\pm$0.21} & 85.2 \\
\textit{Sup+CSI} & 98.7\scalebox{1.0}{$\pm$0.07} & 93.0\scalebox{1.0}{$\pm$0.13} & 95.3\scalebox{1.0}{$\pm$0.20} & 65.9\scalebox{1.0}{$\pm$0.25} & 74.1\scalebox{1.0}{$\pm$0.28} & 85.4 \\
\bottomrule
\end{tabular}
}
\caption*{
{The calibrated versions (+c) of our methods are omitted as calibration does not affect TIL performance. Exemplar-free methods are italicized. The last column Avg. shows the average TIL accuracy of each method over all datasets.}
}
\label{Tab:til_full}
\end{table}

\section{Hyper-parameters}\label{apx:hyper_params}
Here we report the hyper-parameters that we did not report in the main paper due to space limitations. We mainly report the hyper-parameters of the proposed methods, HAT+CSI, Sup+CSI, and their calibrated versions. For all the experiments of the proposed methods, {we use the values chosen by the original CSI~\citep{tack2020csi}.}
We use LARS~\citep{you2017large} optimization with a learning rate 0.1 for training the feature extractor. We linearly increase the learning rate by 0.1 per epoch for the first 10 epochs. After that, we use cosine scheduler~\citep{loshchilov2016sgdr} without restart as in \citep{tack2020csi,chen2020simple}. After training the feature extractor, we train the linear classifier for 100 epochs with SGD with a learning rate 0.1 and reduce the rate by 0.1 at 60, 75, and 90 epochs. For all the experiments except MNIST, we train the feature extractor for 700 epochs with batch size 128. 

For the following hyper-parameters, we use 10\% of training data for validation to find a good set of values. For the number of epochs and batch size for MNIST, Sup+CSI trains for 1000 epochs with a batch size of 32 while HAT+CSI trains for 700 epochs with a batch size of 256.
The hard attention regularization penalty $\lambda_i$ in HAT is different by experiments and task $i$. For MNIST, we use $\lambda_1 = 0.25$, and $\lambda_2 = \cdots = \lambda_5 = 0.1$. For C10-5T, we use $\lambda_1=1.0$, and $\lambda_2 = \cdots = \lambda_5 = 0.75$. For C100-10T, $\lambda_1=1.5$, and $\lambda_2 = \cdots = \lambda_{10} = 1.0$ are used. For C100-20T, $\lambda_1 = 3.5$, and $\lambda_2 = \cdots = \lambda_{20} = 2.5$ are used. For T-5T, $\lambda_i = 0.75$ for all tasks, and lastly, for T-10T, $\lambda_1 = 1.0$, and $\lambda_2 = \cdots = \lambda_{10} = 0.75$ are used. We use larger $\lambda_1$ for the first task than the later tasks as we have found that the larger regularization on the first task results in better accuracy. This is by the definition of regularization in HAT. The earlier task gives a lower penalty than the later tasks. We manually give a larger penalty to the first task. We did not search hyper-parameter $\lambda_t$ for tasks $t \geq 2$. For sparsity in Sup+CSI, we simply choose the least sparsity value of 32 used in the original Sup paper without parameter search.

Calibration methods (HAT+CSI+c and Sup+CSI+c) are based on its memory-free versions (i.e. HAT+CSI and Sup+CSI). Therefore, the model training part uses the same hyper-parameters as their calibration-free counterparts. 
{For calibration training, 
we use SGD with a learning rate 0.01, 160 training iterations, and a batch size of 15 for HAT+CSI+c for all experiments. For Sup+CSI+c, we use the same values for all the experiments except for MNIST.
For MNIST,
we use a learning rate 0.05, batch size of 8, and run 280 iterations.
}

For the baselines, we use the hyper-parameters reported in the original papers or their code. If the hyper-parameters are unknown or the code does not reproduce the result (e.g., the baseline did not implement a particular dataset or the code had undergone significant version change), we search for the hyper-parameters as we did for HAT+CSI and Sup+CSI.

\section{Forgetting Rate}\label{apx:forgetting}
We discuss forgetting rate (i.e., backward transfer)~\citep{Lopez2017gradient}, which is defined for task $t$ as
\begin{align}
    \mathcal{F}^{t} = \frac{1}{t-1} \sum_{k=1}^{t-1} \mathcal{A}_{k}^{\text{init}} - \mathcal{A}_{k}^{t},
\end{align}
where $\mathcal{A}_{k}^{\text{init}}$ is the classification accuracy of task $k$'s data after learning it for the first time and $\mathcal{A}_{k}^{t}$ is the accuracy of task $k$'s data after learning task $t$. We report the forgetting rate after learning the last task.

\begin{figure}[h!]
    \centering
    \includegraphics[width=4in]{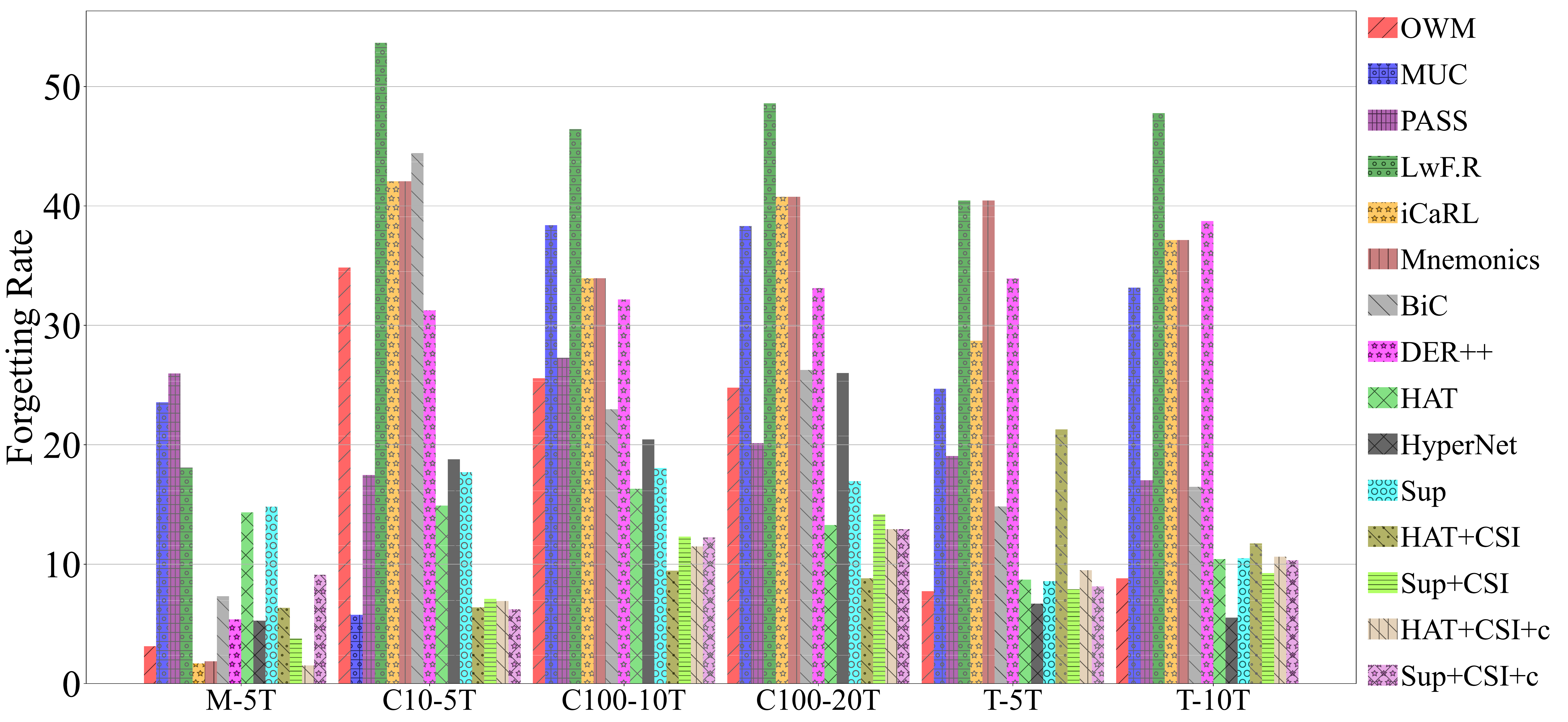}
    \caption{
    Average forgetting rate (\%). The lower the value, the better the method is on forgetting.
    }
    \label{fig:forget}
\end{figure}

Figure~\ref{fig:forget} shows the forgetting rates of each method. Some methods (e.g., OWM, iCaRL) experience less forgetting than the proposed methods HAT+CSI and Sup+CSI on M-5T. On this dataset, all the systems performed well. For instance, OWM and iCaRL achieve 95.8\% and 96.0\% accuracy while HAT+CSI and HAT+CSI+c achieve 94.4 and 96.9\% accuracy. As we have noted in the main paper, Sup+CSI and Sup+CSI+c achieve only 80.7 and 81.0 on M-5T although they have improved drastically from 70.1\% of the base method Sup.

OWM and HyperNet show lower forgetting rates than HAT+CSI+c and Sup+CSI+c on T-5T and T-10T. However, they are not able to adapt to new classes as OWM and HyperNet achieve the classification accuracy of only 10.0\% and 7.9\%, respectively, on T-5T and 8.6\% and 5.3\% on T-10T. HAT+CSI+c and Sup+CSI+c achieves 51.7\% and 49.2\%, respectively, on T-5T and 47.6\% and 46.2\% on T-10T.

In fact, the performance reduction (i.e., forgetting) in our proposed methods occurs not because the systems forget the previous task knowledge, but because the systems learn more classes and the classification naturally becomes harder. The continual learning mechanisms (HAT and Sup) used in the proposed methods experience little or no forgetting because they find an independent subset of parameters for each task, and the learned parameters are not interfered with during training. For the forgetting rate results in the TIL setting, refer to our earlier workshop paper \citep{kim2022continual}.

\section{Pseudo-Code} \label{appendix:pseudo}
For task $k$, Let $p(y | \vx, k) = \text{softmax} f(h(\vx, k; \theta, \ve^{k}); \phi_k)$, where $\theta$ is the parameters for the adapter, $\ve^{k}$ is the trainable embedding for hard attentions, and $\phi_{k}$ is the set of parameters of the classification head of task $k$. Algorithm~\ref{MOREtrain} and Algorithm~\ref{MOREpred} describe the training and testing processes, respectively. We add comments with the symbol ``//''.
\begin{algorithm}[H]
\caption{Training MORE}\label{MOREtrain}
\begin{algorithmic}[1]
    \Require Memory $\mathcal{M}$, learning rate $\lambda$, a sequence of tasks $\mathcal{D}=\{\mathcal{D}^k \}_{k=1}$, and parameters $\{\theta, \ve, \phi \}$, 
    where $\ve$ and $\phi$ are collections of task embeddings $\ve^{k}$ and task heads $\phi_{k}$
    \Statex // CL starts
    \For{each task data $\mathcal{D}^k \in \mathcal{D}$}
        \Statex \hspace{0.5cm} // Model training
        \For{a batch $(\mX_{i}^{k}, \vy)$ in $\mathcal{D}^k$, until converge}
            \State $\mX_s$ = $sample(\mathcal{M})$
            \State Compute loss (Eq.~\ref{ood_obj}$+$Eq.~\ref{eq.loss_hat}) and gradients of parameters
            \State Modify the model parameters $\nabla \theta \leftarrow \nabla \theta'$ using Eq.~\ref{eq:grad_mod}
            \State Update parameters as $\theta \leftarrow \theta - \lambda \nabla \theta, \ \ve^{k} \leftarrow \ve^{k} - \lambda \partial \mathcal{L}, \ \phi_{k} \leftarrow \phi_{k} - \lambda \partial \mathcal{L}$ 
        \EndFor
        \Statex \hspace{0.5cm} // Back-updating in Section~\ref{sec.backward}
        \State Randomly select $\tilde{\mathcal{D}} \subset \mathcal{D}^{k}$, where $|\tilde{\mathcal{D}}| = |\mathcal{M}|$
        \For{each task $j$, until converge}
            \State minimize $\mathcal{L}(\phi_{j})$ of Eq.~\ref{backward}
        \EndFor
        \Statex \hspace{0.5cm} // Obtain statistics in Section~\ref{sec.ensemble_scores}
        \State Compute $\vmu^k_{j}$ using Eq.~\ref{eq.compute_mu} and $\mS^{k}$ using Eq.~\ref{eq.compute_sigma}
    \EndFor
\end{algorithmic}
\end{algorithm}

\begin{algorithm}[H]
\caption{MORE Prediction}\label{MOREpred}
\begin{algorithmic}[1]
    \Require Test instance $\vx$ and parameters $\{\theta, \ve, \phi \}$
    \For {each task $k$}
    \State Obtain $p(\mathcal{Y}^{k} | \vx, k)$
    \State Obtain $s^{k}(\vx)$ using Eq.~\ref{eq.ensemble}
    \EndFor
    \Statex // Concatenate outputs for final prediction $y$ and OOD score $s$
    \State $y = \argmax \bigoplus_{1\leq k \leq t}  p(\mathcal{Y}^{k} | \vx, k) s^{k}(\vx)$ (i.e. Eq.~\ref{final_prediction})
    \State $s = \max \bigoplus_{1\leq k \leq t}  p(\mathcal{Y}^{k} | \vx, k) s^{k}(\vx)$
\end{algorithmic}
\end{algorithm}

\section{Size of Memory Required}\label{appendix:size_of_memory}
In this section, we report the memory size required by each method in Section~\ref{sec.experiment2}. The sizes include network size, replay buffer, and all other parameters or examples kept in memory simultaneously for a model to be functional.

We use an `entry' to refer to a parameter or element in a vector or matrix to calculate the total memory required to train and test. The pre-trained backbone uses 21.6 million (M) entries (parameters). The adapter modules use 1.2M entries for CIFAR10 and 2.4M for other datasets. The baselines and our method use 22.9M and 24.1M entries for the model on CIFAR10 and other datasets, respectively. The unique technique of each method may add additional entries for training and test/inference. 

The total memory required for each method without considering the replay memory buffer is reported in Table~\ref{Tab:memory}. Our method is competitive in memory consumption. Baselines such as OWM and A-GEM take considerably more memory than our system. iCaRL and DER++ take the least amount of memory, but the differences between our method and theirs are only 0.8M, 1.8M, 3.6M, 1.0M, and 1.8M for C10-5T, C100-10T, C100-20T, T-5T, and T-10T.

\begin{table*}[t]
\centering
\caption{Required Memory}
\resizebox{0.7\columnwidth}{!}{
\begin{tabular}{l c c c c c}
\toprule
\multirow{1}{*}{Method}  & \multicolumn{1}{c}{C10-5T}  &  \multicolumn{1}{c}{C100-10T} &  \multicolumn{1}{c}{C100-20T} &  \multicolumn{1}{c}{T-5T} & \multicolumn{1}{c}{T-10T} \\
\midrule
OWM & 26.6M & 28.1M & 28.1M & 28.2M & 28.2M \\
MUC & 22.9M & 24.1M & 24.1M & 24.1M & 24.1M \\
PASS & 22.9M & 24.2M & 24.2M & 24.3M & 24.4M \\
LwF & 22.9M & 24.1M & 24.1M & 24.1M & 24.1M \\
iCaRL & 22.9M & 24.1M & 24.1M & 24.1M & 24.1M \\
A-GEM & 26.5M & 31.4M & 31.4M & 31.5M & 31.5M \\
EEIL & 22.9M & 24.1M & 24.1M & 24.1M & 24.1M \\
GD & 22.9M & 24.1M & 24.1M & 24.1M & 24.1M \\
BiC & 22.9M & 24.1M & 24.1M & 24.1M & 24.1M \\
DER++ & 22.9M & 24.1M & 24.1M & 24.1M & 24.1M \\
HAL & 22.9M & 24.1M & 24.1M & 24.1M & 24.1M \\
HAT & 23.0M & 24.7M & 25.4M & 24.6M & 25.1M \\
Sup & 24.7M & 33.7M & 45.7M & 27.7M & 33.7M \\
MORE & 23.7M & 25.9M & 27.7M & 25.1M & 25.9M  \\
\bottomrule
\end{tabular}
}
\caption*{
Total memory (in entries) required for each method without the replay memory buffer.
}
\label{Tab:memory}
 \vspace{-3mm}
\end{table*}

Many replay-based methods (e.g., iCaRL, HAL) need to save the previous network for distillation during training. This requires an additional 1.2M or 2.4M entries for CIFAR10 or other datasets. Our method does not save the previous model as we do not use distillation.

Note that a large memory consumption usually comes from the memory buffer as the raw data is of size 32*32*3 or 64*64*3 for CIFAR and T-ImageNet. For a memory buffer of size 2000, a system needs 6.1M or 24.6M entries for CIFAR or T-ImageNet. Therefore, saving a smaller number of samples is important for reducing memory consumption. As we demonstrated in Table~\ref{Tab:maintable} and Table~\ref{Tab:smaller_memory} in the main paper, our method performs better than the baselines even with a smaller memory buffer. In Table~\ref{Tab:maintable}, we use large memory sizes (e.g., 200 and 2000 for CIFAR10 and other datasets). In Table~\ref{Tab:smaller_memory}, we reduce the memory size by half. When we compare the accuracy of our method in Table~\ref{Tab:smaller_memory} to those of the baselines in Table~\ref{Tab:maintable}, our method still outperforms them on all datasets. Our method with a smaller memory buffer achieves average classification accuracy of 88.13, 71.69, 71.29, 64.17, 61.90 on C10-5T, C100-10T, C100-20T, T-5T, and T-10T. On the other hand, the best baselines achieve 88.98, 69.73, 70.03, 61.03, 58.34 on the same experiments with a larger memory buffer.

\section{OOD Detection on Different Datasets}\label{appendix:different_ood}
\begin{figure}
    \centering
    \caption{AUC of the continually trained models following the final task: (a) We use all 10 classes learned from 5 tasks of CIFAR-10 as IND and consider LSUN, CIFAR-10, and CIFAR-100 as OOD. (b) We use all 200 classes learned from 10 tasks of Tiny-ImageNet as IND and consider LSUN, CIFAR-10, and CIFAR-100 as OOD.}
    \resizebox{0.6\columnwidth}{!}{
    \subfigure[]{
        \centering
        \captionsetup{justification=centering} 
        \begin{tabular}{l c c c c}
        \toprule
        & MUC & Sup & HAT & MORE \\
        \midrule
        LSUN & 92.60 & 93.60 & 92.88 & 95.61 \\
        CIFAR-100 & 84.23 & 86.64 & 87.88 & 89.56 \\
        Tiny-ImageNet & 93.43 & 94.78 & 91.90 & 95.45 \\
        \bottomrule
        \end{tabular}
    }
    }
    \resizebox{0.6\columnwidth}{!}{
    \subfigure[]{
        \centering
        \captionsetup{justification=centering} 
        \begin{tabular}{l c c c c}
        \toprule
        & MUC & Sup & HAT & MORE \\
        \midrule
        LSUN & 82.33 & 80.26 & 76.82 & 82.96 \\
        CIFAR-10 & 75.52 & 73.86 & 78.98 & 79.76 \\
        CIFAR-100 & 78.58 & 75.98 & 77.62 & 80.21 \\
        \bottomrule
        \end{tabular}
    }
    }
\label{tab:different_ood}
\end{figure}

In Section~\ref{sec.ood} of the main text, we used the classes from unseen tasks of the same dataset used in continual learning as OOD. Here, we also use classes drawn from a dataset that is completely different from the datasets used in the continual learning process. We use the continual learning models after the final task of C10-5T (with the least number of classes, 10) and T-ImageNet-10T (with the largest number of classes, 200) to detect novel samples from completely different datasets. We compare the novelty/OOD detection performance with the three best-performing baseline methods, MUC, HAT, and SupSup. Table~\ref{tab:different_ood} shows that our method, MORE, outperforms the baselines on the completely different OOD classes from different datasets.

\bibliographystyle{elsarticle-harv} 
\bibliography{elsarticle}

\end{document}

%% file: math_commands.tex
\usepackage{amsmath,amsfonts,bm}

\def\eqref#1{equation~\ref{#1}}

\def\1{\bm{1}}

\def\vmu{{\bm{\mu}}}

\def\ve{{\bm{e}}}

\def\vx{{\bm{x}}}
\def\vy{{\bm{y}}}

\def\mS{{\bm{S}}}

\def\mX{{\bm{X}}}

\DeclareMathAlphabet{\mathsfit}{\encodingdefault}{\sfdefault}{m}{sl}
\SetMathAlphabet{\mathsfit}{bold}{\encodingdefault}{\sfdefault}{bx}{n}

\DeclareMathOperator*{\argmax}{arg\,max}